%% file: main.tex
\documentclass[aos,preprint]{imsart}

\usepackage[utf8]{inputenc}
\usepackage{algorithm}
\usepackage{algpseudocode}
\usepackage{amsmath}
\usepackage{amssymb}
\usepackage{amsthm}
\usepackage{xargs}
\usepackage{mathrsfs}
\usepackage{bm}
\usepackage{graphicx}
\usepackage{booktabs}
\usepackage{tabularx}
\usepackage[colorlinks=true, linkcolor=blue]{hyperref}
\usepackage{scalerel,stackengine}
\usepackage{multirow}
\usepackage{makecell}
\usepackage{chngcntr}
\usepackage{apptools}
\usepackage{dsfont}
\usepackage[T1]{fontenc}
\usepackage[numbers]{natbib}
\usepackage{longtable}
\usepackage{thmtools, thm-restate}
\usepackage{cancel}
\usepackage{tikz}
\usepackage{mathtools}
\DeclarePairedDelimiter\ceil{\lceil}{\rceil}
\DeclarePairedDelimiter\floor{\lfloor}{\rfloor}
\usepackage[capitalise,noabbrev]{cleveref}

\usepackage{artefacts}

\input{artefacts/artefacts.tex}

\DeclareMathOperator*{\argmin}{arg\,min}
\let\P\undefined
\DeclareMathOperator{\P}{\mathbb P}
\DeclareMathOperator{\X}{\mathbb X}

\theoremstyle{plain}
\newtheorem{theorem}{Theorem}[]
\newtheorem{proposition}{Proposition}[]

\newtheorem{lemma}{Lemma}[]

\newtheorem{assumption}{Assumption}[]

\crefname{assumption}{Assumption}{Assumptions}
\crefname{example}{Example}{Examples}
\crefname{equation}{Eq.}{Eqs.}
\creflabelformat{equation}{#2#1#3}
\newcommandx{\pto}[0]{\overset{P}{\to}}
\usepackage{deferproofs}
\newEndThm{exampleE}{example}
\newcommand{\numberinappendix}[1]{%
  \counterwithin{#1}{section}%
  \expandafter\def\csname theH#1\endcsname{\theHsection.\arabic{#1}}}

\begin{document}
\begin{frontmatter}
\title{Highly Adaptive Ridge}
\runtitle{Highly Adaptive Ridge}

\begin{aug}
\author[A]{\fnms{Alejandro}~\snm{Schuler}\ead[label=e1]{alejandro.schuler@berkeley.edu}},
\author[A]{\fnms{Mingxun}~\snm{Wang}\ead[label=e4]{michael12321@berkeley.edu}},
\author[B]{\fnms{Alexander}~\snm{Hagemeister}\ead[label=e2]{xanhagem@berkeley.edu}},
\and
\author[A]{\fnms{Mark}~\snm{van der Laan}\ead[label=e3]{laan@berkeley.edu}}
\address[A]{Division of Biostatistics, UC Berkeley\printead[presep={,\ }]{e1,e3,e4}}
\address[B]{EECS, UC Berkeley\printead[presep={,\ }]{e2}}

\runauthor{A. Schuler et al.}
\end{aug}

\begin{abstract}
 In this paper we propose the Highly Adaptive Ridge (HAR): a regression method that achieves an $n^{-1/3}$ dimension-free
 $\mathscr L_2$ convergence rate (up to log factors) in the class of right-continuous functions with square-integrable sectional derivatives.
 This is a large nonparametric function class that is particularly appropriate for tabular data.
 For squared error loss, HAR is exactly kernel ridge regression with a specific data-adaptive kernel based on a saturated zero-order tensor-product spline basis expansion, and thus scales reasonably well, computationally.
 We prove close connections between HAR and kernel ridge regression in Sobolev spaces of dominating mixed derivatives, and highlight practical differences.
 We use simulation and real data to confirm our theory.
\end{abstract}

\end{frontmatter}

\input{sections/10-introduction}
\input{sections/20-method}
\input{sections/30-theory/00-theory}

\input{sections/30-theory/10-har-functions}
\input{sections/30-theory/20-mixed-sobolev}
\input{sections/30-theory/30-inductive-biases}
\input{sections/30-theory/40-convergence-rate}
\input{sections/40-demonstration/00-demonstration}
\input{sections/40-demonstration/10-simple-function}
\input{sections/40-demonstration/20-bias-simulation}
\input{sections/40-demonstration/30-empirical}
\input{sections/50-discussion}

\let\doi\relax
\bibliographystyle{plainnat}
\bibliography{references}

\appendix
\numberinappendix{theorem}
\numberinappendix{proposition}
\numberinappendix{lemma}
\numberinappendix{corollary}
\numberinappendix{assumption}
\numberinappendix{example}

\input{sections/A1-function-classes}

\input{sections/A2-proof-rate}

\input{sections/A3-experiments}

\end{document}

%% file: artefacts/artefacts.tex
\input{artefacts/numbers/bias-dimension.tex}

\input{artefacts/numbers/bias-truths.tex}

\input{artefacts/numbers/biases.tex}

\input{artefacts/numbers/depth.tex}

\input{artefacts/numbers/empirical.tex}

\input{artefacts/numbers/fits.tex}

\input{artefacts/numbers/kernel-scale.tex}

%% file: artefacts/numbers/bias-dimension.tex
\artefactdefine{mthBiasDimPs}{\{\num{4}, \num{8}, \num{12}, \num{16}, \num{24}\}}
\artefactdefine{mthBiasDimReps}{\num{200}}
\artefactdefine{resBiasDimMaxRelSe}{\qty{1.2}{\percent}}

%% file: artefacts/numbers/bias-truths.tex
\artefactdefine{mthBiasTruthPoints}{\num{300}}

%% file: artefacts/numbers/biases.tex
\artefactdefine{mthBiasDeepDepths}{\numlist{12;16}}
\artefactdefine{mthBiasDenseCentre}{\num{0.1}}
\artefactdefine{mthBiasDenseMass}{\qty{80}{\percent}}
\artefactdefine{mthBiasDenseTop}{\num{0.2}}
\artefactdefine{mthBiasDensityHigh}{\num{4}}
\artefactdefine{mthBiasDensityLow}{\num{0.25}}
\artefactdefine{mthBiasGbtRateLow}{\num{0.03}}
\artefactdefine{mthBiasGbtRateNext}{\num{0.1}}
\artefactdefine{mthBiasN}{\num{500}}
\artefactdefine{mthBiasNTest}{\num{5000}}
\artefactdefine{mthBiasP}{\num{4}}
\artefactdefine{mthBiasRampWidth}{\num{0.5}}
\artefactdefine{mthBiasReps}{\num{200}}
\artefactdefine{mthBiasRsq}{\qty{80}{\percent}}
\artefactdefine{mthBiasRuleCentres}{(\num{0.35}, \num{0.55}, \num{0.6}, \num{0.3})}
\artefactdefine{mthBiasSigma}{\num{0.5}}
\artefactdefine{mthBiasSparseCentre}{\num{0.6}}
\artefactdefine{mthBiasSparseWidth}{\num{0.8}}
\artefactdefine{mthBiasWideScaleMin}{\num{-14}}
\artefactdefine{resBiasBvMseWaveOneHigh}{\num{0.022}}
\artefactdefine{resBiasDenseHarMs}{\num{0.42} (\num{0.41}, \num{0.44})}
\artefactdefine{resBiasDenseHarMsWins}{\qty{100}{\percent}}
\artefactdefine{resBiasDepthAtLeastK}{\qty{99}{\percent}}
\artefactdefine{resBiasDepthOneOnWaveOne}{\qty{95}{\percent}}
\artefactdefine{resBiasFlipHal}{\num{1.39} (\num{1.36}, \num{1.43})}
\artefactdefine{resBiasFlipHar}{\num{1.87} (\num{1.82}, \num{1.91})}
\artefactdefine{resBiasFlipHarWorse}{\qty{100}{\percent}}
\artefactdefine{resBiasFlipMsMaxRel}{\num{1e-13}}
\artefactdefine{resBiasFlippedHarMs}{\num{0.79} (\num{0.78}, \num{0.80})}
\artefactdefine{resBiasGbtCapShare}{\qty{0}{\percent}}
\artefactdefine{resBiasGbtDeepAddedShare}{\qty{0}{\percent}}
\artefactdefine{resBiasGbtDeepMseFour}{\num{0.34}}
\artefactdefine{resBiasGbtDeepMseThree}{\num{0.20}}
\artefactdefine{resBiasGbtEdgeShare}{\qty{31}{\percent}}
\artefactdefine{resBiasGbtMseFour}{\num{0.34}}
\artefactdefine{resBiasGbtMseThree}{\num{0.20}}
\artefactdefine{resBiasGbtRateGain}{\qty{1.0}{\percent}}
\artefactdefine{resBiasGbtRateShare}{\qty{44}{\percent}}
\artefactdefine{resBiasMaxRelSe}{\qty{1.9}{\percent}}
\artefactdefine{resBiasMinZ}{\num{17}}
\artefactdefine{resBiasNContrasts}{\num{27}}
\artefactdefine{resBiasPenaltyEdgeFits}{\num{1}}
\artefactdefine{resBiasPenaltyFits}{\num{7200}}
\artefactdefine{resBiasRampsHigh}{\num{0.62}}
\artefactdefine{resBiasRampsLow}{\num{0.41}}
\artefactdefine{resBiasRbfEdgeShare}{\qty{80}{\percent}}
\artefactdefine{resBiasRbfMseFour}{\num{0.053}}
\artefactdefine{resBiasRbfMseThree}{\num{0.053}}
\artefactdefine{resBiasRbfMseWaveOne}{\num{0.033}}
\artefactdefine{resBiasRbfWideMse}{\num{0.031}}
\artefactdefine{resBiasRbfWideScaleHigh}{\num{0}}
\artefactdefine{resBiasRbfWideScaleLow}{\num{-10}}
\artefactdefine{resBiasSparseMsHar}{\num{0.63} (\num{0.62}, \num{0.65})}
\artefactdefine{resBiasSparseMsHarWins}{\qty{100}{\percent}}
\artefactdefine{resBiasStepsHigh}{\num{0.84}}
\artefactdefine{resBiasStepsLow}{\num{0.58}}
\artefactdefine{resBiasWaveFourBvHigh}{\num{0.43}}
\artefactdefine{resBiasWaveFourBvLow}{\num{0.28}}
\artefactdefine{resBiasWaveFourHigh}{\num{8.2}}
\artefactdefine{resBiasWaveFourLow}{\num{5.3}}
\artefactdefine{resBiasWaveOneBvHigh}{\num{0.022}}
\artefactdefine{resBiasWaveOneBvLow}{\num{0.021}}
\artefactdefine{resBiasWaveOneHigh}{\num{0.67}}
\artefactdefine{resBiasWaveOneLow}{\num{0.65}}
\artefactdefine{resBiasWaveOneWinsLow}{\qty{96}{\percent}}
\artefactdefine{resBiasWaveRbfHigh}{\num{0.054}}
\artefactdefine{resBiasWaveRbfLow}{\num{0.033}}
\artefactdefine{resBiasWaveThreeBvHigh}{\num{0.28}}
\artefactdefine{resBiasWaveThreeBvLow}{\num{0.20}}
\artefactdefine{resBiasWaveTwoBvHigh}{\num{0.11}}
\artefactdefine{resBiasWaveTwoBvLow}{\num{0.090}}
\artefactdefine{resBiasWaveTwoHigh}{\num{2.1}}
\artefactdefine{resBiasWaveTwoLow}{\num{1.7}}

%% file: artefacts/numbers/depth.tex
\artefactdefine{mthDepthReps}{\num{5}}
\artefactdefine{resDepthBelowHar}{\num{9}}
\artefactdefine{resDepthBelowMs}{\num{8}}
\artefactdefine{resDepthFullHar}{\num{1776}}
\artefactdefine{resDepthFullHarData}{slice}
\artefactdefine{resDepthFullMs}{\num{1404}}
\artefactdefine{resDepthFullMsData}{slice}
\artefactdefine{resDepthNData}{\num{11}}

%% file: artefacts/numbers/empirical.tex
\artefactdefine{mthCpu}{Apple M1 Pro}
\artefactdefine{mthDepthPatience}{\num{2}}
\artefactdefine{mthEnetEps}{\num{e-6}}
\artefactdefine{mthEnetMixtures}{\numlist{0.1;0.5;0.9}}
\artefactdefine{mthEnetNAlphas}{\num{150}}
\artefactdefine{mthFolds}{\num{5}}
\artefactdefine{mthGbtDepths}{\numlist{1;2;3;4;6;8}}
\artefactdefine{mthGbtMaxTrees}{\num{2000}}
\artefactdefine{mthGbtPatience}{\num{100}}
\artefactdefine{mthGbtRates}{\numlist{0.03;0.1;0.3}}
\artefactdefine{mthHalBudget}{\num{20}}
\artefactdefine{mthHalEps}{\num{e-5}}
\artefactdefine{mthHalNAlphas}{\num{167}}
\artefactdefine{mthHalPatience}{\num{1}}
\artefactdefine{mthHarDepths}{\numlist{1;2;3;4;6;8;12;16;24;32;48;64;96;128;192;256}}
\artefactdefine{mthKrrEps}{\num{e-3}}
\artefactdefine{mthKrrFloor}{\num{e-12}}
\artefactdefine{mthKrrNAlphas}{\num{50}}
\artefactdefine{mthMarsMaxTerms}{\num{201}}
\artefactdefine{mthMlpBatch}{\num{64}}
\artefactdefine{mthMlpDecays}{\numlist{e-5;e-3;e-1}}
\artefactdefine{mthMlpLayers}{\num{2}}
\artefactdefine{mthMlpMaxEpochs}{\num{1000}}
\artefactdefine{mthMlpPatience}{\num{50}}
\artefactdefine{mthMlpRate}{\num{e-3}}
\artefactdefine{mthMlpWidths}{\numlist{128;512}}
\artefactdefine{mthRbfScaleMax}{\num{5}}
\artefactdefine{mthRbfScaleMin}{\num{-6}}
\artefactdefine{mthRfTrees}{\num{2000}}
\artefactdefine{mthSklearn}{1.5.1}
\artefactdefine{mthTableReps}{\num{5}}
\artefactdefine{mthTableRows}{\num{2000}}
\artefactdefine{mthTableTestPct}{\qty{20}{\percent}}
\artefactdefine{mthTableTrainPct}{\qty{80}{\percent}}
\artefactdefine{mthXgboost}{3.4.1}

%% file: artefacts/numbers/fits.tex
\artefactdefine{mthDemoN}{\num{50}}
\artefactdefine{mthDemoReps}{\num{3}}
\artefactdefine{mthDemoSigma}{\num{0.3}}

%% file: artefacts/numbers/kernel-scale.tex
\artefactdefine{mthKernelScaleN}{\num{1000}}
\artefactdefine{mthKernelScaleP}{\num{200}}
\artefactdefine{mthKernelScaleSeed}{\num{2026}}
\artefactdefine{resKernelScaleMax}{\num{1.6e60}}
\artefactdefine{resKernelScaleMaxOff}{\num{8.7e40}}
\artefactdefine{resKernelScaleMin}{\num{2.5e30}}

%% file: sections/10-introduction.tex
\section{Introduction}

In regression our task is to find a function that maps features to an outcome such that the expected loss is minimized \cite{Hastie2009-ix}.
In the past decades a huge number of flexible regression methods have been developed that effectively search over high- or infinite-dimensional function spaces.
These are often collectively called ``machine learning'' methods for regression. 

Here we present a new machine learning method (highly adaptive ridge; HAR) that achieves an $O_{\P}(n^{-1/3}(\log n)^{2(p-1)/3})$ $\mathscr L_2$ convergence rate under mild conditions on the true data-generating process.
This rate is remarkable because it is dimension-free (up to log factors) as a consequence of the assumed function class.
The proposed method has close ties with the highly adaptive lasso (HAL) \cite{hal, hal-og}, tree boosting \cite{ltb, Ki2026-xgboost}, and kernel ridge regression (KRR) \cite{krr, tps-anova}.

We begin with a brief setup, before presenting the HAR estimator and describing its concrete implementation.
We then give some theoretical motivation and connect HAR to existing methods, in particular kernel ridge regression in the Sobolev class of dominating mixed derivatives.
We then provide our rate results.
We conclude with empirical evaluations that highlight the tradeoffs between HAR and other methods.

\subsection{Notation and Preliminaries}

Let $(X_i,Y_i) \in \mathcal X \times \mathds R$ be IID copies of a generic pair $(X,Y)$ with joint distribution $\P$, and let $\X$ denote the marginal distribution of the covariates $X$. 
We will take $\mathcal X = [0,1]^p$ without loss of generality for applications with bounded covariates, and consider $p$ fixed throughout.
We use $\mu$ to refer to the Lebesgue measure on $[0,1]^p$.
Norms are $\mathscr L_2(\mathbb P)$ norms unless stated otherwise. 

We identify all measures with their distribution functions and their expectations and use the same symbol with minor variation for all three.
If $\mathbb Q$ is a probability measure, we use $\mathbb Q\{\mathcal Z\}$ with curly braces to denote the $\mathbb Q$-measure of a set $\mathcal Z$. 
We use the same braces for other signed measures.
We use $\mathbb Q(z) = \mathbb Q\{Z \le z\}$ with parentheses for the value of the cumulative distribution function of $Z \sim \mathbb Q$ at $z$.
And we use the empirical process notation $\mathbb Q g$ or $\mathbb Q[g]$ with square brackets for the expectation $\int g(z) \, \mathrm{d}\mathbb Q(z)$.
In the latter notation, the measure does not average over any randomness in the argument $g$.
$\mathbb E[\cdot]$ takes an expectation over the joint distribution of all random arguments. 
When not clear from context, the measure is included in the subscript, e.g. $\mathbb E_{\mathbb Q}[\cdot]$.
Distributions with an $n$ subscript, e.g. $\mathbb P_n$, indicate the (random) empirical measure of a draw of $n$ IID observations from $\P$.

Let $\ell$ be some loss (e.g. squared error loss), which we construct so that it takes a prediction function $g$ as an input and returns a function of $X,Y$ as output.
For example, if we want $\ell$ to be squared error loss, we let $\ell(g)(X,Y) = (g(X) - Y)^2$.
Let $f = \argmin_{g:\mathcal X \to \mathds R} \P \ell(g)$.
This is the standard regression setup where our goal is to estimate the function $f$ from $n$ samples of the vector of predictors $X$ and the outcome $Y$. 

%% file: sections/20-method.tex
\section{Method}
\label{sec:method}

Highly adaptive ridge performs a ridge regression in a specific, data-adaptive, high-dimensional transformation $H(x)$ of the covariates.
The estimated function $\hat f_n$ is the empirical minimizer of a loss function $\ell$ in the parametric model $\{H(x)^\top \beta : \| \beta \|_2^2 \le \tilde M^2\}$, for a constant $\tilde M > 0$.
The constant $\tilde M$ may be chosen by model selection (suppressed in the notation). 

The high-dimensional basis expansion $H$ is constructed as follows.
Let $s \subseteq \{1, \dots, p\}$ denote a ``section'', i.e. some subset of the coordinates of our unit cube.
For each section and each data point $X_i$, let 
$$h_{i,s}(x) = \prod_{j \in s} 1(0 < X_{i,j} \le x_j)$$
be a single, scalar-valued basis function indexed by $i$ and $s$.
Here and in what follows, we use the convention that $\prod_{j \in \emptyset} u_j = 1$, so the functions $h_{i,\emptyset}(x) = 1$ all give an ``intercept'' term.
The bases $h_{i,s}$ are standard tensor-product zero-order splines each with ``knot point'' $c_{i,s} \in [0,1]^p$ where each element of the knot is $c_{i,s,j} = X_{i,j}$ if $j \in s$ and $c_{i,s,j}=0$ if $j \not\in s$.
In other words, the knot $c_{i,s}$ is the vector $X_i$ with the non-$s$ elements set to 0, i.e. the vector $X_i$ projected to the $s$-face of the unit cube. 
With this notation we can write $h_{i,s}(x) = 1(c_{i,s} \le x)$, where the inequality must hold in all dimensions, when $X_{i,j} > 0$ for every $j \in s$, and $h_{i,s} = 0$ otherwise.
The bases $h_{i,s}$ are data-dependent (random) because $X_i$ is an observed data point.
Our full list of basis functions is 
$$H = [h_{i,s} : i \in \{1, \dots, n\}, s\subseteq \{1, \dots, p\}].$$

We use $d$ to refer to the number of basis functions $|H|$, which is $n2^{p}$ (there are $2^p$ sections and $n$ ``knots'' per section).
Technically the number of distinct nonzero bases can be smaller if there are ``ties'' or zero values in the data and certainly we have already over-counted the intercept term $n-1$ times.
To keep the notation clean, however, we consider $H$ to be a multiset that allows repeated elements and we can proceed with the exact equality $d = n2^p$.
This will make no difference in the computations and theory that follow.

Formally, the HAR estimator is
\begin{align}
    \hat f_n &= \argmin_{g \in {\mathscr F} } \mathbb P_n \ell(g) \nonumber \\
    {\mathscr F} &=
    \left\{
    \begin{array}{l}
        H(x)^\top \beta
        :  \|\beta\|_2^2 \le \tilde M^2
    \end{array}
    \right\}
    \label{eq:har-model}
\end{align}
This is identical to the highly adaptive lasso estimator \cite{hal-og, hal} (or the constrained Hardy-Krause variation least squares estimator of \citet{Fang2019-or}) except that the constraint on the coefficients is on a 2-norm, not a 1-norm.

\subsection{Computation for Squared Error Loss}
\label{sec:computation}

The constrained minimization problem in the coefficient vector $\beta$ described above is most often solved using the Lagrangian formulation 
$$
\hat \beta = \argmin_\beta \ n\,\mathbb{P}_n \ell(H^\top \beta)  + \lambda \|\beta\|_2^2
$$ 
which for squared error loss has the closed-form solution $\hat\beta = (\bm H \bm H^\top + \lambda I_{d})^{-1} \bm H \bm Y$ where $\bm H = [H(X_1), \dots, H(X_n)]$ is the $d \times n$ matrix with columns $H(X_i)$ and $\bm Y^\top = [Y_1 \dots Y_n]$. 

Because there are $d = n2^p$ columns in the ``predictor'' matrix $\bm H^\top$, this problem is impossible to solve computationally for even moderate values of $p$.
Even instantiating the array in memory can be prohibitive.
This is the same problem encountered by the highly adaptive lasso \cite{hal, hal-og, vanderlaan2023higher} and related methods \cite{Fang2019-or, Ki2021-fa, ltb}. 

HAR circumvents this problem with the kernel trick.
This is possible because of the ridge, rather than lasso, regularization.
An application of the Woodbury matrix identity reveals the equivalent expression $\hat\beta = \bm H (\bm H^\top \bm H + \lambda I_n)^{-1} \bm Y$, meaning that a prediction at $x$ can be computed as $\hat f_n(x) = H(x)^\top \bm H (\bm H^\top \bm H + \lambda I_n)^{-1} \bm Y$.
The advantage of this is that prediction at a point depends only on inner products of the form $H(x)^\top H(x')$.
We can analytically work out the kernel function that computes this inner product directly from the lower-dimensional $x,x'$ and avoid ever having to instantiate the design matrix $\bm H^\top$ or invert a $d \times d$ matrix:
\begin{align}
    K(x,x') 
    &= H(x)^\top H(x') \notag \\
    &= \sum_i\sum_s \left(\prod_{j\in s} 1(0 < X_{i,j} \le x_j)\right)\left(\prod_{j\in s} 1(0 < X_{i,j} \le x'_j)\right) \notag \\
    &= \sum_i\sum_s \prod_{j\in s} 1(0 < X_{i,j} \le (x \land x')_j) \notag \\
    &= \sum_i\sum_{s \subseteq s_i(x,x')} 1 \notag \\    
    &= \sum_i 2^{|s_i(x,x')|}
\label{eq:har-kernel}
\end{align}

where $x \land x'$ denotes the elementwise minimum and $s_i(x,x') = \{j: 0 < X_{i,j} \le (x \land x')_j\}$.
The fourth equality follows because the product term is 1 only if $s\subseteq s_i(x,x')$ and 0 otherwise.
This is a simple computation: we compare the point $x\land x'$ to each $X_i$ and count the number of dimensions in which the former is greater than or equal to the latter and the latter is positive.
This does not require us to compute the basis expansions $H(x)$ or the values of $\hat\beta$ and is thus more scalable to high dimensions than explicitly computing the design matrix with $n2^p$ columns, which is almost impossible for e.g. $p>25$.

The work for each knot and each pair of points can also be made small.
The indicator $1(0 < X_{i,j} \le (x \land x')_j)$ is the product of $1(0 < X_{i,j} \le x_j)$ and $1(0 < X_{i,j} \le x'_j)$, and each factor depends on only one of the two points.
We compute the $p$ indicators $1(0 < X_{i,j} \le x_j)$ once for each training or test point $x$ and each knot $X_i$, and store them as a bit string, with the $p$ bits packed into $\ceil{p/64}$ 64-bit integers.
The count $|s_i(x,x')|$ is then the number of positions where the bit strings for $x$ and $x'$ both have a 1, which is computed with a hardware-level bitwise \texttt{AND} and a count of the 1 bits, for 64 coordinates at a time.
For the $n \times n$ training kernel, the bit strings take $O(n^2 p)$ comparisons to build and the counts take $O(n^3 \ceil{p/64})$ integer operations in place of the $O(n^3 p)$ comparisons of a direct implementation.
This small implementation detail gives order-of-magnitude speedups.

As with other kernel methods, cross-validation over the regularization parameter in HAR can be done efficiently using an exact closed-form expression for leave-one-out CV \cite{cawley2005estimating}.
Technically this is a slight abuse because the kernel is data-adaptive so it should be recomputed without the $i$th knot point but this should not make a meaningful difference in practice.
In appendix \ref{sec:krr-reg} we describe the simple method we use to set the values of the regularization $\lambda$ to cross-validate over.

\subsubsection{Sectional Weights}
\label{sec:weights}

A problem that arises with the HAR kernel is that, for large $p$, the elements of the kernel matrix can become massive and even overflow their computational representations.
For example, with one draw of \mthKernelScaleN samples from $X \sim \mathsf{Unif}([0,1]^{\mthKernelScaleP})$, the largest element of the kernel matrix was \resKernelScaleMax while the smallest was \resKernelScaleMin.
The reason this problem arises is that the number of high-dimensional sections balloons when $p$ is large.
These sections each contribute independently to the count above.

Since the problem arises from higher-dimensional sections contributing equally in the count, one solution is to down-weight their contributions.
Specifically, each section can be assigned a weight $w_s \ge 0$ so that the kernel is
\begin{align*}
    K_w(x,x')
    &= \sum_i\sum_s w_s \, h_{i,s}(x)\, h_{i,s}(x') \\
    &= \sum_i\sum_{s \subseteq s_i(x,x')} w_s
\end{align*}
Using the weights is equivalent to imposing the penalty $\sum_{i,s} \beta_{i,s}^2 / w_s$ (infinite penalty $w_s=0$ fixing $\beta_{i,s} = 0$) rather than the standard ridge penalty because the weighted kernel $K_w$ is the kernel of the rescaled basis functions $\sqrt{w_s}\,h_{i,s}$.

Using per-section weights creates a massive number of tunable hyperparameters, so we suggest simplifying by setting $w_s = 1(|s| \le D)\,\rho^{|s|}$ for a maximum interaction depth $D \in \{1, \dots, p\}$ and interaction decay $\rho > 0$ (vanilla HAR corresponds to $D = p$ and $\rho=1$).
This corresponds to the ``order-dependent weights'' of \citet{dick2006good} and to the order variances of additive Gaussian processes \cite{duvenaud2011additive}. 
Tying the regularization to the dimensionality of the section controls the influence of (often unimportant) high-dimensional sections.

Tuning over decay or depth requires re-computing the kernel for each weight setting.
If the regularization strength $\lambda$ is also being tuned, we recommend using a different grid per kernel.

Our Python implementation of HAR can be found at \url{https://github.com/alejandroschuler/HighlyAdaptiveRidge}.

%% file: sections/30-theory/00-theory.tex
\section{Theory}

This section motivates the HAR estimator from a theoretical perspective.
We first introduce functions of bounded sectional variation and then narrow them down further to what we call the HAR function class, explaining how the HAR estimator approximates empirical loss minimization in this nonparametric class.
We then prove some close connections between the HAR function space and the Sobolev space with dominating mixed derivatives.
We explain how HAR's properties affect the kinds of functions it is good at learning.
Finally, we prove that HAR has a dimension-free (up to log factors) convergence rate in the nonparametric class of HAR functions.

\subsection{Functions of Bounded Sectional Variation}

Let $s\subseteq \{1, \dots, p\}$ be some set of coordinates that we call a \textit{section}.
When we write $x_s$ we mean the subvector $(x_j)_{j \in s}$ of $x$, and by $x_{-s}$ we mean the subvector $(x_j)_{j \notin s}$ of the other coordinates.
For points $x, z \in \mathds R^p$, we write $(x_s, z_{-s})$ for the point of $\mathds R^p$ with the coordinates $x_j$ for $j \in s$ and $z_j$ for $j \notin s$, which \citet{Owen2005-er} writes as $x_s : z_{-s}$, and $g(x_s, z_{-s})$ for the value of a function $g$ at this point.

Let $[a,b] = \prod_j [a_j, b_j]$ be a box in $\mathds R^p$ and define the generalized difference (also called ``quasi-volume'' or ``alternating sum'') $\Delta([a,b]) = \sum_{s\subseteq \{1, \dots, p\}} (-1)^{|s|}g(a_s, b_{-s})$.
The points $(a_s, b_{-s})$ are nothing other than the corners of the box $[a,b]$.
If $g$ is treated as a multivariate CDF of a random variable $Z$, the generalized difference expresses the probability that $Z \in (a,b]$.

Let $\mathcal P$ represent a partition of the domain of $X$ into a grid of boxes $[a_i, b_i]$.
The \textit{Vitali variation} of a function $g$ is $\operatorname{Vit}(g) = \sup_{\mathcal P} \sum_{[a_i,b_i] \in \mathcal P} |\Delta([a_i, b_i])|$. 
The Vitali variation captures a nice notion of global variation but it has some unintuitive behavior. 
For example, if $g(x_1, x_2) = 1/x_1$ for $x_1 > 0$ and $g(0, x_2) = 0$, which is unbounded, then the Vitali variation is \textit{zero}.
That motivates the notion of \textit{sectional variation} (also called \textit{Hardy-Krause variation}).\footnote{
This is the variation ``anchored at 0''.
Other authors sometimes use different conventions \cite{Owen2005-er, Fang2019-or, Aistleitner2014-cc}.} 

Let $\mathcal X_s = \prod_{j\in s} (0,1]_j\prod_{j \not\in s} \{0\}$ define the $s$-\textit{face} of the unit cube. 
It is important that zero-edges are not included in these faces so that they are disjoint while their union recovers the unit cube.
In a mild abuse of the subvector notation, we sometimes treat $x_s$ as an element of $\mathcal X_s$, even though the dimensions do not strictly agree. 
Another way of seeing this is that we take $\mathcal X_s$ to simultaneously represent both $\prod_{j\in s} (0,1]_j\prod_{j \not\in s} \{0\}$ and $\prod_{j\in s} (0,1]_j$.

Let $g_s: x_s \mapsto g(x_s, 0_{-s})$ be the restriction of $g$ to the face $\mathcal X_s$.
We write $\mu_s$ for Lebesgue measure on the face $\mathcal X_s$, which for $s = \emptyset$ is the unit point mass at the origin.
The sectional variation of $g$ is defined to be $\sum_s \operatorname{Vit}(g_s)$ where $\operatorname{Vit}(g_{\emptyset}) = |g(0)|$ by convention.
The sectional variation defined this way is a norm and we use the notation $\|\cdot \|_{\mathscr M}$, after the class $\mathscr M$ defined below, to distinguish sectional variation norm from a standard $\mathscr L_2$ norm.\footnote{
In our definition we penalize the value $g(0)$ which makes this a norm.
In other definitions the section over the empty set is not penalized and the sectional variation is a semi-norm \cite{Owen2005-er}.} 
\citet{Owen2005-er} provides an excellent summary of these topics including didactic proofs.

In higher dimensions we say that a function is \textit{right-continuous} if the univariate functions $x_{\{j\}} \mapsto g(x_{\{j\}} , x_{-\{j\}})$ are right-continuous for all $j$ and all values $x_{-\{j\}}$ of the other coordinates in $[0,1]$.
Right-continuous\footnote{
A \textit{cadlag} function is a right-continuous function with left-hand limits.
In some recent work \cite{hal, hal-og, vanderlaan2023higher, ltb} the authors consider \textit{cadlag} functions of bounded sectional variation instead of \textit{right-continuous} functions of bounded sectional variation.
The two classes are the same, however, because the left-hand limits are redundant.
Right-continuous functions of bounded variation can be decomposed into two bounded \textit{entirely monotone} functions which have both right and left limits \cite{Aistleitner2014-cc}.
} functions of bounded sectional variation correspond 1-to-1 with finite signed measures the same way that cumulative distribution functions correspond with probability measures \cite{Aistleitner2014-cc}. 

Abusing notation slightly, we can treat any right-continuous function of bounded sectional variation $g$ as both a function $\mathcal X \to \mathds R$ \textit{and} as a measure on $\mathcal X$ that assigns each box $(a,b]$ the generalized difference of the function $g$ over $[a,b]$.
That is, $g$ operates like a cumulative distribution function of a (signed) measure, which we also call $g$ for simplicity.
We use $\mathscr M$ to denote the set of all such functions (or, equivalently, the set of all finite signed measures).

The restrictions $g_s$ also represent measures on the faces $\mathcal X_s$ that operate the same way.
Given the identification between functions and measures, right-continuous functions of bounded sectional variation on $[0,1]^p$ and their sectional variation norms can be expressed in terms of Lebesgue-Stieltjes integrals.
\input{theory/sectional-representation}

%% file: theory/sectional-representation.tex
\begin{restatable}[{\cite[Theorem 3]{Aistleitner2014-cc}}]{sectionalrepresentation}{sectionalrep}
\phantomsection
\label{thm:sectional-representation}

Every right-continuous function $g$ of bounded sectional variation on $[0,1]^p$ can be uniquely represented as
\begin{align}
g(x) &= \sum_{s \subseteq \{1, \dots, p\}}\int_{(0,x_s]} \mathrm{d}g_s(u), \label{eq:sectional-representation} \\
 \text{with norm }\|g\|_{\mathscr M} &= \sum_{s \subseteq \{1, \dots, p\}}\int_{\mathcal X_s} |\mathrm{d}g_s(u)|, \label{eq:sectional-norm}
\end{align}
where the term for $s = \emptyset$ is $g(0)$ in the first line and $|g(0)|$ in the second.
Conversely, any function that can be written in the way above for finite signed measures $g_s$ defined on faces $\mathcal X_s$ is right-continuous of bounded sectional variation.

\end{restatable}

%% file: sections/30-theory/10-har-functions.tex
\subsection{HAR Functions}
\label{sec:har-functions}
Let $\X_s$ be the marginal distribution of $X_s$, with the distribution function $\X_s(x_s) = \P\{X_s \le x_s\}$.
In our convergence rate result for HAR, we will assume that the Radon-Nikodym derivatives $\mathrm{d}f_s/\mathrm{d}\X_s$ are defined and in $\mathscr L_2(\X_s)$ for all $s$.
The sectional measure $g_s$ of a function $g$ is defined only on the face $\mathcal X_s$, so we restrict the marginal $\X_s$ to the face when we take the derivative $\mathrm{d}g_s/\mathrm{d}\X_s$, and we use $\mathscr L_2(\X_s)$ to denote the space of functions on $\mathcal X_s$ that are square-integrable with respect to $\X_s$.
For $s = \emptyset$ we set $\mathrm{d}g_\emptyset/\mathrm{d}\X_\emptyset = g(0)$, matching the convention $\operatorname{Vit}(g_\emptyset) = |g(0)|$.

The derivatives $\mathrm{d}g_s/\mathrm{d}\X_s$ can be thought of as mixed derivatives defined on faces.
For example, if $X$ is uniform, then $\frac{\mathrm{d}g_{\{1,2\}}}{\mathrm{d}\X_{\{1,2\}}} =\frac{\partial^2 g}{\partial x_1 \partial x_2}$ at the points $(x_1, x_2, 0, \dots, 0)$ of the face $\mathcal X_{\{1,2\}}$, when this mixed partial derivative exists and is continuous. 

We use $\mathscr K(\mathbb X)$ to refer to right-continuous functions where these first-order mixed derivatives w.r.t. $\mathbb X$ exist and are in $\mathscr L_2$ (by itself implying bounded sectional variation).
That is,
\[
\mathscr K(\mathbb X) = \left\{ g:[0,1]^p \to \mathds R \text{ right-continuous} \ : \ \frac{\mathrm{d}g_s}{\mathrm{d}\X_s} \in \mathscr L_2(\X_s) \ \forall s \right\}.
\]
Except where necessary to disambiguate, we use the naked symbol $\mathscr K$ to refer to this class and reduce notational clutter. 

We use $\mathscr K (\mathbb X)$ where needed to emphasize that this is a \textit{distribution-dependent} function class.
For example, the same definition with Lebesgue measure $\mu$ in place of $\X$ gives the class $\mathscr K(\mu)$ for uniform covariates.
When $X$ is not uniform, the HAR class $\mathscr K(\mathbb X)$ can differ from the uniform HAR class $\mathscr K(\mu)$: if the marginal of $X_j$ has an atom at a point $a > 0$, then the step $x \mapsto 1(x_j \ge a)$ is in $\mathscr K(\mathbb X)$ but not in $\mathscr K(\mu)$.

We give $\mathscr K(\mathbb X)$ the inner product
\[
\langle g, h \rangle_{\mathscr K(\mathbb X)} = \sum_{s \subseteq \{1, \dots, p\}} \int_{\mathcal X_s} \frac{\mathrm{d}g_s}{\mathrm{d}\X_s} \, \frac{\mathrm{d}h_s}{\mathrm{d}\X_s} \, \mathrm{d}\X_s,
\]
in which the term for $s = \emptyset$ is $g(0)h(0)$, and the corresponding norm
$$
\|g\|_{\mathscr K(\mathbb X)} = \big(\sum_s \|\mathrm{d}g_s/\mathrm{d}\X_s\|^2_{\mathscr L_2(\X_s)}\big)^{1/2}.
$$

\input{theory/har-rkhs}

To our knowledge, this class has not been formalized for a general covariate distribution, with dependent coordinates, atoms or no density, beyond the uniform case \cite{wahba1990spline, sloan1998when, tps-anova, dick2013high, hefter2016equivalence, simon} and the product densities of quasi-Monte Carlo theory \cite{kuo2006randomly, kuo2010optimal, nichols2014fast}.

\subsubsection{Relation to the HAR Estimator}

Although $\mathscr K(\mathbb X)$ is an RKHS, it is not possible to perform a direct kernel ridge regression using its kernel because the covariate distribution $\mathbb X$ is unknown.
Instead, HAR (\cref{eq:har-model}) performs kernel ridge regression in the \textit{approximating} data-dependent RKHS $\mathscr K(\mathbb X_n)$, where $\mathbb X_n$ denotes the empirical distribution of the sample $X_1, \dots, X_n$.

To show this, note first that every function $g \in \mathscr K(\mathbb X_n)$ has a representation $g(x) = H(x)^\top\beta$.
Membership in $\mathscr K(\mathbb X_n)$ implies sectional Radon-Nikodym derivatives with respect to the empirical distributions $\mathbb X_{n,s}$ of $X_{1,s}, \dots, X_{n,s}$.
That means that $g_s$ must correspond to a measure that places all of its mass at the points $X_{i,s}$: that is, a linear combination of the step functions $h_{i,s}$ used by HAR.
Assembling the faces gives the result for $g$.

When the representation is unique, $\|g\|^2_{\mathscr K(\mathbb X_n)} = n\|\beta\|_2^2$, and, when it is not, there exists a representation (the one with smallest $\|\beta\|_2^2$) for which the equation holds. 
Therefore the data-dependent HAR working model with bound $\tilde M^2 = M^2/n$ (\cref{eq:har-model}) is the ball of radius $M$ in $\mathscr K(\mathbb X_n)$, which we denote $\mathscr K(\mathbb X_n)[M]$.
The HAR estimator is the empirical loss minimizer in this class.
The scaling by $n^{-1}$ is essential: if $\tilde M$ is set to a constant as $n$ grows, the estimate $\hat f_n$ can grow in complexity (sectional variation) without bound as $n$ increases.

\Cref{prop:har-rkhs} with the empirical distribution $\mathbb X_n$ in place of $\mathbb X$ shows that the class $\mathscr K(\mathbb X_n)$ of empirical HAR functions is an RKHS with kernel $\sum_s \X_{n,s}\{(0, x_s \land x'_s]\}$.
This is the HAR kernel in \cref{eq:har-kernel}, scaled by $n^{-1}$, since $2^{|s_i(x,x')|}=\sum_s 1(X_{i,s} \in (0, x_s \land x'_s])$.
The kernel for the empirical HAR class $\mathscr K(\mathbb X_n)$ is the plug-in approximation of the kernel for the HAR class $\mathscr K(\mathbb X)$.
This is a good approximation, even uniformly over distributions $\mathbb X$.

\input{theory/har-kernel-limit}

Although the kernel for the empirical HAR class is a good approximation to the kernel of the true HAR class, the two classes are in fact disjoint (modulo constants) for any $n$ in many cases. 
Any function in any empirical HAR class is a finite linear combination of HAR bases and thus represents a discrete measure.
If the covariates have a Lebesgue density, no discrete measure can have sectional derivatives with respect to the true covariate distribution, meaning that no non-constant function in the empirical HAR class has the sectional derivatives it requires for membership in the true HAR class.

The gap between $\mathscr K(\mathbb X_n)$ and $\mathscr K(\mathbb X)$ means that HAR is not \textit{standard} kernel ridge regression.
The kernel function for HAR is a \textit{data-dependent} approximation of a \textit{distribution-dependent} ``true'' kernel and thus the algorithm is not trivial to analyze with standard RKHS theory \cite{krr-rates-1, krr-rates-2, krr-rates-3}.
Standard kernel ridge would work with balls in a fixed class of functions, e.g. a ball $\mathscr K(\mathbb X)[r]$ of fixed radius $r$, but of course the covariate distribution $\mathbb X$ is unknown so this choice is not possible.

In the sense that HAR's limit kernel depends on the true, unknown distribution, HAR is a kind of ``kernel learning'' method \cite{kernel-learning}.
In the sense that $n^{-1}\mathbb E[H(x)^\top H(x')] = K_{\mathscr K(\mathbb X)}$ (recall $H$ is a random set of functions that we can take expectation over), HAR is approximating a true limit kernel with a random feature map, and thus HAR counts as a random features approximation \cite{random-features, random-features-rates}.
However, the features are not random independently from the data: their randomness comes \textit{from} the data, so typical analyses of random features approximations \cite{random-features-rates} do not apply.
Moreover, the number of random features is \textit{greater} than $n$ (the point of random features approximations is to avoid inverting an $n \times n$ matrix) and the resulting problem needs to be re-kernelized for computational tractability.

Despite having started from the sectional variation norm, the (approximate) norm under which HAR works is $\|\cdot\|_{\mathscr K}$, an $\mathscr L_2$ norm of the sectional derivatives. 
The sectional variation norm is the corresponding $\mathscr L_1$ norm: when the derivatives exist, $\|g\|_{\mathscr M} = \sum_s \int |\mathrm{d}g_s/\mathrm{d}\X_s| \, \mathrm{d}\X_s$.
The highly adaptive lasso bounds the sectional variation norm of its fit, which is at most the 1-norm $\|\beta\|_1$ of its coefficient vector, and HAR bounds the $\mathscr L_2$ norm approximately, through the squared 2-norm $\|\beta\|_2^2$ of the coefficients and the scaled HAR kernel.
Nonetheless, the sectional representation and appeals to bounded sectional variation prove useful in our convergence result.

%% file: theory/har-rkhs.tex
\begin{propositionE}[][text link={The proof is in \autoref{pf:har-rkhs}.}, category=har-rkhs]
\label{prop:har-rkhs}
For every covariate distribution $\X$, let
\[
K_{\mathscr K(\mathbb X)}(x,x') = \sum_{s \subseteq \{1, \dots, p\}} \X_s\big\{(0, x_s \land x'_s]\big\}
\qquad \text{for } x, x' \in [0,1]^p,
\]
where $\X_s\big\{(0, x_s \land x'_s]\big\}$ is the probability that $0 < X_j \le x_j \land x'_j$ for every $j \in s$, and the term for $s = \emptyset$ is $1$.
Then $\mathscr K(\mathbb X)$ with the inner product $\langle \cdot, \cdot \rangle_{\mathscr K(\mathbb X)}$ is a reproducing kernel Hilbert space (RKHS), and its reproducing kernel is $K_{\mathscr K(\mathbb X)}$.
\end{propositionE}

\begin{proofE}
By the definition of an RKHS \cite{aronszajn1950theory}, it is enough to show:
\begin{enumerate}
\item The HAR class $\mathscr K(\mathbb X)$ is a vector space, and $\langle \cdot, \cdot \rangle_{\mathscr K(\mathbb X)}$ is an inner product on it.
\item The class $\mathscr K(\mathbb X)$ is complete in the norm $\|\cdot\|_{\mathscr K(\mathbb X)}$.
\item For every $x' \in [0,1]^p$, the function $K_{\mathscr K(\mathbb X)}(\cdot, x')$ is in $\mathscr K(\mathbb X)$.
\item For every $g \in \mathscr K(\mathbb X)$ and $x' \in [0,1]^p$, $\langle g, K_{\mathscr K(\mathbb X)}(\cdot, x') \rangle_{\mathscr K(\mathbb X)} = g(x')$.
\end{enumerate}
All four follow from the two directions of the \hyperref[thm:sectional-representation]{sectional representation theorem}. 

Item 1.
Sums and multiples of right-continuous functions are right-continuous, and by the uniqueness of the representation, the sectional measures $g_s$, and so the derivatives $\mathrm{d}g_s/\mathrm{d}\X_s$, depend linearly on $g$.
So $\mathscr K(\mathbb X)$ is a vector space, and the form $\langle \cdot, \cdot \rangle_{\mathscr K(\mathbb X)}$ is bilinear, symmetric, and finite by the Cauchy-Schwarz inequality.
If $\langle g, g \rangle_{\mathscr K(\mathbb X)} = 0$, then $g(0) = 0$ and each derivative $\mathrm{d}g_s/\mathrm{d}\X_s$ with $s \neq \emptyset$ is zero $\X_s$-almost everywhere on the face $\mathcal X_s$, so $g = 0$ by  the \hyperref[thm:sectional-representation]{sectional representation theorem}.

Item 2.
By  the \hyperref[thm:sectional-representation]{sectional representation theorem}, the map $g \mapsto (\mathrm{d}g_s/\mathrm{d}\X_s)_s$ is one-to-one from the HAR class $\mathscr K(\mathbb X)$ onto the product of the spaces $\mathscr L_2(\X_s)$ for $s \neq \emptyset$ and $\mathds R$ for $s = \emptyset$, and it preserves the norm.
Each of these spaces is complete, so their product is complete, and so is $\mathscr K(\mathbb X)$.

Item 3.
In the rest of this proof, fix a point $x' \in [0,1]^p$.
For a point $u$ of the face $\mathcal X_s$, $u \in (0, x_s]$ and $u \le x'_s$ exactly when $u \in (0, x_s \land x'_s]$, so
\[
K_{\mathscr K(\mathbb X)}(x, x') = 1 + \sum_{s \neq \emptyset} \int_{(0,x_s]} 1(u \le x'_s) \, \mathrm{d}\X_s(u)
\qquad \text{for every } x \in [0,1]^p.
\]
The display writes $K_{\mathscr K(\mathbb X)}(\cdot, x')$ as in the \hyperref[thm:sectional-representation]{sectional representation theorem}, with the value $1$ at the origin and the finite measures $1(u \le x'_s) \, \mathrm{d}\X_s(u)$ on the faces.
So, by the theorem, $K_{\mathscr K(\mathbb X)}(\cdot, x')$ is right-continuous and its sectional derivatives are the indicators $1(u \le x'_s)$, which are in $\mathscr L_2(\X_s)$ because $\int_{\mathcal X_s} 1(u \le x'_s)^2 \, \mathrm{d}\X_s(u) \le 1$.
So $K_{\mathscr K(\mathbb X)}(\cdot, x')$ is in the HAR class $\mathscr K(\mathbb X)$.

Item 4.
By item 3 and the inner product definition, for every $g \in \mathscr K(\mathbb X)$,
\begin{align*}
\langle g, K_{\mathscr K(\mathbb X)}(\cdot, x') \rangle_{\mathscr K(\mathbb X)}
&= g(0) + \sum_{s \neq \emptyset} \int_{\mathcal X_s} \frac{\mathrm{d}g_s}{\mathrm{d}\X_s}(u) \, 1(u \le x'_s) \, \mathrm{d}\X_s(u) \\
&= g(0) + \sum_{s \neq \emptyset} \int_{(0,x'_s]} \frac{\mathrm{d}g_s}{\mathrm{d}\X_s} \, \mathrm{d}\X_s
= g(x').
\end{align*}
\end{proofE}

%% file: theory/har-kernel-limit.tex
\begin{propositionE}[][text link={The proof is in \autoref{pf:har-kernel-limit}.}, category=har-kernel-limit]
\label{prop:har-kernel-limit}
There is a {universal} constant $C$ such that, for every {dimension $p$ and every} $n \ge 1$,
\[
\sup_{\X} \ \mathbb E_{\X} \left[\left\| K_{\mathscr K(\mathbb X_n)} - K_{\mathscr K(\mathbb X)} \right\|_\infty\right] \le C\, (2^p - 1) \sqrt{p/n},
\]
where the outer supremum is over all covariate distributions $\X$ on $[0,1]^p$, $\mathbb E_{\X}$ is the expectation over a sample $X_1, \dots, X_n$ drawn independently from $\X$, and $\|\cdot\|_\infty$ is the supremum over $x, x' \in [0,1]^p$.
\end{propositionE}

\begin{proofE}
Recall that $\X_{n,s}$ is the empirical distribution of $X_{1,s}, \dots, X_{n,s}$.
By \cref{prop:har-rkhs}, applied to the covariate distribution $\X$ and to the empirical distribution $\X_n$, the kernels of the HAR class and of the empirical HAR class are
\[
K_{\mathscr K(\mathbb X)}(x,x') = \sum_s \X_s\big\{(0, x_s \land x'_s]\big\}
\qquad \text{and} \qquad
K_{\mathscr K(\mathbb X_n)}(x,x') = \sum_s \X_{n,s}\big\{(0, x_s \land x'_s]\big\},
\]
where both terms for $s = \emptyset$ are $1$.

In this proof, let $\tilde X \in \mathds R^p$ have the coordinates $\tilde X_j = X_j$ if $X_j > 0$ and $\tilde X_j = 2$ if $X_j = 0$, which moves the mass of each coordinate at zero above the unit interval.
Let $\tilde{\mathbb X}$ be the distribution of $\tilde X$, and let $\tilde{\mathbb X}_n$ be the empirical distribution of the transformed sample $\tilde X_1, \dots, \tilde X_n$, which is an independent sample from $\tilde{\mathbb X}$.
Every coordinate $X_j$ of the covariates is in $[0,1]$, so, for every $u \in [0,1]$, the event $0 < X_j \le u$ is the event $\tilde X_j \le u$: if $X_j > 0$, then $\tilde X_j = X_j$, and if $X_j = 0$, then neither event happens, because $\tilde X_j = 2 > u$.
Every coordinate $\tilde X_j$ of the transformed covariates is at most $2$.

In this proof, for a nonempty section $s$ and points $x, x' \in [0,1]^p$, let $\eta \in \mathds R^p$ have the coordinates $\eta_j = x_j \land x'_j$ for $j \in s$ and $\eta_j = 2$ for $j \notin s$.
Then
\begin{align*}
\X_s\big\{(0, x_s \land x'_s]\big\}
&= \P\{0 < X_j \le x_j \land x'_j \text{ for every } j \in s\} \\
&= \P\{\tilde X_j \le \eta_j \text{ for every } j \in s\} \\
&= \P\{\tilde X_j \le \eta_j \text{ for } j = 1, \dots, p\} \\
&= \tilde{\mathbb X}(\eta).
\end{align*}
The first equality is the definition of the marginal $\X_s$.
The second is the equality of events of the previous paragraph, at the point $x_j \land x'_j \in [0,1]$.
The third holds because $\tilde X_j \le 2 = \eta_j$ for $j \notin s$.
The last is the definition of the distribution function $\tilde{\mathbb X}$ of the transformed covariates.
The same chain, with the probability replaced by the fraction of the sample, gives $\X_{n,s}\{(0, x_s \land x'_s]\} = \tilde{\mathbb X}_n(\eta)$ for the empirical distribution function $\tilde{\mathbb X}_n$.
So the difference of the two kernels at $(x, x')$ is a sum of $2^p - 1$ values of the error $\tilde{\mathbb X}_n - \tilde{\mathbb X}$ of the empirical distribution function, one for each nonempty section, and
\[
\big\| K_{\mathscr K(\mathbb X_n)} - K_{\mathscr K(\mathbb X)} \big\|_\infty \le (2^p - 1) \sup_{\xi \in \mathds R^p} \bigl\lvert \tilde{\mathbb X}_n(\xi) - \tilde{\mathbb X}(\xi) \bigr\rvert .
\]

So it is enough to show that the expected largest distance $\mathbb E\big[\sup_{\xi \in \mathds R^p} \lvert \tilde{\mathbb X}_n(\xi) - \tilde{\mathbb X}(\xi) \rvert\big]$ between the empirical and the true distribution function is at most $C \sqrt{p/n}$, with a universal constant $C$ that does not depend on the distribution $\tilde{\mathbb X}$ of the transformed covariates.

The cells $(-\infty, \xi]$ with $\xi \in \mathds R^p$ form a Vapnik-Chervonenkis class of index $p + 1$ \cite[Example 2.6.1]{Van_der_Vaart1996-we}, and their indicators are pointwise measurable, because the cells with rational corners approximate every cell from above \cite[Example 2.3.4]{Van_der_Vaart1996-we}.
In this proof, let $\mathscr I$ be the class of the indicators of these cells, with the constant envelope $1$, and, for a probability measure $\mathbb Q$ on $\mathds R^p$, let the covering number $N(\varepsilon, \mathscr I, \mathscr L_2(\mathbb Q))$ be the number of $\mathscr L_2(\mathbb Q)$ balls of radius $\varepsilon$ needed to cover $\mathscr I$.
The covering bound for Vapnik-Chervonenkis classes \cite[Theorem 2.6.4]{Van_der_Vaart1996-we} holds with a universal constant, and it still holds when that constant is made larger, so, in this proof, let $C_1 \ge 1$ be such a constant.
With the index $p + 1$ and the $\mathscr L_2(\mathbb Q)$ norm, the bound gives, for every probability measure $\mathbb Q$ on $\mathds R^p$ and every $0 < \varepsilon < 1$,
\[
\log N\big(\varepsilon, \mathscr I, \mathscr L_2(\mathbb Q)\big) \le \log C_1 + \log(p+1) + (p+1)\log(4e) + 2p\log(1/\varepsilon).
\]
Because $p \ge 1$, we have $1 \le p$, $\log(p+1) \le p$ and $p + 1 \le 2p$.
Because the constant $C_1$ of the covering bound is at least $1$, we have $\log C_1 \ge 0$, and so $\log C_1 \le p \log C_1$.
In this proof, let $C_2 = 2 + \log C_1 + 2\log(4e)$, a universal constant that is at least $2$.
Then
\[
1 + \log N\big(\varepsilon, \mathscr I, \mathscr L_2(\mathbb Q)\big)
\le C_2\, p + 2p\log(1/\varepsilon)
\le C_2\, p\,\big(1 + \log(1/\varepsilon)\big).
\]

The envelope $1$ has norm $1$ in $\mathscr L_2(\mathbb Q)$ for every probability measure $\mathbb Q$, so the uniform entropy integral of the class $\mathscr I$ of indicators \cite[Section 2.14.1]{Van_der_Vaart1996-we} is the left side of
\begin{align*}
\sup_{\mathbb Q} \int_0^1 \Big(1 + \log N\big(\varepsilon, \mathscr I, \mathscr L_2(\mathbb Q)\big)\Big)^{1/2} \, \mathrm{d}\varepsilon
&\le \sqrt{C_2\, p} \int_0^1 \big(1 + \log(1/\varepsilon)\big)^{1/2} \, \mathrm{d}\varepsilon \\
&\le \sqrt{C_2\, p} \left(\int_0^1 \big(1 + \log(1/\varepsilon)\big) \, \mathrm{d}\varepsilon\right)^{1/2} \\
&= \sqrt{2 C_2\, p},
\end{align*}
where the supremum is over the discrete probability measures $\mathbb Q$ on $\mathds R^p$.
The first inequality is the bound on the covering numbers above, the second is the Cauchy-Schwarz inequality, and the equality holds because $\int_0^1 \log(1/\varepsilon)\, \mathrm{d}\varepsilon = 1$.

The mean of the indicator of a cell $(-\infty, \xi]$ is $\tilde{\mathbb X}_n(\xi)$ under the empirical distribution $\tilde{\mathbb X}_n$ and $\tilde{\mathbb X}(\xi)$ under the distribution $\tilde{\mathbb X}$ of the transformed covariates.
So the supremum of the empirical process $\sqrt n (\tilde{\mathbb X}_n - \tilde{\mathbb X})$ over the class $\mathscr I$ of indicators is $\sqrt n$ times the largest distance between the two distribution functions.
The class $\mathscr I$ is pointwise measurable, so the maximal inequality \cite[Theorem 2.14.1]{Van_der_Vaart1996-we}, for the first moment and with the envelope $1$, bounds the expected supremum of this process by a universal constant, which in this proof we call $C_3$, times the uniform entropy integral, so
\[
\mathbb E\Big[\sqrt n \sup_{\xi \in \mathds R^p} \bigl\lvert \tilde{\mathbb X}_n(\xi) - \tilde{\mathbb X}(\xi) \bigr\rvert\Big]
\le C_3 \sqrt{2 C_2}\, \sqrt p .
\]
Division by $\sqrt n$ gives the bound $C \sqrt{p/n}$, with the universal constant $C = C_3 \sqrt{2 C_2}$, built from the constant $C_3$ of the maximal inequality and the constant $C_2$ of the bound on the covering numbers.

\end{proofE}

%% file: sections/30-theory/20-mixed-sobolev.tex
\subsection{Comparison to Sobolev Classes with Dominating Mixed Derivatives}
\label{sec:mixed-sobolev}

The class $\mathscr K$ can also be described without appeal to sectional variation, which illustrates some interesting connections to prior work.
In this subsection, let $\alpha \in \mathds N_0^p$ be a multi-index of integers and define the general mixed derivative as 
$$
\partial^\alpha g = \frac
{\partial^{\|\alpha\|_1} g}
{\prod_j (\partial x_j)^{\alpha_j}}.
$$
As an example, if $\alpha = [2,0,1]$, the mixed derivative in question would be $\partial^3g/\partial x_1^2\  \partial x_3$.
The derivative $\partial^\alpha g$ should be read as a ``weak'' derivative \cite[Section 5.2.1]{Evans2010-pde}.
A function we call $\partial^\alpha g$ is considered to be a derivative of $g$ on the open cube $(0,1)^p$ as long as it satisfies the multidimensional integration by parts formula $\int_{(0,1)^p} \phi (\partial^\alpha g) \, \mathrm{d}\mu = (-1)^{\|\alpha\|_1}\int_{(0,1)^p} (\partial^\alpha \phi) g \, \mathrm{d}\mu$ for every infinitely differentiable function $\phi$ with compact support in $(0,1)^p$.

A \textit{Sobolev class} with smoothness $\gamma$ and integrability exponent $q$ is the set
$$
\mathscr W_q^{(\gamma)} = \left\{ g:(0,1)^p \to \mathds R \ : \ \partial^\alpha g \in \mathscr L_q(\mu) \ \forall \alpha \text{ with } \|\alpha\|_1 \le \gamma \right\},
$$
where $\mathscr L_q(\mu)$ is the usual space of $\mu$-measurable functions $g$ with $\int_{(0,1)^p} |g|^q \, \mathrm{d}\mu < \infty$, and here $\mu$ is Lebesgue measure on $(0,1)^p$.
Sobolev classes are usually understood to formalize a notion of functions having a certain amount of ``global'' smoothness. 

Following \citet{simon} and \citet{Hansen2012-pg}, we define the \textit{Sobolev class with dominating mixed derivatives (mixed Sobolev class)} with smoothness $\gamma$ to be
$$
\mathscr S_q^{(\gamma)} = \left\{ g:(0,1)^p \to \mathds R \ : \ \partial^\alpha g \in \mathscr L_q(\mu) \ \forall \alpha \text{ with } \|\alpha\|_\infty \le \gamma \right\}.
$$
The only difference between this definition and the one above is that we have replaced $\|\alpha\|_1$ by $\|\alpha\|_\infty$, but this changes things substantially.
For example, consider $p=2$ so $\alpha = [\alpha_1, \alpha_2]$ and let $\gamma=1$. 
The Sobolev class $\mathscr W^{(1)}_1$ constraint on $\alpha$ is just concerned with $\alpha: \alpha_1 + \alpha_2 \le 1$, so we only require integrability of $g$ itself and its first derivatives in the $x_1$ or $x_2$ direction.
There is no constraint on the mixed derivative.
However for $\mathscr S_1^{(1)}$ the case $\alpha = [1,1]$ satisfies the condition $\|\alpha\|_\infty = \max_j \alpha_j \le 1$ and therefore the mixed derivative must also be integrable in order for $g$ to be in $\mathscr S^{(1)}_1$.
In general, $\|\alpha\|_\infty \le \|\alpha\|_1 \le p\|\alpha\|_\infty$ for every multi-index $\alpha$, so the definitions give the inclusions $\mathscr W_q^{(p)} \subseteq \mathscr S_q^{(1)} \subseteq \mathscr W_q^{(1)}$, and on the cube $(0,1)^p$ with $q = 2$ and $p \ge 2$ both inclusions are strict \cite[Section 3.2]{simon}.
So a mixed Sobolev class can be thought of as more ``smooth'' than the equivalent Sobolev class, but not as smooth as the Sobolev class that penalizes $\gamma=p$-th order derivatives.

A mixed Sobolev class of functions of dimension $p$ can also be thought of as a tensor product of the corresponding Sobolev classes of one-dimensional functions.
For example, the mixed Sobolev class $\mathscr S_2^{(1)}$ on $(0,1)^p$ is the Hilbert-space tensor product of $p$ copies of the Sobolev class $\mathscr W_2^{(1)}$ on $(0,1)$, and Appendix B of \citet{simon} summarizes this and the related facts about reproducing kernels. 

The mixed Sobolev class $\mathscr S_2^{(1)}$ on $(0,1)^p$ can be made into an RKHS by equipping it with the standard kernel
$$
K_{\mathscr S}(x,x') = \frac{1}{\sinh(1)^p} \prod^p_j \cosh(x_j \land x'_j)\cosh(1-x_j\lor x'_j),
$$
where $\lor$ denotes the maximum.
{Each factor is the reproducing kernel of the Sobolev class $\mathscr W_2^{(1)}$ on $(0,1)$ under the norm $(\|g\|^2_{\mathscr L_2(\mu)} + \|g'\|^2_{\mathscr L_2(\mu)})^{1/2}$ \cite{thomasagnan1996computing, berlinet2004reproducing}.}
\citet{simon} study sieve estimators over this class and use kernel ridge regression with this kernel as an oracle benchmark.

The mixed Sobolev kernel also suffers from very large values in some cases.
As with HAR, we propose the introduction of depth and decay section weights $w_s = 1(|s| \le D)\,\rho^{|s|}$ to fix the problem (see \cref{sec:weights}).
Let $K_{\mathscr W}(x_j, x'_j) = \cosh(x_j \land x'_j)\cosh(1 - x_j \lor x'_j)/\sinh(1)$ be the factor of the mixed Sobolev kernel $K_{\mathscr S}$ for coordinate $j$.
Writing each factor as $1 + (K_{\mathscr W} - 1)$, expanding the product, and weighting the term of each section $s$ by $w_s$ gives what we call the sectionally-weighted mixed Sobolev kernel
$$
K_{\mathscr S, w}(x, x') = \sum_{s \subseteq \{1, \dots, p\}} w_s \prod_{j \in s} \big(K_{\mathscr W}(x_j, x'_j) - 1\big),
$$
where the term for $s = \emptyset$ is $w_\emptyset$.
This is the reproducing kernel of an RKHS (see \cref{prop:mixed-sobolev-weights} in the appendix).
This specific kernel has not been previously studied, to our knowledge.
However, the construction combines the kernels of zero-mean functions of \citet{Durrande2013-anova} with weights on the terms of the expansion, as in smoothing spline ANOVA models \cite{wahba1990spline}, additive Gaussian processes \cite{duvenaud2011additive} and quasi-Monte Carlo \cite{sloan2004finite, dick2006good}.

The mixed Sobolev space $\mathscr S_2^{(1)}$ on $(0,1)^p$ bears a remarkable similarity to our HAR class $\mathscr K(\mathbb X)$, modulo two differences: a) the nature of the derivative: weak vs. Radon-Nikodym with respect to a measure $\mathbb X$, and b) where the derivatives are evaluated: in the interior of the domain vs. on sections. 

Here we show that when $\mathbb X$ is uniform over $[0,1]^p$, that is, $\X = \mu$, the two classes $\mathscr K(\mu)$ and $\mathscr S_2^{(1)}$ are essentially the same, but use different (equivalent) norms. 
To compare their norms, let $\alpha(s) \in \{0,1\}^p$ be the indicator vector of a section $s$, so that $\partial^{\alpha(s)} = \partial^{|s|}/\prod_{j \in s} \partial x_j$. 
The mixed Sobolev norm corresponding to the kernel above is $\|g\|_{\mathscr S} = \big(\sum_s \|\partial^{\alpha(s)} g\|^2_{\mathscr L_2(\mu)}\big)^{1/2}$ \cite[Appendix B]{simon}.

\input{theory/har-equals-mixed-sobolev}

The HAR kernel generalizes a nonstandard kernel associated with the mixed Sobolev class.
For uniform $X$, the kernel of \cref{prop:har-rkhs} is
$$
K_{\mathscr K(\mu)}(x, x') = \prod_{j=1}^p \left(1 + x_j \land x'_j\right),
$$
which is known as the kernel of the Sobolev space anchored at 0 in the quasi-Monte Carlo literature \cite[Section 2.1, Example 4]{dick2015proof}.
For uniform $X$, the limit $K_{\mathscr K(\mu)}$ of the scaled HAR kernel and the mixed Sobolev kernel $K_{\mathscr S}$ therefore define the same class of functions, up to the identification in \cref{prop:har-equals-mixed-sobolev}, with equivalent norms. 
In fact, \citet[Theorem 3]{Aistleitner2014-cc} and \citet[Proposition 11]{hefter2016equivalence} give what is necessary to prove that the (anchored) mixed Sobolev class $\mathscr S_1^{(1)}$ (where the derivatives are required merely to be integrable rather than square-integrable) is equivalent to the right-continuous functions of bounded sectional variation that are absolutely continuous (having densities with respect to the Lebesgue measure).
This has never been explicitly stated, to our knowledge.

When the covariate distribution is \textit{not} uniform, the HAR space $\mathscr K(\mathbb X)$ and the mixed Sobolev space $\mathscr S_2^{(1)}$ can differ.
If $X$ has a density that is bounded above and bounded away from zero on $[0,1]^p$, they are still the same set of functions with equivalent norms, as in the uniform case.
But when this density is unbounded or not bounded away from zero, or when $X$ has no Lebesgue density, they are generally not the same set of functions.
So while the mixed Sobolev space can be used as a fixed point to gain some intuition about HAR, it is important to remember that HAR is adaptive to the covariate distribution. 

Even when the two classes are the same set, the norm $\|\cdot\|_{\mathscr K(\mathbb X)}$ weights the mixed derivatives by the inverse of the covariate density.
Suppose that each marginal distribution $\X_s$ has a density $\mathrm{d}\X_s/\mathrm{d}\mu_s$ that is positive on $\mathcal X_s$, and that $g \in \mathscr K(\mathbb X)$ has continuous mixed derivatives $\partial^{\alpha(s)} g$ on $[0,1]^p$ for every section $s$.
Then the sectional measure $g_s$ has the density $\partial^{\alpha(s)} g$ on the face $\mathcal X_s$, so the sectional derivative $\mathrm{d}g_s/\mathrm{d}\X_s$ is $\partial^{\alpha(s)} g / (\mathrm{d}\X_s/\mathrm{d}\mu_s)$ there, and
\begin{equation}
\|g\|^2_{\mathscr K(\mathbb X)} = \sum_s \int_{\mathcal X_s} \frac{(\partial^{\alpha(s)} g)^2}{\mathrm{d}\X_s/\mathrm{d}\mu_s} \, \mathrm{d}\mu_s .
\label{eq:har-vs-mskrr}
\end{equation}
Relative to the norm $\|\cdot\|_{\mathscr K(\mu)}$ of the Sobolev space anchored at 0, whose square is the display above with the density set to one, the HAR norm $\|\cdot\|_{\mathscr K(\mathbb X)}$ charges the price $(\mathrm{d}\X_s/\mathrm{d}\mu_s)^{-1}$ per unit of squared mixed derivative at each point of the face $\mathcal X_s$.
In other words, HAR penalizes a sectional interaction less where the covariates are dense and more where they are sparse (in that margin).

Unlike the usual mixed Sobolev kernel, the HAR kernel depends on the orientation of each covariate.
For example, with uniform $X$, the function $x_1 x_2$ has squared norm $1$ in the uniform HAR class all from the $\{1,2\}$-section, while its reflection $(1 - x_1)(1-x_2)$ has squared norm $4$, because each of the four sections $\emptyset$, $\{1\}$, $\{2\}$, and $\{1,2\}$ contributes $1$.
So, reversing the direction of a covariate can change the fit of HAR, although a strictly increasing transformation of a covariate that maps $[0,1]$ onto itself cannot.

\subsection{Returning to Sectional Variation}
\label{sec:returning-variation}

The HAR class, the empirical HAR class and the mixed Sobolev class are all subsets of the class $\mathscr M$ of right-continuous functions of bounded sectional variation.
Moreover, each ball in each of these classes, measured in the norm of its class, is contained in a ball of $\mathscr M$ in the sectional variation norm.
For $B > 0$, we write $\mathscr M[B]$ for the variation ball of functions $g \in \mathscr M$ with $\|g\|_{\mathscr M} \le B$.
The mixed Sobolev class needs the identification of \cref{prop:har-equals-mixed-sobolev}, because its members are defined only almost everywhere.

\input{theory/variation-balls}

The proofs of \cref{thm:main-rate,cor:cv-rate} use this proposition to put the HAR model and the truth in one variation ball, whose bound depends only on the HAR constant and on $p$.

%% file: theory/har-equals-mixed-sobolev.tex
\begin{propositionE}[][text link={The proof is in \autoref{pf:har-equals-mixed-sobolev}.}, category=har-equals-mixed-sobolev]
\label{prop:har-equals-mixed-sobolev}
Let $X$ be uniform on $[0,1]^p$ and let $\mathscr S_2^{(1)}$ be the first-order mixed Sobolev class on $(0,1)^p$.

If $g$ is in the uniform HAR class $\mathscr K(\mu)$, then the restriction of $g$ to $(0,1)^p$ is in the mixed Sobolev class $\mathscr S_2^{(1)}$.
Conversely, if $g$ is in the mixed Sobolev class, then exactly one member $\tilde g$ of the uniform HAR class is equal to $g$ $\mu$-almost everywhere, and $\tilde g$ is continuous on $[0,1]^p$.
Moreover, the HAR norm $\|\tilde g\|_{\mathscr K(\mu)}$ and the mixed Sobolev norm $\|g\|_{\mathscr S}$ are equivalent on $\mathscr S_2^{(1)}$.
\end{propositionE}

\begin{proofE}
The plan is to show that the sum of partial integrals $(\phi_s)_s \mapsto \sum_s \mathbf I_s\, \phi_s$ is a bijection from the direct sum $\bigoplus_s \mathscr L_2(\mu_s)$ onto the uniform HAR class $\mathscr K(\mu)$.
Its restriction to the open cube is the map $\Phi$, which \cref{lem:anchored-representation} makes a bijection onto the mixed Sobolev class, and the norm claim then follows from \cref{lem:anchored-norm}.

Since $X$ is uniform, each marginal distribution $\X_s$ of the covariates is the Lebesgue measure $\mu_s$ on the face $\mathcal X_s$, so the squared HAR norm $\|g\|^2_{\mathscr K(\mu)}$ is $\sum_s \|\mathrm{d}g_s/\mathrm{d}\mu_s\|^2_{\mathscr L_2(\mu_s)}$ for $g \in \mathscr K(\mu)$.
By the definition of the partial integrals $\mathbf I_s$, the representation \cref{eq:sectional-representation} writes a member $g$ of the uniform HAR class $\mathscr K(\mu)$ as $\sum_s \mathbf I_s\, (\mathrm{d}g_s/\mathrm{d}\mu_s)$, the sum of the partial integrals of its sectional derivatives.

Next we show that the sum of partial integrals, taken on all of $[0,1]^p$, is a bijection from the direct sum $\bigoplus_s \mathscr L_2(\mu_s)$ onto the uniform HAR class $\mathscr K(\mu)$, with inverse $g \mapsto (\mathrm{d}g_s/\mathrm{d}\mu_s)_s$.
For $(\phi_s)_s$ in the direct sum and each nonempty section $s$, the signed measure on the face $\mathcal X_s$ with the density $\phi_s$ with respect to $\mu_s$ is finite, because $\mu_s$ is a finite measure and $\phi_s \in \mathscr L_2(\mu_s)$.
So, by the \hyperref[thm:sectional-representation]{sectional representation theorem}, the sum $\sum_s \mathbf I_s\, \phi_s$ is right-continuous with the value $\phi_\emptyset$ at the origin and, for $s \neq \emptyset$, the sectional measures with the densities $\phi_s$ with respect to the Lebesgue measures $\mu_s$, which puts it in the uniform HAR class $\mathscr K(\mu)$ with the sectional derivatives $\phi_s$.
Conversely, every member of $\mathscr K(\mu)$ is the sum of the partial integrals of its sectional derivatives, as shown above.

The restriction of $\sum_s \mathbf I_s\, \phi_s$ to $(0,1)^p$ is $\Phi((\phi_s)_s)$, and, by \cref{lem:anchored-representation}, $\Phi$ is a bijection onto the mixed Sobolev class $\mathscr S_2^{(1)}$, where functions that are equal almost everywhere are identified.
Thus restriction to $(0,1)^p$, which is the inverse $g \mapsto (\mathrm{d}g_s/\mathrm{d}\mu_s)_s$ followed by $\Phi$, is a bijection from $\mathscr K(\mu)$ onto $\mathscr S_2^{(1)}$, which proves the first claim.
In this paragraph, given $g \in \mathscr S_2^{(1)}$, let $\tilde g$ be the member of $\mathscr K(\mu)$ that restricts to $g$, which by this bijection is the only member that is equal to $g$ almost everywhere.
It is continuous, because it is the sum of the partial integrals $\mathbf I_s\, (\mathrm{d}\tilde g_s/\mathrm{d}\mu_s)$ of its sectional derivatives.
For a nonempty section $s$ and $\phi \in \mathscr L_2(\mu_s)$, the partial integral $\mathbf I_s\, \phi$ is continuous in $x$ by dominated convergence.

For the norms, in this paragraph let $g$ be in the mixed Sobolev class $\mathscr S_2^{(1)}$, let $\tilde g$ be the member of $\mathscr K(\mu)$ that is equal to $g$ almost everywhere, and let $\phi_s = \mathrm{d}\tilde g_s/\mathrm{d}\mu_s$.
Then the HAR norm $\|\tilde g\|_{\mathscr K(\mu)}$ is the norm $\big(\sum_s \|\phi_s\|^2_{\mathscr L_2(\mu_s)}\big)^{1/2}$ of the family $(\phi_s)_s$ in the direct sum, and, by \cref{eq:sectional-representation}, $\Phi((\phi_s)_s)$ is the restriction of $\tilde g$ to $(0,1)^p$, which is equal to $g$ almost everywhere.
So \cref{lem:anchored-norm} bounds the mixed Sobolev norm $\|g\|_{\mathscr S}$ above and below by multiples of the HAR norm $\|\tilde g\|_{\mathscr K(\mu)}$, with factors that depend only on $p$, and the two norms are equivalent.
\end{proofE}

%% file: theory/variation-balls.tex
\begin{propositionE}[][text link={The proof is in \autoref{pf:variation-balls}.}, category=variation-balls]
\label{prop:variation-balls}
For every covariate distribution $\X$, the HAR class $\mathscr K(\mathbb X)$ is contained in the class $\mathscr M$ of right-continuous functions of bounded sectional variation, and
\[
\|g\|_{\mathscr M} \le 2^{p/2}\, \|g\|_{\mathscr K(\mathbb X)}
\qquad \text{for every } g \in \mathscr K(\mathbb X),
\]
so the ball $\mathscr K(\mathbb X)[r]$ of radius $r$ in the HAR class lies in the variation ball $\mathscr M[2^{p/2} r]$, for every $r > 0$.
In particular, both claims hold for the empirical HAR class $\mathscr K(\mathbb X_n)$, with the empirical distribution $\mathbb X_n$ in place of $\X$.

For the mixed Sobolev class $\mathscr S_2^{(1)}$ on $(0,1)^p$, there is a constant $C$ that depends only on $p$ such that, for every $g \in \mathscr S_2^{(1)}$, the member $\tilde g$ of the uniform HAR class $\mathscr K(\mu)$ that is equal to $g$ almost everywhere (\cref{prop:har-equals-mixed-sobolev}) is in $\mathscr M$, and its sectional variation norm is at most $C$ times the mixed Sobolev norm of $g$: $\|\tilde g\|_{\mathscr M} \le C\, \|g\|_{\mathscr S}$.
\end{propositionE}

\begin{proofE}
In this paragraph, let $g \in \mathscr K(\mathbb X)$.
For each nonempty section $s$, the sectional measure $g_s$ has the density $\mathrm{d}g_s/\mathrm{d}\X_s \in \mathscr L_2(\X_s)$ with respect to the marginal $\X_s$ restricted to the face $\mathcal X_s$, which gives the face mass at most one.
So, by the Cauchy-Schwarz inequality, $\int_{\mathcal X_s} |\mathrm{d}g_s/\mathrm{d}\X_s| \, \mathrm{d}\X_s \le \|\mathrm{d}g_s/\mathrm{d}\X_s\|_{\mathscr L_2(\X_s)}$.
The sectional measure $g_s$ is therefore a finite signed measure, and its total variation is this integral.
By the \hyperref[thm:sectional-representation]{sectional representation theorem}, $g$ is in $\mathscr M$, and its sectional variation norm is the sum of these total variations and $|g(0)|$.
The Cauchy-Schwarz inequality over the $2^p$ sections then gives
\begin{align*}
\|g\|_{\mathscr M}
&= \sum_s \int_{\mathcal X_s} \Big|\frac{\mathrm{d}g_s}{\mathrm{d}\X_s}\Big| \, \mathrm{d}\X_s
\le \sum_s \Big\|\frac{\mathrm{d}g_s}{\mathrm{d}\X_s}\Big\|_{\mathscr L_2(\X_s)} \\
&\le 2^{p/2} \Big(\sum_s \Big\|\frac{\mathrm{d}g_s}{\mathrm{d}\X_s}\Big\|^2_{\mathscr L_2(\X_s)}\Big)^{1/2}
= 2^{p/2}\, \|g\|_{\mathscr K(\mathbb X)},
\end{align*}
where the terms for $s = \emptyset$ are $|g(0)|$.
The argument uses only that $\X$ is a probability distribution on $[0,1]^p$, so it also applies to the empirical distribution $\mathbb X_n$.

For the mixed Sobolev class, in this paragraph let $g \in \mathscr S_2^{(1)}$, and let $\tilde g$ be the member of the uniform HAR class $\mathscr K(\mu)$ that is equal to $g$ almost everywhere, which exists and is unique by \cref{prop:har-equals-mixed-sobolev}.
By the same proposition, the norms $\|\cdot\|_{\mathscr K(\mu)}$ and $\|\cdot\|_{\mathscr S}$ are equivalent, so there is a constant $C'$ that depends only on $p$ such that $\|\tilde g\|_{\mathscr K(\mu)} \le C' \|g\|_{\mathscr S}$.
The first claim, with the uniform distribution $\mu$ in place of $\X$, puts $\tilde g$ in $\mathscr M$ and gives the first inequality of
\[
\|\tilde g\|_{\mathscr M}
\le 2^{p/2}\, \|\tilde g\|_{\mathscr K(\mu)}
\le 2^{p/2} C'\, \|g\|_{\mathscr S},
\]
and the norm equivalence gives the second.
So the claim holds with $C = 2^{p/2} C'$.
\end{proofE}

%% file: sections/30-theory/30-inductive-biases.tex
\subsection{Motivation and Inductive Biases}
\label{sec:inductive-biases}

There is no free lunch in machine learning: every method must have functions which it is ``bad'' at learning and others which it is ``good'' at learning.
Understanding what this inductive bias is for a machine learning method is essential for knowing when it will be useful (in conjunction with practical and computational considerations).
Concretely, when should we expect HAR to outperform or underperform other methods?
In this section we compare HAR qualitatively, based on theoretical understanding, to alternative methods.
We confirm these hypotheses in later simulations.

\subsubsection*{Bounding Sectional Variation}
The HAR estimator constrains sectional variation norm (\cref{prop:variation-balls}), much like HAL, mixed Sobolev KRR (MS-KRR), and, arguably, gradient boosting with trees \cite{ltb, Ki2026-xgboost}.
By the \hyperref[thm:sectional-representation]{sectional representation theorem}, a function of bounded sectional variation is a (possibly infinite) linear combination of HAR-like basis functions with knots placed anywhere in the unit cube.
What this means is that variation is ``cheap'' when it is axis-aligned: the function $1(x_1>c_1, x_2>c_2)$ is easy to learn, while $1(x_1 + x_2 > 1)$ is not.
Learners that penalize sectional variation are therefore highly sensitive to rotations or other multivariate transformations of the covariates.
This is good if the true function has axis-aligned variability, as one might expect in social science or clinical settings with tabular data, and bad otherwise, as one might expect when classifying images from raw pixel data.

There is also a sense in which ``interactions'' are specifically penalized.
Given a univariate function $g(x)$ with $g(0)=0$ (for simplicity) and sectional variation $\|g\|_{\mathscr M}$, the sectional variation of the product $\prod^k_j g(x_j)$ is $\|g\|_{\mathscr M}^k$.
Relative to an additive function $\sum g_j$, which has sectional variation $k\|g\|_{\mathscr M}$, the cost of the single ``interaction term'' $\prod^k_j g(x_j)$ grows exponentially in dimension as long as $\|g\|_{\mathscr M} > 1$.
Thus, while departures from additivity are allowed ($1(x_1>c_1, x_2>c_2)$ is non-additive and cheap), the non-additive components cannot vary in multiple dimensions simultaneously without incurring great cost.

\subsubsection*{Penalizing Derivative Norms}
Relative to HAL and tree boosting, HAR and MS-KRR also penalize $\mathscr L_2$ norms of certain derivatives.
This adds smoothness constraints to the functions of bounded sectional variation.
Thus HAL and tree boosting should be better at fitting functions that have more abrupt changes or jumps, whereas HAR and MS-KRR are better suited for gradual changes (still along the axes). 

\subsubsection*{Density-Sensitivity and Anchoring}
HAR and MS-KRR are very closely related, but not identical, so they have separate inductive biases.

The first difference is \textit{where} the derivatives are penalized.
Recall that the mixed Sobolev class $\mathscr S_2^{(1)}$ is nearly equivalent to $\mathscr K(\mu)$, and $\mathscr K(\mathbb X_n)$ approximates $\mathscr K(\mathbb X)$. 
The difference between $\mu$ and $\mathbb X$ creates a weighting effect: \cref{eq:har-vs-mskrr} shows that the HAR penalty weights each squared mixed derivative by the reciprocal of the joint density of the covariates (in each section).
Therefore, relative to MS-KRR, HAR's inductive bias is towards interaction or variation that is strong in regions of high covariate density.
In contrast, MS-KRR should be better at learning variation in low-density regions.

The second difference is HAR's ``anchor''. 
Even if $\mathbb X = \mu$, HAR and the mixed Sobolev class still use \textit{different norms}.
As previously discussed, the function $x_1 x_2$ has squared norm $1$ in the uniform HAR class, while its reflection $(1 - x_1)(1-x_2)$ has squared norm $4$.
MS-KRR penalizes these two functions equally.
So, relative to MS-KRR, HAR's (and HAL's) inductive bias is towards interactions that switch ``on'' as their covariates rise from their lowest values.
This can be an advantage in cases where the strongest interactions are understood and variables can be coded in the proper orientation (e.g. the combination of increasing age and increasing obesity super-additively increasing heart attack risk), and a disadvantage otherwise.

%% file: sections/30-theory/40-convergence-rate.tex
\subsection{Convergence Rate}

We now state our formal convergence rate result. 
When the true function lies in $\mathscr K(\mathbb X)$ and the norm bound is suitably chosen, HAR achieves a rate of convergence that is dimension-free, up to log factors.

\input{theory/main-rate}
Essentially, the condition on the 2-norm bound $M^2/n$ ensures that the optimizer is in the set of right-continuous functions with bounded sectional variation so we inherit the same metric entropy and rate of convergence from that class.

The required conditions on the loss $\ell$ are mild and  satisfied by squared error loss with a bounded outcome and by binary log loss (\cref{ex:squared-error,ex:logistic-regression}).
The density condition is very mild and explicitly allows point masses.
It holds of course when $\X$ has a bounded density with respect to Lebesgue measure, but also, for example, when $\X$ has a bounded density with respect to the product of its own marginal distributions, whatever they may be.

\paragraph*{Growing Norms}
In the proof of \cref{thm:main-rate}, the radius $M$, which must exceed the HAR norm $\|f\|_{\mathscr K(\mathbb X)}$ of the truth, fixes the variation ball of radius $2^{p/2}M$ that holds the estimate and the truth (\cref{prop:variation-balls}), and the entropy, Lipschitz and sup-norm bounds of the proof are bounds over this ball.
In this paragraph, the truth $f$ and the radius $M$ may change with $n$.
The proof can be made explicit in $M$, so under squared error loss with a bounded outcome we expect the HAR estimate to satisfy $\|\hat f_n - f\| = O_{\P}(M n^{-1/3}(\log n)^{2(p-1)/3})$, with a constant that does not depend on $M$, for radii $M \ge 1$ that keep this bound tending to zero, as long as the oracle approximation of \cref{sec:proof-rate} stays in the model with probability tending to one.
\citet[Theorem 1]{Bibaut2019-zf} prove this dependence for HAL with a variation bound that grows with $n$.

\subsubsection*{Comparison to Mixed Sobolev Kernel Ridge}
In the regression setting of this paper, the minimax rate over balls of $\mathscr S_2^{(1)}$ is, to our knowledge, not fully established.
However, \citet[Theorem 3]{liu2023derivatives} proves an in-expectation $n^{-1/3} (\log n)^{(p-1)/3}$ lower bound in a model where the covariate density is bounded above and away from zero and errors are sub-Gaussian.\footnote{
Outside the regression setting, \citet[Theorem 1.1]{tps-anova} establishes the same rate in a Gaussian white noise model.}
This rate is not directly comparable to our result in \cref{thm:main-rate} since our rate is in probability.

Results in the literature also come close to proving a matching upper bound.
We suspect that combining \citet[Theorem 1]{krr-rates-2} and \citet[Lemma 2.3]{tps-anova} gives MS-KRR the $\mathscr L_2(\mathbb X)$ rate $O(n^{-1/3} (\log n)^{(p-1)/3})$ in expectation uniformly over a model where  $X$ is uniform, $f$ lies in a ball of $\mathscr S_2^{(1)}$ in the norm of $K_{\mathscr S}$, and the conditional variance of the noise is at most a fixed constant.
Markov's inequality then gives the same upper bound in probability.
Verifying this and establishing an in-probability lower bound over a sensibly large regression model is left to future work.

\subsubsection*{Comparison to Highly Adaptive Lasso}
In its constrained form, HAL minimizes the empirical risk over a ball of the sectional variation norm $\|\cdot\|_{\mathscr M}$.
With a bounded outcome, covariates whose density with respect to Lebesgue measure is bounded above, and a norm bound at least the sectional variation norm $\|f\|_{\mathscr M}$ of the truth, \citet[Theorem 1, Proposition 2 and Section 4.1]{Bibaut2019-zf} show that HAL attains the $\mathscr L_2(\mathbb X)$ rate $O_{\P}(n^{-1/3} (\log n)^{2(p-1)/3})$, the same as our rate in \cref{thm:main-rate}.
As stated, their result is pointwise, but the constants only depend on the dimension, variation bound, outcome bound, and bound on the covariate density, so it holds uniformly under a model respecting those bounds.
The proof of \cref{thm:main-rate} gives HAL the same rate under the weaker density condition of our theorem, since the HAL estimate and the truth both lie in a variation ball.
\citet[Theorem 4.5]{Fang2019-or} give an in-probability upper bound\footnote{
The bound is stated in expectation in the paper; Markov's inequality makes it hold in probability as well.
} of $n^{-1/3}(\log n)^{(2p-1)/6}$ at the design points in a setting with a lattice design and Gaussian errors, which is a log factor faster than the rate in \citet{Bibaut2019-zf} when $p>1$. 

There is not a lower bound in the literature that matches our in-probability, random-covariate setting.
For covariates on an equally spaced lattice and Gaussian errors, \citet[Theorem 4.6]{Fang2019-or} show that, for $p \ge 2$, no estimator has a smaller worst-case root mean squared error over a variation ball, at the design points, than order $n^{-1/3} (\log n)^{(p-1)/3}$.
That rate is a factor of $(\log n)^{(p-1)/3}$ below the rate of HAL.
However, this model neither is contained in nor contains the model of \citet[Theorem 1, Proposition 2 and Section 4.1]{Bibaut2019-zf}, and moreover the bound is in expectation rather than in probability.
Since norm balls in the mixed Sobolev class lie in sectional variation balls (\cref{prop:variation-balls}), lower bounds for models with the mixed Sobolev class transfer too.
Therefore the $n^{-1/3} (\log n)^{(p-1)/3}$ rate from \citet[Theorem 3]{liu2023derivatives} is an in-expectation lower bound in a model where the covariate density is bounded above and away from zero and errors are sub-Gaussian.

\subsubsection*{Rates in $\mathscr K(\mathbb X)$}
By itself, our HAR rate does not upper-bound any minimax rate for a model where $\mathscr K(\mathbb X)$ is the model for the outcome regression because our results are pointwise rather than uniform. 
Therefore we cannot say that HAR is minimax optimal, or even within a specific tolerance of the minimax rate.
However, we do believe our results can be strengthened to uniform rates in mostly the same way that the HAL rate can be made uniform. 
The difference that poses a challenge is that HAL performs exact empirical minimization whereas for HAR the proof relies on an intermediate oracle. 
Addressing this difficulty is left to future work.

Even without uniform rates for HAR, however, the HAL estimator already provides an upper bound for the minimax rate in the model considered by \citet{Bibaut2019-zf}.
Since a ball $\mathscr K(\mathbb X)[r]$ of radius $r$ in the HAR class lies in the variation ball $\mathscr M[2^{p/2} r]$, and HAL attains its rate uniformly over the latter model, the same rate is an upper bound for the in-probability minimax rate in $\mathscr K(\mathbb X)[r]$. 
This bound is even uniform over $\mathbb X$ since the model under $\mathscr M[2^{p/2} r]$ is agnostic to the covariate distribution up to regularity conditions.

In terms of a lower bound, since the mixed Sobolev class $\mathscr S_2^{(1)}$ is the same set of functions as $\mathscr K(\mu)$ up to equivalence and equivalent norms, lower bounds from the former class transfer to the latter.
However, these bounds may not transfer across all covariate distributions, and moreover are rates in expectation rather than in probability. 
Nonetheless, we conjecture that HAR is within log factors of minimax-optimal in a variety of sensible models where $f$ is in $\mathscr K(\mathbb X)$ for some general class of covariate distributions.

\subsubsection{Rate Under Cross-Validation}
\label{sec:cv-rate}

In \cref{thm:main-rate}, the radius $M$ of the HAR model must exceed the norm $\|f\|_{\mathscr K(\mathbb X)}$ of the truth, which is unknown.
The analyst also needs the factor $1/n$ in the bound $\tilde M^2 = M^2/n$ on the coefficients.
The next corollary shows that cross-validation over a fixed number of folds can choose the bound on the coefficients from an adequate grid, with no loss in the rate and with no knowledge of the factor $1/n$ or of the norm $\|f\|_{\mathscr K(\mathbb X)}$.

We write the bound on the coefficient vector $\beta$ of \cref{eq:har-model} as $\|\beta\|_2 \le \tilde M$, so that $\tilde M^2 = M^2/n$, and the working model is the ball of fixed radius $M = \sqrt n\,\tilde M$ in the empirical HAR class $\mathscr K(\mathbb X_n)$.
Let $\tilde M_{n,1} < \dots < \tilde M_{n,J_n}$, with $J_n$ values, be a grid of coefficient bounds.
We choose a bound from this grid by cross-validation with a fixed number $V \ge 2$ of folds.
The folds are formed at random, independently of the observations, and each has $\floor{n/V}$ or $\ceil{n/V}$ observations, so that all folds have the same size when $V$ divides $n$.

In the context of the cross-validation results, for each grid index $k$ and each fold $v$, let $\hat f_{n,k,v}$ be the HAR fit with the bound $\|\beta\|_2 \le \tilde M_{n,k}$ on its coefficient vector $\beta$, estimated based on the observations outside fold $v$.
Let the index $k^*_n$ minimize the cross-validated risk, the average over the folds $v$ of the empirical risk of $\hat f_{n,k,v}$ on fold $v$, and let the estimate be the average $\bar f_n = V^{-1}\sum_v \hat f_{n,k^*_n,v}$ of the fits at that index.

\input{theory/cv-rate}
The grid condition only requires one value at the scale $1/\sqrt n$.
Choosing a coefficient bound $\tilde M$ is choosing the radius $M = \sqrt n\,\tilde M$ of \cref{thm:main-rate}, so cross-validation finds the scale $1/\sqrt n$ and the constant $M$ jointly.

As in the super learner of \citet[Section 5]{hal-og}, the estimate is the average of the fits on the training samples at the selected bound, and not a refit on the full sample with that bound.
The cross-validation oracle inequality of \citet{cv-oracle}, on which the proof rests, controls only the fits on the training samples, and our proof does not cover a refit on the full sample at the selected bound.

Our proof covers cross-validation with a fixed number of folds, but in our experiments we use leave-one-out cross-validation instead because it has a closed-form expression that makes computation cheaper (see \cref{sec:computation}).
We refit on the full sample with the selected penalty to reflect standard practice.

In practice, we recommend tuning the interaction depth or the decay weights in addition to the bound.
\Cref{cor:cv-rate} still holds when cross-validation chooses the bound and the depth or the decay jointly as long as the grid satisfies the conditions of the corollary.

\paragraph*{Constrained vs. Penalized HAR}
Our convergence results are formulated in terms of the $M$-constrained, rather than $\lambda$-penalized, HAR estimator.
In practice, it is the penalized version that is used and it is the penalty $\lambda$ that is selected.
While, for a fixed dataset, there is a monotone relationship $\lambda(M)$ that produces the same fit that a bound $M$ does, that relationship is dataset-specific so asymptotic arguments about the constrained estimator, and cross-validation in particular, do not immediately carry over to the penalized estimator.

We nonetheless expect that, under squared error loss, the penalized estimator attains the same rate if the penalty (as defined in \cref{sec:computation}) scales like $\lambda_n \asymp n^{4/3}(\log n)^{4(p-1)/3}$.
The hypothesized scaling follows from the HAR rate and the penalized least squares theory of \citet[Theorem~10.2]{vandegeer2000}.
However, that theory assumes a fixed design and a fixed penalty functional, so it is not immediately applicable. 
Nonetheless, the intuition does hold for MS-KRR \cite{krr-rates-2, tps-anova} so we expect it to hold for HAR as well.
We also expect the rate to be retained under our cross-validation setup over a grid with sufficiently large options as $n$ increases.

%% file: theory/main-rate.tex
\begin{theoremE}[][text link={The proof is in \autoref{pf:main-rate}.}, category=main-rate]
\label{thm:main-rate}
Fix $\mathbb P$. Define the ``truth'' $f = \argmin_{\{g:[0,1]^p \to \mathds R\}} \P \ell(g)$ for a loss function $\ell$.
For a constant $M > 0$, let $\hat f_n = \argmin_{g \in \mathscr K(\mathbb X_n)[M]} \P_n \ell(g)$ be the HAR estimate using bound $M^2/n$ on the coefficients $\beta$.

Suppose the following regularity conditions hold:
\begin{enumerate}
\item \label{cond:main-rate-loss} The loss function is Lipschitz in $g(x)$, with a Lipschitz constant that may depend on the variation bound of $g$ but is the same for all $(x,y)$ in the support of $\P$.
Also, $\|g-f\|^2 \lesssim \P[\ell(g) - \ell(f)] \lesssim \|g-f\|^2$ for every right-continuous $g$ of bounded sectional variation, with constants that may depend on the variation bound.
\item \label{cond:main-rate-density} The covariate distribution $\X$ has a density that is bounded above with respect to some product of probability measures on $[0,1]$, one for each coordinate.
\end{enumerate}

Then, if $f$ is right-continuous with sectional derivatives $\mathrm{d}f_s/\mathrm{d}\X_s \in \mathscr L_2$ for all $s$ and $M$ is large enough so 
$f\in \mathscr K(\mathbb X)[M-\varepsilon]$ for some $\varepsilon > 0$,
the HAR estimator attains the $\mathscr L_2(\mathbb X)$ rate
$$\|\hat f_n -f \| = O_{\P}(n^{-1/3} (\log n)^{2(p-1)/3}).$$
\end{theoremE}

\begin{proofE}
The truth is in the ball $\mathscr K(\mathbb X)[M - \varepsilon]$, so its norm $\|f\|_{\mathscr K(\mathbb X)}$ is at most $M - \varepsilon$, which is less than $M$.
In this proof, let $\Xi = [M, M]$ be the interval that holds only the radius $M$.
So \cref{lem:rate-window} applies with this interval, and its estimate $\hat f_{n,M}$ is the HAR estimate $\hat f_n$.
The lemma gives $\|\hat f_n - f\| = O_{\P}(n^{-1/3}(\log n)^{2(p-1)/3})$.
\end{proofE}

%% file: theory/cv-rate.tex
\begin{corollaryE}[][text link={The proof is in \autoref{pf:main-rate}.}, category=cv-rate]
\label{cor:cv-rate}
Let $\bar f_n$ be the cross-validated HAR estimate, with $V$ folds and the grid $\tilde M_{n,1} < \dots < \tilde M_{n,J_n}$ of coefficient bounds.
Suppose that the following conditions hold.
\begin{description}
\item[(Loss and density)] Conditions \ref{cond:main-rate-loss} and \ref{cond:main-rate-density} of \cref{thm:main-rate} hold, and the loss is convex in the value $g(x)$.
\item[(Grid)] The grid may depend on $n$ but not on the data, and its size $J_n$ grows at most polynomially in $n$.
\item[(Bounded values)] For a constant $\kappa$, the truth $f$ and all the fits $\hat f_{n,k,v}$ take values in $[-\kappa, \kappa]$, and condition \ref{cond:main-rate-loss} holds for all functions with values in this interval, with $\kappa$ in place of the variation bound.
\end{description}
If the truth $f$ is in the HAR class $\mathscr K(\mathbb X)$ and there is an interval $[\underline M, \bar M]$ that does not depend on $n$ such that, for all large $n$, some grid index $k_n$ has
\[
\frac{\|f\|_{\mathscr K(\mathbb X)}}{\sqrt{1 - 1/V}} < \underline M \le \sqrt n\,\tilde M_{n,k_n} \le \bar M,
\]
then
\[
\|\bar f_n - f\| = O_{\P}\bigl(n^{-1/3}(\log n)^{2(p-1)/3}\bigr).
\]
\end{corollaryE}

\begin{proofE}
Write $R(g) = \P[\ell(g) - \ell(f)]$ for the excess risk of a function $g$, which is nonnegative because $f$ minimizes the risk.
For a grid index $k$, write $R_{\mathrm{av}}(k) = V^{-1}\sum_v R(\hat f_{n,k,v})$ for the average excess risk of the training fits at that index.
Write $\P_{n,v}$ for the empirical distribution of fold $v$, so that the cross-validated risk at the index $k$ is $V^{-1}\sum_v \P_{n,v} \ell(\hat f_{n,k,v})$.
Let $r_n = n^{1/3}(\log n)^{-2(p-1)/3}$ be the inverse rate.
Write $n_v$ and $m_v$ for the numbers of observations outside and inside fold $v$.
The fold has at most $\ceil{n/V} \le (n+V-1)/V$ observations, so $n \ge n_v \ge (1 - 1/V)(n-1) \ge (n-1)/2$, where the last step uses $V \ge 2$, and it has $m_v \ge \floor{n/V}$ observations.
The partition into folds is independent of the observations, and the bounds below hold for every partition with these fold sizes, so we treat the partition as fixed.
Then the training sample of fold $v$, the observations outside it, is an IID sample of size $n_v$, and the fold is an IID sample of size $m_v$ that is independent of its training sample.

The proof compares the selected index $k^*_n$ with the index $k_n$ of the statement.
The fits at $k_n$ attain the rate by \cref{lem:rate-window}, and cross-validation loses against them at most a factor three and a remainder that is $O_{\P}(r_n^{-2})$, of the order of the squared rate, by \cref{lem:cv-inequality,lem:cv-maximal}.
Step~\ref{step:cv-selected}, which bounds the remainder, is the key step.

\begin{enumerate}
\item \label{step:cv-bounds}
\emph{There is a constant $b$ such that every training fit $g = \hat f_{n,k,v}$ has $|\ell(g) - \ell(f)| \le b$ on the support of $\P$ and $\P(\ell(g) - \ell(f))^2 \le b\,R(g)$.}
By (Bounded values), the truth and the fits take values in $[-\kappa, \kappa]$, and condition \ref{cond:main-rate-loss} holds for functions with values in this interval, with a Lipschitz constant that we write $C_\kappa$.
So, on the support of $\P$, every training fit $g$ has $|\ell(g) - \ell(f)| \le C_\kappa|g - f| \le 2\kappa C_\kappa$.
Squaring and integrating gives $\P(\ell(g) - \ell(f))^2 \le C_\kappa^2\|g - f\|^2$, and the lower bound of condition \ref{cond:main-rate-loss} bounds the squared error $\|g - f\|^2$ by a constant $C'_\kappa$ times the excess risk $R(g)$.
Thus we may choose $b = \max\{2\kappa C_\kappa, C_\kappa^2C_\kappa'\}$ and obtain 
$$
|\ell(g) - \ell(f)| \le b \quad \text{and} \quad \mathbb P(\ell(g) - \ell(f))^2 \le bR(g).
$$

\item \label{step:cv-fixed}
\emph{We want to show that $R_{\mathrm{av}}(k_n) = O_{\P}(r_n^{-2})$.}
Choose a constant $M'$ with
\[
\|f\|_{\mathscr K(\mathbb X)} < M' < \sqrt{1 - 1/V}\,\underline M,
\]
which exists by the display of the statement, and define the interval $\Xi = [M', \bar M]$, with the upper end $\bar M$ of the interval of the statement.
By (Grid), the grid does not depend on the data, so the index $k_n$ is not random.
Fix a fold $v$.
Its training fit $\hat f_{n,k_n,v}$ is the HAR estimate of \cref{lem:rate-window} on the training sample, with $n_v$ in place of $n$ and with the effective radius $M_{n,v} = \sqrt{n_v}\,\tilde M_{n,k_n}$, because its bound $\tilde M_{n,k_n}$ on $\|\beta\|_2$ is this radius divided by $\sqrt{n_v}$.
By the display of the statement and $n_v \le n$, the effective radius satisfies
\[
\sqrt{n_v/n}\,\underline M \le M_{n,v} = \sqrt{n_v/n}\,\sqrt n\,\tilde M_{n,k_n} \le \bar M,
\]
and $n_v/n \to 1 - 1/V$, because the fold has $\floor{n/V}$ or $\ceil{n/V}$ observations.
So, for all large $n$, the effective radius $M_{n,v}$ is in the interval $\Xi$, and the norm of the truth is less than its left end $\min \Xi = M'$.
Thus \cref{lem:rate-window} applies to the training fit at the index $k_n$, and its bound is uniform over the radii in $\Xi$, so
\[
\|\hat f_{n,k_n,v} - f\| = O_{\P}(r_{n_v}^{-1})
\]
as $n$, and with it $n_v$, tends to infinity.
By the bounds on $n_v$, $n_v \asymp n$, so the inverse rates satisfy $r_{n_v} \asymp r_n$, and hence $\|\hat f_{n,k_n,v} - f\| = O_{\P}(r_n^{-1})$.
By (Bounded values), the upper bound of condition \ref{cond:main-rate-loss} holds for the fit, so
\[
R(\hat f_{n,k_n,v}) = \P[\ell(\hat f_{n,k_n,v}) - \ell(f)] \lesssim \|\hat f_{n,k_n,v} - f\|^2 = O_{\P}(r_n^{-2}).
\]
Averaging over the $V$ folds gives
\[
R_{\mathrm{av}}(k_n) = \frac1V\sum_{v=1}^V R(\hat f_{n,k_n,v}) = O_{\P}(r_n^{-2}).
\]

\item \label{step:cv-selected}
\emph{We want to show that $R_{\mathrm{av}}(k^*_n) = O_{\P}(r_n^{-2})$.}
For a grid index $k$ and a fold $v$, let $u_{k,v} = \ell(\hat f_{n,k,v}) - \ell(f)$ be the centered loss of the training fit, so that $\P u_{k,v} = R(\hat f_{n,k,v})$.
The centered losses give the same selector $k^*_n$ as the losses, because the two cross-validated risks differ by $V^{-1}\sum_v \P_{n,v}\ell(f)$, which does not depend on $k$.
So \cref{lem:cv-inequality}, with the index $k_n$, gives
\[
R_{\mathrm{av}}(k^*_n) \le 3R_{\mathrm{av}}(k_n) + \frac1V\sum_v \bigl(\zeta_v^+ + \zeta_v^-\bigr),
\]
with the remainders $\zeta_v^\pm$ of that lemma.
By step~\ref{step:cv-fixed}, it is enough to show that each remainder $\zeta_v^\pm$ is $O_{\P}(r_n^{-2})$, of the order of the squared rate.
\begin{enumerate}
\item \label{step:cv-conditional}
\emph{Given the training sample of fold $v$, the remainder $\zeta_v^+$ has conditional expectation at most $64\,b\log(1+J_n)/m_v$, with the constant $b$ of step~\ref{step:cv-bounds}.}
Given the training sample, the fold is an IID sample of size $m_v$ from $\P$, and the $J_n$ centered losses $u_{k,v}$ are fixed functions, because the grid does not depend on the data, by (Grid).
By step~\ref{step:cv-bounds}, they satisfy the conditions of \cref{lem:cv-maximal} with the constant $b$.
The remainder is $\zeta_v^+ = 2\max_k\{(\P_{n,v} - \P)u_{k,v} - \P u_{k,v}/2\}$, so \cref{lem:cv-maximal} with $\delta = 1/2$ bounds its conditional expectation by $2 \cdot 8(1 + 3)\,b\log(1+J_n)/m_v$.

\item \label{step:cv-remainder}
\emph{We want to show that $\zeta_v^+ = O_{\P}(r_n^{-2})$.}
The remainder can be negative, so Markov's inequality does not apply to it directly, and we add a nonnegative term to it.
In this proof, let $T_v = (\P_{n,v} - \P)u_{k_n,v}$ be the deviation of the fold average of the centered loss at the index $k_n$ from its mean, and let $W_v = \zeta_v^+ + 2|T_v| + R(\hat f_{n,k_n,v})$.
The maximum that defines $\zeta_v^+$ is at least its term at $k_n$, which is $2T_v - R(\hat f_{n,k_n,v})$, so the variable $W_v$ is nonnegative:
\[
W_v \ge 2T_v - R(\hat f_{n,k_n,v}) + 2|T_v| + R(\hat f_{n,k_n,v}) = 2T_v + 2|T_v| \ge 0.
\]
Given the training sample, the deviation $T_v$ is an average of $m_v$ IID terms with mean zero, so, for an observation $(X,Y)$ from $\P$, its conditional variance is
\[
\frac{\operatorname{Var} u_{k_n,v}(X,Y)}{m_v} \le \frac{\P u_{k_n,v}^2}{m_v} \le \frac{b\,R(\hat f_{n,k_n,v})}{m_v},
\]
where the last step is step~\ref{step:cv-bounds}.
By the Cauchy--Schwarz inequality, the conditional expectation of the absolute deviation $|T_v|$ is at most $\{b\,R(\hat f_{n,k_n,v})/m_v\}^{1/2}$, and the inequality of arithmetic and geometric means gives
\[
2\Bigl\{\frac{b\,R(\hat f_{n,k_n,v})}{m_v}\Bigr\}^{1/2} \le R(\hat f_{n,k_n,v}) + \frac{b}{m_v}.
\]
With step~\ref{step:cv-conditional}, the conditional expectation of $W_v$ is at most
\[
\frac{64\,b\log(1+J_n) + b}{m_v} + 2R(\hat f_{n,k_n,v}).
\]
The grid size $J_n$ grows at most polynomially in $n$, by (Grid), so $\log(1+J_n) = O(\log n)$.
The number $V$ of folds is fixed and the fold size $m_v$ is at least $\floor{n/V}$, so $m_v \asymp n$, and hence the first term is
\[
\frac{64\,b\log(1+J_n) + b}{m_v} = O\Bigl(\frac{\log n}{n}\Bigr) = o(r_n^{-2}),
\]
of smaller order than the squared rate $r_n^{-2}$.
The second term, twice the excess risk of the fit at the index $k_n$, is $O_{\P}(r_n^{-2})$, of the order of the squared rate, by step~\ref{step:cv-fixed}.
By Markov's inequality applied given the training sample, a nonnegative variable whose conditional expectation is $O_{\P}(r_n^{-2})$, of the order of the squared rate, is itself $O_{\P}(r_n^{-2})$.
So the nonnegative variable $W_v$ is $O_{\P}(r_n^{-2})$.
By the bound above on the absolute deviation $|T_v|$ of the fold average, the nonnegative variable $2|T_v| + R(\hat f_{n,k_n,v})$ has conditional expectation at most twice the excess risk $R(\hat f_{n,k_n,v})$ of the fit at the index $k_n$ plus $b/m_v$, which is $O_{\P}(r_n^{-2})$ by step~\ref{step:cv-fixed} and $m_v \asymp n$, so this variable is $O_{\P}(r_n^{-2})$ as well, by Markov's inequality applied given the training sample.
The remainder $\zeta_v^+$ is at most the nonnegative variable $W_v$, because $2|T_v| + R(\hat f_{n,k_n,v}) \ge 0$, and at least $2T_v - R(\hat f_{n,k_n,v}) \ge -2|T_v| - R(\hat f_{n,k_n,v})$, by the display above.
Both bounds are $O_{\P}(r_n^{-2})$, so the remainder $\zeta_v^+$ is $O_{\P}(r_n^{-2})$.
\end{enumerate}
The same argument, with the bound of \cref{lem:cv-maximal} for $\P - \P_m$ in place of $\P_m - \P$, shows that the other remainder $\zeta_v^-$ is $O_{\P}(r_n^{-2})$ as well, since this remainder is also at least $-2|T_v| - R(\hat f_{n,k_n,v})$, with the deviation $T_v$ and the excess risk of the fit at the index $k_n$.

Recall that $R_{\mathrm{av}}(k^*_n) \le 3R_{\mathrm{av}}(k_n) + V^{-1}\sum_v(\zeta_v^+ + \zeta_v^-)$.
Both terms on the right are $O_{\P}(r_n^{-2})$: the first by step~\ref{step:cv-fixed}, and the second because it averages $2V$ remainders that are each $O_{\P}(r_n^{-2})$, and the number $V$ of folds is fixed.
So $R_{\mathrm{av}}(k^*_n) = O_{\P}(r_n^{-2})$.

\item \label{step:cv-average}
\emph{The average $\bar f_n$ attains the rate: $\|\bar f_n - f\| = O_{\P}(r_n^{-1})$.}
The loss is convex in the value $g(x)$, by (Loss and density), so, for every observation $(x,y)$, Jensen's inequality gives
\[
\ell(\bar f_n)(x,y) \le \frac1V\sum_{v=1}^V \ell(\hat f_{n,k^*_n,v})(x,y).
\]
Integrating over a new observation gives $\P\ell(\bar f_n) \le V^{-1}\sum_v \P\ell(\hat f_{n,k^*_n,v})$, and subtracting the risk $\P\ell(f)$ of the truth from both sides gives
\[
R(\bar f_n) \le R_{\mathrm{av}}(k^*_n) = O_{\P}(r_n^{-2}),
\]
where the last step is step~\ref{step:cv-selected}.
Finally, the average takes values in $[-\kappa, \kappa]$, because an average of functions with values in this interval stays in it.
So the lower bound of condition \ref{cond:main-rate-loss}, quadratic growth, holds for the average by (Bounded values), and it gives $\|\bar f_n - f\|^2 \lesssim R(\bar f_n)$.
Therefore $\|\bar f_n - f\|^2 = O_{\P}(r_n^{-2})$, so $\|\bar f_n - f\| = O_{\P}(r_n^{-1})$, which is the rate of the statement.
\end{enumerate}
\end{proofE}

%% file: sections/40-demonstration/00-demonstration.tex
\section{Demonstration}

In this section we compare the performance of HAR against other estimators, highlighting qualitative and quantitative differences that support our theoretical understanding in both simulation and empirical application.
All code for the implementations, simulations, and empirical demonstrations can be found at \url{https://github.com/alejandroschuler/HighlyAdaptiveRidge}.

\subsection{Estimators}
\label{sec:estimators}
The estimators we compared and their tuning setups are as follows.
Where an estimator used cross-validation, it used the same \mthFolds folds of the training data as every other estimator.
No intercepts were penalized by any of the kernel methods.
Each kernel method used the same technique to generate a grid of regularization strengths $\lambda$, run for each generated kernel matrix.
Specifically, we generated \mthKrrNAlphas penalties evenly spaced on the log scale from \mthKrrFloor times the largest eigenvalue of the kernel matrix up to the largest penalty of the method of appendix \ref{sec:krr-reg}, with tolerance \mthKrrEps, which is designed to give a near-constant prediction function.

\paragraph*{HAR}
We used the section weights $w_s = 1(|s| \le D)$ of \cref{sec:weights} for nonempty sections $s$, with no decay ($\rho=1$) and depth $D$ tuned by early-stopping cross-validation over \mthHarDepths (where less than $p$), and then $D = p$, which is plain HAR.
At each depth and in each fold, we built the HAR kernel matrix of the fold's training rows, chose the penalty by leave-one-out cross-validation on those rows (penalty grid set per-depth and per-fold via method above), and predicted the fold's held-out rows with that fit.
We stopped after \mthDepthPatience depths in a row that did not lower the smallest cross-validated error so far.
Last, we refit at the chosen depth on all the training data, with the penalty again chosen by leave-one-out cross-validation over a grid set from the kernel matrix of all the training rows.

\paragraph*{Mixed Sobolev KRR}
Kernel ridge regression (KRR) with the mixed Sobolev kernel is the most similar method to HAR and thus the most apt comparator.
We used the weighted mixed Sobolev kernel $K_{\mathscr S, w}$ of \cref{sec:mixed-sobolev} with the section weights of HAR, $w_s = 1(|s| \le D)$ for nonempty sections $s$.
We tuned it exactly as we tuned HAR: the depth $D$ by cross-validation with early stopping along the same path, the penalty by leave-one-out cross-validation within each fold, and a final refit at the chosen depth.
As for HAR, the full depth $D = p$ gives plain mixed Sobolev KRR.
Since the kernel is defined on the unit cube, we min-max scaled each covariate to $[0,1]$ on the training data, and clipped new points to $[0,1]$.

\paragraph*{Radial Basis KRR}
To showcase the impact of choosing a kernel (explicit or approximated/learned) that controls sectional variation, we compared against a kernel ridge regression using the radial basis kernel \cite{Hastie2009-ix}, which does not.
We tuned it like HAR, with the bandwidth in place of the depth: the kernel is $\exp(-2^k \|x - x'\|_2^2 / p)$ for an integer $k$, which the cross-validation searched over $\mthRbfScaleMin, \dots, \mthRbfScaleMax$ without early stopping.
We standardized the covariates on the training data, and, as for the other kernel methods, we left the outcome as it was.

\paragraph*{HAL}
HAR is the $\mathscr L_2$ equivalent of HAL, which essentially controls the $\mathscr L_1$ norm of the sectional derivatives.
It is thus an apt comparison, even though it suffers computationally.
We used our own fast implementation of zero-order HAL that leverages bitwise operations.
Inspired by the R package hal9001 \cite{hal9001}, we used early stopping over the depth $D$, keeping only bases corresponding to interactions of order $\le D$.
At each depth, \mthFolds-fold cross-validation selected over a grid of \mthHalNAlphas penalties, evenly spaced on the log scale from the smallest penalty that sets every coefficient to zero down to \mthHalEps times it.
We stopped after \mthHalPatience depth that did not lower the smallest cross-validated error so far, and also before a depth whose predicted time would take the fit past \mthHalBudget minutes, where the predicted time of a depth was the time of the depth before it, scaled in proportion to the number of basis functions (the first depth always ran).
We did not run HAL on yearmsd and slice because our fast implementation handles at most 63 covariates (sections are stored as the bits of a 64-bit integer).

\paragraph*{Gradient Boosted Trees}
Gradient boosted trees are popular and have a similar inductive bias to HAR and HAL \cite{ltb, Ki2026-xgboost}.
We included them as a high-performing baseline due to their reputation as one of the best out-of-the-box supervised learning methods \cite{Grinsztajn2022-pz, Shwartz-Ziv2022-lm}.
We used xgboost \cite{xgboost}, version \mthXgboost, with squared error loss and every setting at its default, except that a leaf had no minimum size (\texttt{min\_child\_weight} set to zero).
We tuned the depth of the trees over \mthGbtDepths, the learning rate over \mthGbtRates, and the number of trees, by \mthFolds-fold cross-validation.
We used early stopping over the number of trees with patience set to \mthGbtPatience trees and a maximum of \mthGbtMaxTrees trees.

\paragraph*{Elastic Net}
Lastly, we included the elastic net \cite{Zou2005-elasticnet} as a rudimentary baseline so that the scale of performance differences between the other methods is properly contextualized.
We used the elastic net of scikit-learn \cite{scikit-learn} on standardized covariates, with the penalty and the mixing weight (share of the lasso in the penalty) tuned by \mthFolds-fold cross-validation.
The mixing weight took the values \mthEnetMixtures.
For each mixing weight, the penalty grid had \mthEnetNAlphas values, evenly spaced on the log scale from the smallest penalty that sets every coefficient to zero down to \mthEnetEps times it.

\paragraph*{}
Appendix \ref{sec:additional-empirical} describes three more estimators and adds them to the comparison: multivariate adaptive regression splines (MARS), a random forest and a multilayer perceptron (MLP).
These are commonly used ML methods, but less closely related to HAR so we relegate them to the appendix.
We checked any fixed hyperparameter grids for all methods to ensure that extending the boundaries in either direction did not change conclusions in any simulations or empirical datasets.

%% file: sections/40-demonstration/10-simple-function.tex
\subsection{Fitting a Simple Function}

First we qualitatively illustrate what HAR fits look like compared to other methods.
In this subsection the covariate $X$ is one-dimensional ($p = 1$), and we drew \mthDemoN points from the data-generating process (DGP):
\begin{align*}
    X &\sim \mathsf{Unif}([-1,1]) \\
    Y \mid X &\sim \mathsf N\left(\begin{cases}
        -X & X\le0\\
        \sin(2\pi X) & X> 0
    \end{cases}, \ \mthDemoSigma^2\right)
\end{align*}

We repeated the entire process \mthDemoReps times and visualized the truth and predictions (over a fine grid) on a test set, shown in \cref{fig:fits}.
Each estimator was tuned as in \cref{sec:estimators}, on the \mthDemoN points of its repetition.

\begin{figure*}[ht]
    \centering
    \includegraphics[width=\textwidth]{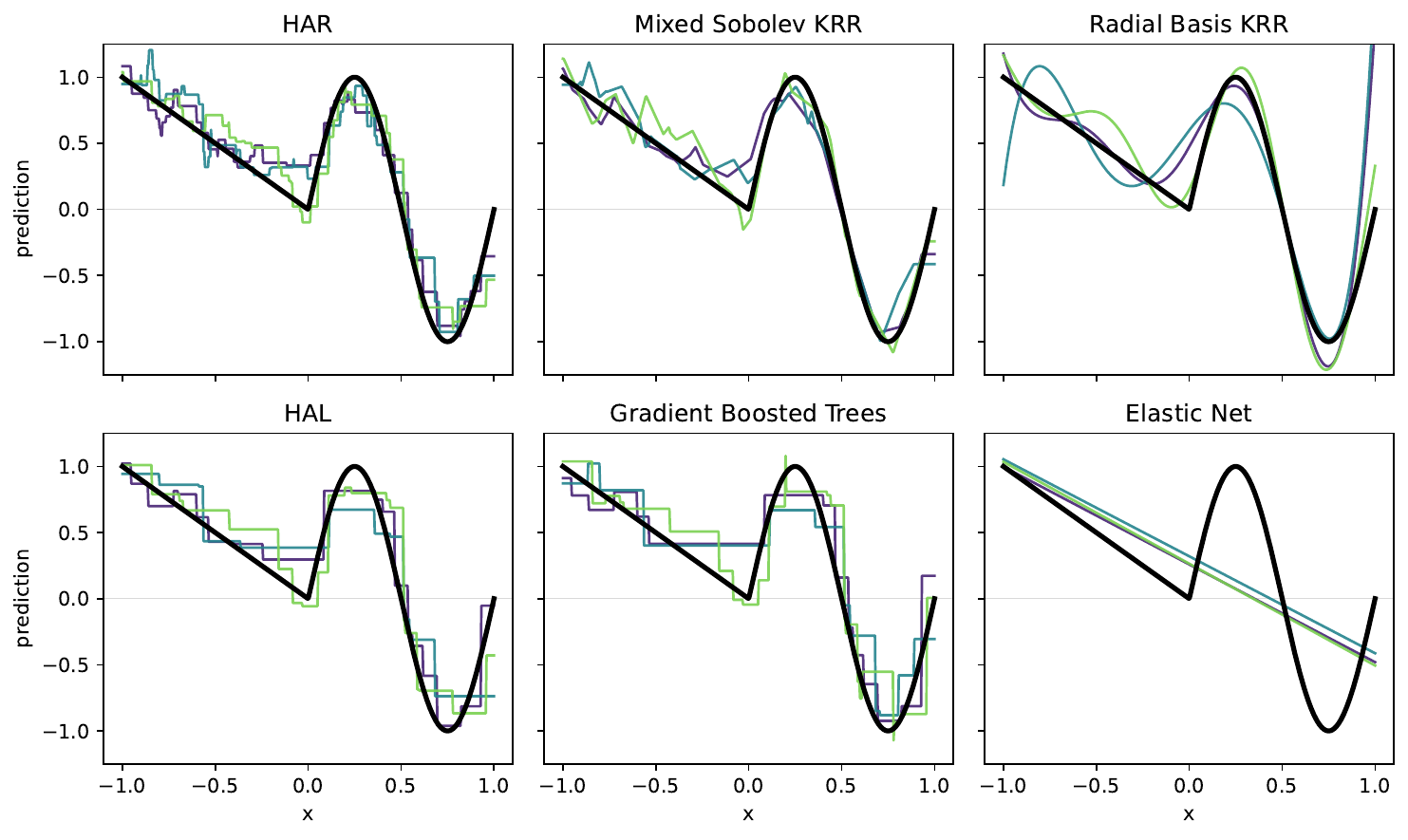}
    \caption{Fits of HAR and other methods on simple one-dimensional data.
    Each panel shows the fits of one estimator in the \mthDemoReps repetitions (colored lines) and the true function (black).}
    \label{fig:fits}
\end{figure*}

Several facts are visually obvious and support our theoretical understanding of HAR's inductive biases.
Firstly, like HAL and gradient boosted trees, HAR fits piecewise constants.
However, HAR's fits are smoother than HAL's, since HAR spreads its fit over more basis functions.
Ignoring the piecewise constancy of HAR, the fits are quite similar to those from MS-KRR.
Our theory predicts this, since the covariate distribution in this case is uniform.
Both are qualitatively more distinct from the radial basis KRR, which ``flails'' at the endpoints and overshoots the minimum.

%% file: sections/40-demonstration/20-bias-simulation.tex
\subsection{Interrogating Inductive Biases in Simulation}
\label{sec:bias-simulation}

In this section we report targeted simulation analyses, which we ran to verify that we understand when HAR should do better or worse than alternatives.
Specifically, we provide empirical evidence for the claims made in \cref{sec:inductive-biases}:

\begin{enumerate}

    \item Estimators that penalize sectional variation (HAR, HAL, MS-KRR, boosted trees) do better when variation is axis-aligned and have difficulty in particular with high-dimensional interactive variation relative to alternatives that are more agnostic to sectional variation (radial basis KRR).
    \item Estimators that penalize derivative norms (HAR, MS-KRR) benefit from smoothness and suffer from abrupt changes relative to those that only penalize sectional variation (HAL, boosted trees).
    \item Relative to MS-KRR, HAR benefits from variation in high-density regions and interactions that turn ``on'' as coordinates increase.
\end{enumerate}

\paragraph*{Setup}
Every simulated dataset had \mthBiasN training points with an outcome $Y$ with $Y \mid X \sim \mathsf N(f(X), \mthBiasSigma^2)$.
In the display of each claim, $\propto$ means that the truth is the expression on its right, centered and scaled to mean $0$ and variance $1$ under the distribution of $X$.
In each data-generating process, the conditional mean function explains \mthBiasRsq of the variance of $Y$.
We chose this low noise level so that the effects that the claims predict show at a smaller sample size, which saves compute time. 
In pilot experiments we saw all the same patterns at higher noise levels, but required more simulation repetitions to confirm the contrasts.

In this subsection, the generalization MSE of a fit is the mean squared difference between the fit and the truth $f$ at \mthBiasNTest new draws of $X$.
Since each conditional mean has variance $1$, a constant fit at the mean of the truth scores about $1$.

The table of each claim gives the mean generalization MSE of each estimator over \mthBiasReps datasets from each process.

\paragraph*{Claim 1: Axis-Aligned and Interactive Variation}
In this subsection, we examine DGPs with waves that are either axis-aligned or rotated along diagonals.
We use the DGP
\begin{equation*}
\begin{gathered}
    X \sim \mathsf{Unif}\big([0,1]^{\mthBiasP}\big), \qquad
    Y \mid X \sim \mathsf N\big(f_k(X), \mthBiasSigma^2\big), \\
    f_k(x) \propto \sin\bigg(2\pi\, \frac{x_1 + \dots + x_k - k/2}{\sqrt k}\bigg),
\end{gathered}
\end{equation*}
where $k \in \{1, \dots, \mthBiasP\}$ is the number of covariates that a wave spans. 
For $k = 1$ the wave is a main effect of $x_1$.
For larger $k$, the sine in the wave $f_k$ is the sine in the one-covariate wave $f_1$ rotated about the center $(1/2, \dots, 1/2)$ of the cube (\cref{fig:bias-truths-waves}).
For larger $k$ the waves represent interactions of $k$ covariates, and the mixed derivative of the sine in all $k$ covariates has amplitude $(2\pi/\sqrt k)^k$, which grows with $k$.

\begin{figure}[ht]
    \centering
    \includegraphics{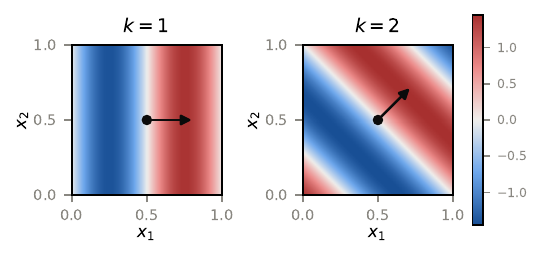}
    \caption{The waves $f_1$ and $f_2$ on the square of the first two covariates.}
    \label{fig:bias-truths-waves}
\end{figure}

\begin{table}[ht]
\centering
\small
\input{artefacts/tables/bias-waves.tex}
\caption{Mean generalization MSE of each estimator on the wave along $k$ covariates, over \mthBiasReps datasets.
The smallest of each row, before rounding, is in bold, and the Monte Carlo standard error of each mean is at most \resBiasMaxRelSe of the mean.
The estimators that penalize sectional variation beat radial basis KRR when the wave lies along one covariate and lose more and more when the wave spans several.}
\label{tab:bias-waves}
\end{table}

The results in \Cref{tab:bias-waves} support our claim.
On the wave along one covariate, each of the four estimators that penalize sectional variation has about two-thirds the mean generalization MSE of radial basis KRR.
Once the wave spans two covariates the order reverses, and the gap widens with $k$: the same ratios lie between \resBiasWaveTwoLow and \resBiasWaveTwoHigh for $k = 2$, and between \resBiasWaveFourLow and \resBiasWaveFourHigh for $k = 4$.
Over the four waves, the mean generalization MSE of radial basis KRR stays between \resBiasWaveRbfLow and \resBiasWaveRbfHigh, while that of the other four estimators rises from at most \resBiasWaveOneBvHigh for $k = 1$ to between \resBiasWaveFourBvLow and \resBiasWaveFourBvHigh for $k = 4$.

The depth tuning of HAR, MS-KRR, and HAL does find the interaction: their chosen depth is at least $k$ in \resBiasDepthAtLeastK of their fits.
The performance hit comes from the price that their norms put on interactions rather than missing them altogether.

While diagonal variation hurts the methods that penalize sectional variation, increasing the ambient dimension (without increasing the interaction dimension) puts them back at parity with radial basis KRR. 
This effect is visible under a DGP similar to the one above:
\begin{equation*}
\begin{gathered}
    X \sim \mathsf{Unif}\big([0,1]^{p}\big), \qquad
    Y \mid X \sim \mathsf N\big(f(X), \mthBiasSigma^2\big), \\
    f(x) \propto \sum_{m=1}^{p/2} \sin\bigg(2\pi\, \frac{x_{2m-1} + x_{2m} - 1}{\sqrt 2}\bigg),
\end{gathered}
\end{equation*}
where $p \in \mthBiasDimPs$ is the number of covariates.
Each term of the sum is the sine in the wave $f_2$, on its own pair of covariates.

\begin{table}[ht]
\centering
\small
\input{artefacts/tables/bias-dimension.tex}
\caption{Mean generalization MSE of each estimator on the sum of waves along pairs of covariates, over \mthBiasDimReps datasets for each number of covariates $p$.
The smallest of each row, before rounding, is in bold, and the Monte Carlo standard error of each mean is at most \resBiasDimMaxRelSe of the mean.
Radial basis KRR has the smallest error with few covariates and loses that lead as $p$ grows.
HAL excluded due to compute time.}
\label{tab:bias-dimension}
\end{table}

This DGP has a varying number of 2-way interactions that are aligned diagonally to the axes.
As the ambient dimension (and number of these interactions) increases, \cref{tab:bias-dimension} shows that the error of radial basis KRR increases more quickly than that of the sectional variation methods.
This suggests that HAR and its relatives are less sensitive to dimension than alternative methods, even in adversarial scenarios (diagonal variation). 

Our ``dimension-free'' asymptotic results are in the fixed-$p$ setting, so they do not specifically predict this result, but the mechanism is related: constraining sectional variation keeps the class of functions searched smaller, even in high dimensions, whereas, e.g., constraining via the radial basis RKHS norm does not. 
The latter therefore unfavorably trades off bias for increased variance in high dimensions, relatively speaking.
The opposite phenomenon is visible in the $k=4$ row of \cref{tab:bias-waves}: when ambient dimension is low but the interaction dimension is high, the bias for the methods that penalize sectional variation outweighs any variance benefit.

\paragraph*{Claim 2: Abrupt and Gradual Change}
In this subsection we construct DGPs with abrupt vs.\ gradual change to see how HAR and MS-KRR fare, primarily versus HAL and boosted trees.
For this we use a clipped ramp function of width $\theta > 0$,
\begin{equation*}
    \varphi_\theta(u) = \operatorname{clip}_{[0,1]}(1/2 + u/\theta),
\end{equation*}
where $\operatorname{clip}_{[0,1]}(z) = \min\{1, \max\{0, z\}\}$, so that $\varphi_\theta(u - \tau)$ rises linearly from $0$ to $1$ over the interval of width $\theta$ centered at a point $\tau$, and we let $\varphi_0(u) = 1(u \ge 0)$, so that $\varphi_0(u - \tau)$ is a step at $\tau$.
Our data-generating process placed several ramps and their interactions through the unit cube:
\begin{equation*}
\begin{gathered}
    X \sim \mathsf{Unif}\big([0,1]^{\mthBiasP}\big), \qquad
    Y \mid X \sim \mathsf N\big(f_\theta(X), \mthBiasSigma^2\big), \\
    \begin{aligned}
    f_\theta(x) \propto{}& \varphi_\theta(x_1 - \tau_1) + \varphi_\theta(x_2 - \tau_2) + \varphi_\theta(x_1 - \tau_3)\, \varphi_\theta(x_2 - \tau_4) \\
    &+ \varphi_\theta(x_3 - \tau_1) + \varphi_\theta(x_4 - \tau_2) + \varphi_\theta(x_3 - \tau_3)\, \varphi_\theta(x_4 - \tau_4),
    \end{aligned}
\end{gathered}
\end{equation*}
where the centers of the ramps are $(\tau_1, \dots, \tau_4) = \mthBiasRuleCentres$.
The main effect and the interaction of each covariate change at different places.
We ran our simulation with $\theta = 0$ (steps) and $\theta = \mthBiasRampWidth$ (ramps), which \cref{fig:bias-truths-rules} draws.

\begin{figure}[ht]
    \centering
    \includegraphics{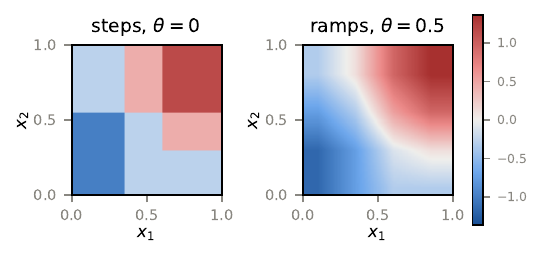}
    \caption{The terms in $x_1$ and $x_2$ of the truth $f_\theta$ of the steps ($\theta = 0$) and of the ramps ($\theta = \mthBiasRampWidth$), which change at the same places.
    The terms in $x_3$ and $x_4$ are the same function of $(x_3, x_4)$.}
    \label{fig:bias-truths-rules}
\end{figure}

\begin{table}[ht]
\centering
\small
\input{artefacts/tables/bias-rules.tex}
\caption{Mean generalization MSE of each estimator on the steps and ramps processes, over \mthBiasReps datasets.
The smallest of each row, before rounding, is in bold, and the Monte Carlo standard error of each mean is at most \resBiasMaxRelSe of the mean.
HAL and gradient boosted trees beat HAR and MS-KRR on the steps, and lose to them on the ramps.}
\label{tab:bias-rules}
\end{table}

\Cref{tab:bias-rules} shows the result.
On the steps, HAL and gradient boosted trees have lower error than HAR and MS-KRR.
On the ramps, HAR and MS-KRR are better by a similar amount.

\paragraph*{Claim 3: Covariate Density and the Anchor}
In this subsection we construct DGPs whose covariates crowd near $0$, and whose truth changes either where the data are dense or where they are sparse, to see how HAR fares against MS-KRR.
We use the ramps of Claim 2, with one center $\tau$ for every ramp:
\begin{equation*}
\begin{gathered}
    X_1, \dots, X_{\mthBiasP} \text{ independent, with density \mthBiasDensityHigh on } [0, \mthBiasDenseTop) \text{ and \mthBiasDensityLow on } [\mthBiasDenseTop, 1], \\
    Y \mid X \sim \mathsf N\big(f_{\tau,\theta}(X), \mthBiasSigma^2\big), \\
    f_{\tau,\theta}(x) \propto \sum_{j=1}^{\mthBiasP} \varphi_\theta(x_j - \tau) + \varphi_\theta(x_1 - \tau)\, \varphi_\theta(x_2 - \tau) + \varphi_\theta(x_3 - \tau)\, \varphi_\theta(x_4 - \tau).
\end{gathered}
\end{equation*}
The density puts \mthBiasDenseMass of the data in $[0, \mthBiasDenseTop)$, the dense part of the range, and the rest in $[\mthBiasDenseTop, 1]$, the sparse part.
In the dense process, the center and the width $(\tau, \theta)$ of every ramp are $(\mthBiasDenseCentre, \mthBiasDenseTop)$, so that every ramp rises across the dense part.
In the sparse process, the center and the width $(\tau, \theta)$ of every ramp are $(\mthBiasSparseCentre, \mthBiasSparseWidth)$, so that every ramp rises across the sparse part.
The flipped process is the dense process with every covariate $X_j$ replaced by $1 - X_j$, on the same datasets.
\Cref{fig:bias-truths-density} illustrates these functions.

\begin{figure}[ht]
    \centering
    \includegraphics[width=\textwidth]{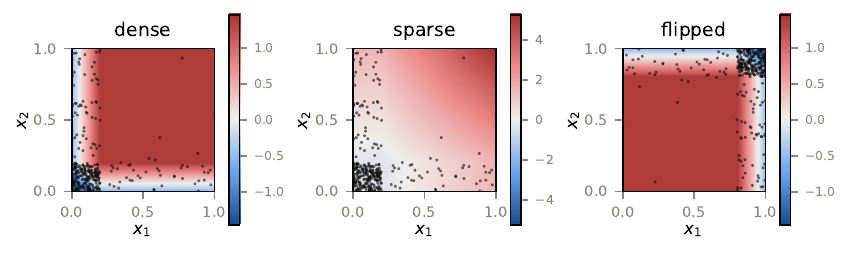}
    \caption{The terms in $x_1$ and $x_2$ of the truths of the dense, sparse and flipped processes, with \mthBiasTruthPoints points of one training draw.
    The terms in $x_3$ and $x_4$ are the same function of $(x_3, x_4)$.
    Colors are scaled per panel.}
    \label{fig:bias-truths-density}
\end{figure}

\begin{table}[ht]
\centering
\small
\input{artefacts/tables/bias-density.tex}
\caption{Mean generalization MSE of each estimator on the processes with skewed covariates, over \mthBiasReps datasets.
The smallest of each row, before rounding, is in bold, and the Monte Carlo standard error of each mean is at most \resBiasMaxRelSe of the mean.
HAR beats MS-KRR when the truth changes where the data are dense and loses to it when the truth changes where they are sparse, and the flip raises the error of HAR and leaves that of MS-KRR as it was.}
\label{tab:bias-density}
\end{table}

\Cref{tab:bias-density} shows the result.
On the dense process, HAR has less than half the generalization MSE of MS-KRR.
On the sparse process the order reverses: MS-KRR has a smaller error than HAR.
The flip leaves the generalization MSE of MS-KRR essentially unchanged in every dataset, but almost doubles the mean generalization MSE of HAR.
HAL suffers the same fate as HAR.
HAR still beats MS-KRR on the flipped process, but by less than on the dense process.

%% file: artefacts/tables/bias-waves.tex
\begin{tabular}{l|rrrrr}
\toprule
Process & HAR & \makecell{Mixed\\Sobolev\\KRR} & HAL & \makecell{Gradient\\Boosted\\Trees} & \makecell{Radial\\Basis\\KRR} \\
\midrule
wave, $k = 1$ & 0.022 & \textbf{0.021} & 0.022 & 0.022 & 0.033 \\
wave, $k = 2$ & 0.11 & 0.090 & 0.11 & 0.10 & \textbf{0.054} \\
wave, $k = 3$ & 0.26 & 0.20 & 0.28 & 0.20 & \textbf{0.053} \\
wave, $k = 4$ & 0.35 & 0.28 & 0.43 & 0.34 & \textbf{0.053} \\
\bottomrule
\end{tabular}

%% file: artefacts/tables/bias-dimension.tex
\begin{tabular}{r|rrrr}
\toprule
$p$ & HAR & \makecell{Mixed\\Sobolev\\KRR} & \makecell{Gradient\\Boosted\\Trees} & \makecell{Radial\\Basis\\KRR} \\
\midrule
4 & 0.11 & 0.089 & 0.15 & \textbf{0.054} \\
8 & 0.29 & 0.23 & 0.31 & \textbf{0.22} \\
12 & 0.50 & \textbf{0.40} & 0.49 & 0.56 \\
16 & 0.66 & \textbf{0.56} & 0.63 & 0.68 \\
24 & 0.72 & \textbf{0.68} & 0.78 & 0.74 \\
\bottomrule
\end{tabular}

%% file: artefacts/tables/bias-rules.tex
\begin{tabular}{l|rrrrr}
\toprule
Process & HAR & \makecell{Mixed\\Sobolev\\KRR} & HAL & \makecell{Gradient\\Boosted\\Trees} & \makecell{Radial\\Basis\\KRR} \\
\midrule
steps & 0.066 & 0.066 & \textbf{0.039} & 0.055 & 0.11 \\
ramps & 0.020 & \textbf{0.019} & 0.032 & 0.046 & 0.023 \\
\bottomrule
\end{tabular}

%% file: artefacts/tables/bias-density.tex
\begin{tabular}{l|rrrrr}
\toprule
Process & HAR & \makecell{Mixed\\Sobolev\\KRR} & HAL & \makecell{Gradient\\Boosted\\Trees} & \makecell{Radial\\Basis\\KRR} \\
\midrule
dense & \textbf{0.016} & 0.038 & 0.033 & 0.049 & 0.066 \\
sparse & 0.054 & 0.034 & 0.056 & 0.053 & \textbf{0.024} \\
dense, flipped & \textbf{0.030} & 0.038 & 0.046 & 0.049 & 0.066 \\
\bottomrule
\end{tabular}

%% file: sections/40-demonstration/30-empirical.tex
\subsection{Empirical Performance}
\label{sec:empirical}

Lastly, we tested HAR against baselines on several real regression datasets from the UCI Machine Learning Repository \citep{Dua2017-yp}.
We took the first \mthTableRows rows of each dataset, split them up randomly into \mthTableTrainPct train and \mthTableTestPct test, and computed the mean squared error (MSE) on the test set.
We used just \mthTableRows rows to speed up the evaluations because the kernel methods we used run slowly for large $n$ (as does HAL), which is a known drawback.
We then repeated the process in its entirety \mthTableReps times and took the average of the test-set MSEs.
\Cref{tab:empirical} gives each method's average as a multiple of the smallest average on the same dataset, so that the best method of each row has the value 1.
A dash marks a method that we did not run on a dataset (\cref{sec:estimators}).
Appendix \ref{sec:additional-empirical} gives the same comparison with the three additional estimators.

The results show the different methods are better for different datasets, with none uniformly dominating. 

The performance patterns are similar for all the methods that penalize sectional variation in theory (HAR, MS-KRR, HAL, and boosting), which supports the claim that the (implicit or explicit) sectional variation control is the main driver of inductive bias for these methods. 
Between HAR and MS-KRR, HAR wins on four of the datasets, but these are also datasets where other methods are the best-performing.

Radial basis KRR is better than the sectional variation methods on the high-dimensional datasets. 
For the slice and yearmsd datasets, methods penalizing sectional variation do worse because the variation in the outcome is likely a complex combination of the covariates with many high-dimensional interactions. 
For yearmsd, the goal is to predict the year a song is from based on low-level features extracted from raw audio files. 
For slice, the goal is to predict the axial position of a CT scan slice based on low-level features extracted from images. 
The targets in these tasks are unlikely to be driven by axis-aligned variation along some section(s) of the covariates. 
That is, it is not the dimensionality itself that drives the difference, but the fact that these particular high-dimensional datasets have complex, high-dimensional interactions.

Some large relative differences are evident on the naval dataset.
The values are very large only because the data from this dataset are from a simulation and therefore all methods have near-zero error. 
On an absolute scale the differences are small.

\paragraph*{Runtime}
We also recorded runtimes (averaged across the \mthTableReps repetitions), which \cref{tab:runtime} gives.
The time of a fit covers its tuning, all of its cross-validation included, the final fit on the training data and the predictions on the test data.
To make the times comparable, each fit ran on a single thread: we limited the linear algebra library, OpenMP and numba to one thread each, so that the whole cross-validation of a fit took place on one core.

The times represent the time needed (with the tuning described above) to reach the accuracy in \cref{tab:empirical}.
They should not be interpreted in any absolute sense since they depend heavily on the tuning strategy, whose grids and searches differ from method to method. 
Nonetheless, they do offer a rough comparison of runtimes under a good-faith effort to tune each method while keeping computation time reasonable, and therefore reflect real-world practice.

The results show that HAR is competitive with the other methods we tested, accounting for tuning, although it is several times slower than MS-KRR, the most similar of the methods we tested. 
This is due to the data-adaptive construction of the kernel. 
HAL is by far the slowest method: even with our low-level implementation and early-stopping depth tuning, HAL is an order of magnitude slower than the kernel methods on the smaller datasets, and does not run at all on the larger ones.

\begin{table}[ht]
\centering
\scriptsize
\setlength{\tabcolsep}{3pt}
\input{artefacts/tables/empirical.tex}
\caption{Test MSE of each method across the UCI datasets, averaged over the \mthTableReps repetitions, as a multiple of the smallest such average on the same dataset, which is in bold.
The column ``rows'' gives the number of rows used, before the split, and the column $p$ gives the number of covariates.
A dash marks a method that we did not run on a dataset.}
    \label{tab:empirical}
\end{table}

\begin{table}[ht]
\centering
\scriptsize
\setlength{\tabcolsep}{3pt}
\input{artefacts/tables/runtime.tex}
\caption{Mean time in seconds over the \mthTableReps repetitions to tune and fit each method on the training data and predict the test data, on one core.
A dash marks a method that we did not run on a dataset.}
    \label{tab:runtime}
\end{table}

\paragraph*{Depth Tuning}
\label{sec:depth-tuning}

HAR and mixed Sobolev KRR both tuned the depth $D$ of their section weights $w_s = 1(|s| \le D)$. 
We investigated this a little more closely than other tuning choices since the weights are in place partly to solve a numerical issue.
To show the effect of this tuning, \cref{fig:depth} plots the cross-validated error of both methods at every depth of the path, from $D = 1$ to the full depth $D = p$, on each dataset of \cref{tab:empirical}.
Each curve comes from the same training data and folds as the fits in \cref{tab:empirical}, with the early stopping removed, and it averages the error over the \mthDepthReps repetitions.
On each dataset, we divided both curves by the smallest value on either original curve.

The depth with the smallest error is below $p$ on \resDepthBelowHar of the \resDepthNData datasets for HAR, and on \resDepthBelowMs for mixed Sobolev KRR.
At the full depth, the error can be orders of magnitude higher than its depth-optimized value for both HAR and MS-KRR.
This does not always occur, but depth tuning prevents that failure mode and early stopping keeps it computationally efficient.

\begin{figure*}[ht]
    \centering
    \includegraphics[width=\textwidth]{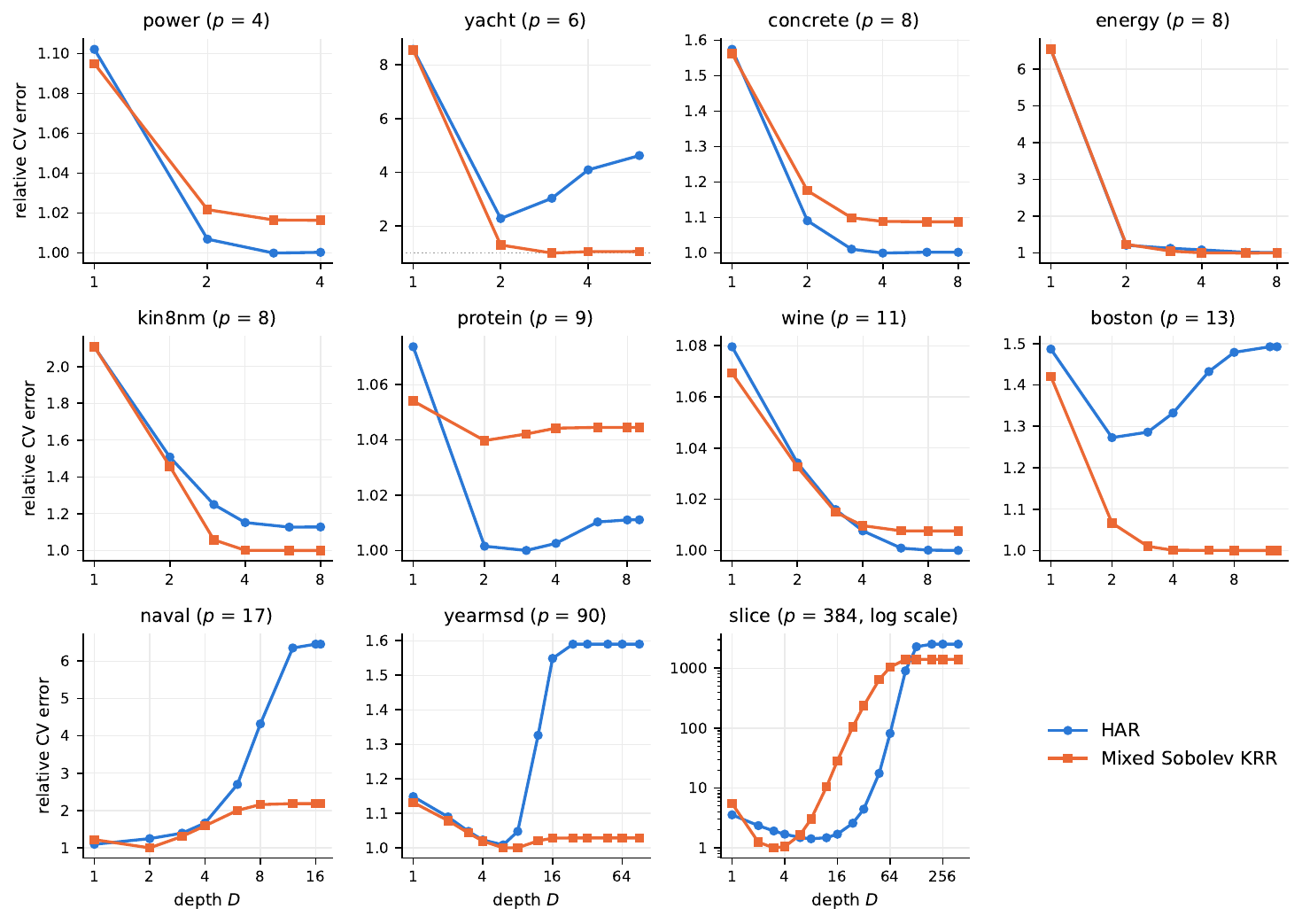}
    \caption{Cross-validated error of HAR and mixed Sobolev KRR at each depth $D$ of the depth path, averaged over the \mthDepthReps repetitions, as a multiple of the smallest such error of either method on the same dataset.
    The last point of each curve is the full depth $D = p$, which is plain HAR or plain mixed Sobolev KRR.
    The depth axis has a log scale, and so does the error axis of a panel whose title says so.}
    \label{fig:depth}
\end{figure*}

%% file: artefacts/tables/empirical.tex
\begin{tabular}{lrr|llllll}
\toprule
data & rows & $p$ & HAR & \makecell{Mixed\\Sobolev\\KRR} & HAL & \makecell{Gradient\\Boosted\\Trees} & \makecell{Radial\\Basis\\KRR} & \makecell{Elastic\\Net} \\
\midrule
power & 2000 & 4 & 1.02 & 1.05 & 1.08 & \textbf{1.00} & 1.13 & 1.34 \\
yacht & 308 & 6 & 2.68 & \textbf{1.00} & 2.57 & 2.01 & 2.10 & 453 \\
concrete & 1030 & 8 & 1.03 & \textbf{1.00} & 1.13 & 1.25 & 1.93 & 7.61 \\
energy & 768 & 8 & 1.22 & 1.24 & 1.01 & \textbf{1.00} & 2.04 & 79.6 \\
kin8nm & 2000 & 8 & 2.46 & 2.16 & 2.71 & 2.56 & \textbf{1.00} & 4.65 \\
protein & 2000 & 9 & 1.03 & 1.07 & \textbf{1.00} & 1.00 & 1.09 & 2.14 \\
wine & 1599 & 11 & 1.00 & \textbf{1.00} & 1.08 & 1.02 & 1.05 & 1.15 \\
boston & 506 & 13 & 1.23 & \textbf{1.00} & 1.35 & 1.17 & 1.12 & 2.58 \\
naval & 2000 & 17 & 168 & 138 & 439 & 1222 & \textbf{1.00} & 6574 \\
yearmsd & 2000 & 90 & 1.03 & 1.05 & --- & 1.06 & \textbf{1.00} & 1.23 \\
slice & 2000 & 384 & 1.77 & 1.25 & --- & 6.46 & \textbf{1.00} & 18.4 \\
\bottomrule
\end{tabular}

%% file: artefacts/tables/runtime.tex
\begin{tabular}{lrr|llllll}
\toprule
data & rows & $p$ & HAR & \makecell{Mixed\\Sobolev\\KRR} & HAL & \makecell{Gradient\\Boosted\\Trees} & \makecell{Radial\\Basis\\KRR} & \makecell{Elastic\\Net} \\
\midrule
power & 2000 & 4 & 48 & 11 & 220 & 24 & 24 & 0.099 \\
yacht & 308 & 6 & 0.29 & 0.24 & 2.5 & 14 & 0.34 & 0.13 \\
concrete & 1030 & 8 & 9.7 & 3 & 64 & 42 & 3.6 & 0.12 \\
energy & 768 & 8 & 4 & 1.4 & 18 & 20 & 1.8 & 0.2 \\
kin8nm & 2000 & 8 & 71 & 14 & 842 & 54 & 24 & 0.1 \\
protein & 2000 & 9 & 52 & 11 & 572 & 15 & 24 & 0.25 \\
wine & 1599 & 11 & 44 & 12 & 567 & 14 & 15 & 0.11 \\
boston & 506 & 13 & 1.1 & 0.85 & 35 & 23 & 0.77 & 0.1 \\
naval & 2000 & 17 & 42 & 11 & 3494 & 18 & 24 & 7.6 \\
yearmsd & 2000 & 90 & 156 & 65 & --- & 420 & 29 & 0.47 \\
slice & 2000 & 384 & 346 & 142 & --- & 1095 & 48 & 33 \\
\bottomrule
\end{tabular}

%% file: sections/50-discussion.tex
\section{Discussion}

HAR provides a conceptually simple and often performant and tractable algorithm with fast convergence in a meaningfully large nonparametric class of functions.
The fast rate means, for example, that many efficient estimators of causal quantities can be shown to be asymptotically linear using HAR under weaker assumptions than would otherwise be required \cite{hal-og, Van_der_Laan2003-du, Tsiatis2007-py}.

For squared error loss HAR is more practical than HAL (which provides the same rate guarantee for a larger function class) because HAR does not require explicit computation of the $n2^p$-column design matrix.
Further speedups borrowing from the KRR literature are likely possible and may also be rate-preserving (e.g. matrix sketching \cite{krr-sketch}, divide-and-conquer \cite{krr-rates-2}).
Recent work \cite{Allerbo2023-kb} suggests that it may also be possible to modify HAR so that multiple solutions along the regularization path can be computed together with warm-start optimization as is done in elastic net algorithms \cite{glmnet}.

The relatively good empirical performance of HAR (and MS-KRR) on many datasets increases the evidence that function classes with constrained sectional variation are often a good model when learning from tabular data.
Whether HAR improves over MS-KRR or other methods in terms of performance depends on the specific data-generating process: HAR is best when variation is smooth enough, mostly axis-aligned, happens in high-density covariate regions, and when coordinates are oriented so interactions turn on as the predictors increase. 

%% file: sections/A1-function-classes.tex
\section{The HAR Class and Mixed Sobolev Classes}
\label{sec:function-classes}

Here we prove the results of \cref{sec:har-functions,sec:mixed-sobolev,sec:returning-variation} on the HAR class and its relation to mixed Sobolev classes and to sectional variation, \cref{prop:har-rkhs,prop:har-kernel-limit,prop:har-equals-mixed-sobolev,prop:variation-balls}.
The last subsection shows that the sectionally-weighted mixed Sobolev kernel of \cref{sec:mixed-sobolev} is the reproducing kernel of an RKHS.

For convenience we restate the definitions of the two classes, with their inner products and kernels.
The HAR class is
\begin{align*}
\mathscr K(\mathbb X) &= \left\{ g:[0,1]^p \to \mathds R \text{ right-continuous} \ : \ \frac{\mathrm{d}g_s}{\mathrm{d}\X_s} \in \mathscr L_2(\X_s) \ \forall s \right\}, \\
\langle g, h \rangle_{\mathscr K(\mathbb X)} &= \sum_{s \subseteq \{1, \dots, p\}} \int_{\mathcal X_s} \frac{\mathrm{d}g_s}{\mathrm{d}\X_s} \, \frac{\mathrm{d}h_s}{\mathrm{d}\X_s} \, \mathrm{d}\X_s, \\
K_{\mathscr K(\mathbb X)}(x,x') &= \sum_{s \subseteq \{1, \dots, p\}} \X_s\big\{(0, x_s \land x'_s]\big\}.
\end{align*}

The Sobolev class of dominating mixed derivatives (with standard kernel) is
\begin{align*}
\mathscr S_2^{(1)} &= \left\{ g:(0,1)^p \to \mathds R \ : \ \partial^{\alpha(s)} g \in \mathscr L_2(\mu) \ \forall s \right\}, \\
\langle g, h \rangle_{\mathscr S} &= \sum_{s \subseteq \{1, \dots, p\}} \int_{(0,1)^p} \partial^{\alpha(s)} g \, \partial^{\alpha(s)} h \, \mathrm{d}\mu, \\
K_{\mathscr S}(x,x') &= \prod_{j=1}^p K_{\mathscr W}(x_j, x'_j)
\end{align*}
where the factor $K_{\mathscr W}(x_j, x'_j) = \cosh(x_j \land x'_j)\cosh(1 - x_j \lor x'_j)/\sinh(1)$ for coordinate $j$ is the reproducing kernel of the Sobolev class $\mathscr W_2^{(1)}$ on $(0,1)$ with the inner product above for $p = 1$.

We also restate the sectional representation theorem for right-continuous functions of bounded sectional variation.
\sectionalrep*

\subsection{Proof of \autoref{prop:variation-balls}}
\label{pf:variation-balls}
First we prove that these function classes restrict sectional variation.

\printProofs[variation-balls]

\subsection{Proof of \autoref{prop:har-rkhs}}
\label{pf:har-rkhs}

\printProofs[har-rkhs]

\subsection{Proof of \autoref{prop:har-kernel-limit}}
\label{pf:har-kernel-limit}

\printProofs[har-kernel-limit]

\subsection{Proof of \autoref{prop:har-equals-mixed-sobolev}}
\label{pf:har-equals-mixed-sobolev}
\Cref{prop:har-equals-mixed-sobolev} establishes that the uniform HAR space is essentially the first-order mixed Sobolev space, but equipped with a different (equivalent) norm.
Its proof rests on two results of Section 3 of \citet{hefter2016equivalence} on the anchored decomposition of functions with mixed smoothness of order one, which writes a function as a sum of one term for each set of coordinates, where each term depends only on the coordinates in its set and vanishes when one of them is zero.
We restate them first, in our notation, and after each one we say how it follows from the source.

In their Sections 2 and 3, \citet{hefter2016equivalence} work on the unit cube of any dimension, with every integrability exponent in $[1, \infty]$, with a weight for each set of coordinates, and with an anchored and an ANOVA case of each result that we use.
We take the dimension $p$, the integrability exponent $2$, all weights equal to one, and the anchored case.
Their sets of coordinates are our sections $s$, their cube of the coordinates in such a set is our face $\mathcal X_s$ up to a Lebesgue-null set, and their distributional derivative for such a set is our weak derivative $\partial^{\alpha(s)} g$.
Since the boundary of the cube is a Lebesgue-null set, the mixed Sobolev class $\mathscr S_2^{(1)}$ with the norm $\|\cdot\|_{\mathscr S}$ is their Sobolev space of dominating mixed smoothness with its standard norm \cite[pp.~4--5]{hefter2016equivalence}.

In this subsection, for a section $s$, a function $\phi \in \mathscr L_2(\mu_s)$ on the face $\mathcal X_s$ and a point $x \in [0,1]^p$, the partial integral
\[
(\mathbf I_s\, \phi)(x) = \int_{(0, x_s]} \phi \, \mathrm{d}\mu_s
\]
integrates each coordinate $j \in s$ of $\phi$ from $0$ to $x_j$, and $\mathbf I_\emptyset\, \phi = \phi$.
The anchored kernel of \citet[p.~4]{hefter2016equivalence} integrates each coordinate $j$ over $[0, x_j)$ in place of $(0, x_j]$, and the two intervals differ by a null set, so their integrals against this kernel are our partial integrals.

In this subsection, let $\Phi$ map each family $(\phi_s)_s$ in the direct sum $\bigoplus_s \mathscr L_2(\mu_s)$, which has one function $\phi_s \in \mathscr L_2(\mu_s)$ for each section $s$, to the restriction of the sum $\sum_s \mathbf I_s\, \phi_s$ of partial integrals to the open cube $(0,1)^p$.
This is the anchored map of \citet[p.~11]{hefter2016equivalence}.

\input{theory/anchored-representation}

\input{theory/anchored-norm}

\printProofs[har-equals-mixed-sobolev]

\subsection{Weighted Mixed Sobolev Kernels}
\label{sec:mixed-sobolev-weights}

The mixed Sobolev kernel $K_{\mathscr S}$ is continuous and bounded, so each member of its RKHS, the mixed Sobolev class $\mathscr S_2^{(1)}$, has a continuous representative, which we use throughout this subsection.

For a section $s$, let $\mathcal H_s$ be the set of functions $g \in \mathscr S_2^{(1)}$ that depend on $x$ only through $x_s$ and have integral zero over $x_j \in (0,1)$ for each $j \in s$ and each value of the other coordinates, that is, $\int_0^1 g(x) \, \mathrm{d}x_j = 0$.
In this subsection, let $K_s(x, x') = \prod_{j \in s} \big(K_{\mathscr W}(x_j, x'_j) - 1\big)$ for a section $s$, with $K_\emptyset = 1$, where $K_{\mathscr W}$ is the factor of the mixed Sobolev kernel $K_{\mathscr S}$ for one coordinate.

\input{theory/anova-decomposition}

In this subsection, for section weights $w_s \ge 0$, written $w$ in a subscript as in the kernel $K_w$ of \cref{sec:weights}, let the weighted space $\mathscr S_w$ be the set of $g \in \mathscr S_2^{(1)}$ with $\Pi_s g = 0$ for each section $s$ with $w_s = 0$, with the inner product
\[
\langle g, h \rangle_{\mathscr S, w} = \sum_{s \,:\, w_s > 0} \frac{\langle \Pi_s g, \Pi_s h \rangle_{\mathscr S}}{w_s}
\qquad \text{for } g, h \in \mathscr S_w.
\]
The weighted mixed Sobolev kernel of \cref{sec:mixed-sobolev} is $K_{\mathscr S, w}(x, x') = \sum_s w_s K_s(x, x')$.

\input{theory/mixed-sobolev-weights}

%% file: theory/anchored-representation.tex
\begin{lemma}[Anchored representation]
\label{lem:anchored-representation}
The map $\Phi$, which sums the partial integrals of a family $(\phi_s)_s$ and restricts the sum to $(0,1)^p$, is a linear bijection from the direct sum $\bigoplus_s \mathscr L_2(\mu_s)$ onto the mixed Sobolev class $\mathscr S_2^{(1)}$, in which functions that are equal $\mu$-almost everywhere are identified.
\end{lemma}

This is the first claim of Proposition 11 of \citet{hefter2016equivalence} in the anchored case, with the integrability exponent $2$.
The proposition says that the direct sum is isomorphic to their Sobolev space of dominating mixed smoothness via this map.
Its proof shows that the map takes the direct sum into that space and that the map sends a family to a function exactly when each member of the family is the corresponding mixed derivative of the anchored component of that function.

%% file: theory/anchored-norm.tex
\begin{lemma}[Anchored norm]
\label{lem:anchored-norm}
There is a constant $C \ge 1$ that depends only on $p$ such that, for every family $(\phi_s)_s$ in the direct sum $\bigoplus_s \mathscr L_2(\mu_s)$,
\[
\frac1C \Big(\sum_s \|\phi_s\|^2_{\mathscr L_2(\mu_s)}\Big)^{1/2}
\le \big\|\Phi\big((\phi_s)_s\big)\big\|_{\mathscr S}
\le C \Big(\sum_s \|\phi_s\|^2_{\mathscr L_2(\mu_s)}\Big)^{1/2}.
\]
\end{lemma}

This is Remark 12 of \citet{hefter2016equivalence} in the anchored case, with all their weights equal to one and the integrability exponent $2$.
The remark shows two facts about this space.
First, with all weights equal to one, their anchored map, which is our map $\Phi$, is an isometry from the direct sum, with the norm $\big(\sum_s \|\phi_s\|^2_{\mathscr L_2(\mu_s)}\big)^{1/2}$, onto the space with their anchored norm.
Second, by the closed graph theorem, their anchored norm is equivalent to their standard norm, which is our norm $\|\cdot\|_{\mathscr S}$.
The two facts together give the two inequalities, with a constant that depends only on the space, and so only on $p$.

%% file: theory/anova-decomposition.tex
\begin{lemma}[Functional ANOVA decomposition]
\label{lem:anova-decomposition}
The zero-integral sets $\mathcal H_s$ are mutually orthogonal closed subspaces of the mixed Sobolev class $\mathscr S_2^{(1)}$ with sum $\mathscr S_2^{(1)}$, and $\mathcal H_s$ with the mixed Sobolev inner product $\langle \cdot, \cdot \rangle_{\mathscr S}$ is an RKHS with reproducing kernel $K_s$.
Hence every $g \in \mathscr S_2^{(1)}$ has exactly one decomposition $g = \sum_s \Pi_s g$ with component $\Pi_s g \in \mathcal H_s$ for each section $s$, and $\Pi_s$ is the orthogonal projection onto $\mathcal H_s$.
\end{lemma}

This is the smoothing spline ANOVA decomposition \cite[Chapter 10]{wahba1990spline} of the mixed Sobolev class $\mathscr S_2^{(1)}$.

%% file: theory/mixed-sobolev-weights.tex
\begin{proposition}
\label{prop:mixed-sobolev-weights}
Let $w_s \ge 0$ be a weight for each section $s$.
Then the weighted space $\mathscr S_w$ with the inner product $\langle \cdot, \cdot \rangle_{\mathscr S, w}$ is an RKHS, and its reproducing kernel is the weighted mixed Sobolev kernel $K_{\mathscr S, w}$ of \cref{sec:mixed-sobolev}.
\end{proposition}

\begin{proof}

By \cref{lem:anova-decomposition}, the weighted space $\mathscr S_w$ is the orthogonal sum of the closed subspaces $\mathcal H_s$ with $w_s > 0$, so it is a closed subspace of the mixed Sobolev class $\mathscr S_2^{(1)}$, and the squared mixed Sobolev norm is $\|g\|_{\mathscr S}^2 = \sum_{s \,:\, w_s > 0} \|\Pi_s g\|_{\mathscr S}^2$ for $g \in \mathscr S_w$.
There are finitely many sections, so $\langle \cdot, \cdot \rangle_{\mathscr S, w}$ is an inner product on $\mathscr S_w$ whose norm is equivalent to the mixed Sobolev norm $\|\cdot\|_{\mathscr S}$.
So the weighted space $\mathscr S_w$ is a Hilbert space, and point evaluation is bounded on it, because it is bounded on the mixed Sobolev class $\mathscr S_2^{(1)}$.

In the rest of this proof, fix $x' \in (0,1)^p$.
By \cref{lem:anova-decomposition}, the kernel function $K_s(\cdot, x')$ is in $\mathcal H_s$, so the weighted kernel function $K_{\mathscr S, w}(\cdot, x') =\sum_{s \,:\, w_s > 0} w_s K_s(\cdot, x')$ is in the weighted space $\mathscr S_w$, where the terms with $w_s = 0$ vanish.
By the uniqueness of the decomposition, its component in $\mathcal H_s$ is $\Pi_s K_{\mathscr S, w}(\cdot, x') = w_s K_s(\cdot, x')$ when $w_s > 0$.
For every $g$ in the weighted space $\mathscr S_w$, the reproducing property of $K_s$ in $\mathcal H_s$ and $\Pi_s g = 0$ for $w_s = 0$ give
\[
\langle g, K_{\mathscr S, w}(\cdot, x') \rangle_{\mathscr S, w}
= \sum_{s \,:\, w_s > 0} \frac{\langle \Pi_s g, w_s K_s(\cdot, x') \rangle_{\mathscr S}}{w_s}
= \sum_{s \,:\, w_s > 0} (\Pi_s g)(x')
= g(x').
\]
So $K_{\mathscr S, w}$ is the reproducing kernel of $\mathscr S_w$.
\end{proof}

%% file: sections/A2-proof-rate.tex
\section{Proofs of the Rate Results}
\label{sec:proof-rate}

Here we prove the rate results of the main text: \cref{thm:main-rate}, for a fixed constant in the bound, and \cref{cor:cv-rate}, for a bound chosen by cross-validation.
We first construct an oracle approximation that converges quickly to the target function and lies in the HAR model with probability tending to one, with coefficients of shrinking $\mathscr L_2$ norm.
Standard empirical process arguments then give the rate for the empirical minimizer in that HAR model, which proves \cref{thm:main-rate}.
The cross-validation oracle inequality of \citet{cv-oracle} then carries this rate over to the estimate of \cref{cor:cv-rate}.
The last subsection proves that squared error loss with a bounded outcome and log loss satisfy the conditions on the loss.
We assume familiarity with empirical process arguments at the basic level of \citet{Van_der_Vaart2000-yx}.

\subsection{Loss Assumptions}
Let \(\mathcal G\) be the class of measurable prediction functions \(g:[0,1]^p\to\mathds R\), and let \(\mathcal Z=[0,1]^p\times\mathcal Y\) be the observation space, where \(\mathcal Y\subseteq\mathds R\) is the outcome space.
The loss operator \(\ell\) assigns each predictor \(g\in\mathcal G\) a measurable loss function \(\ell(g):\mathcal Z\to\mathds R\); thus \(\ell(g)(x,y)\) is its loss at an observation \((x,y)\), and \(\mathbb P \ell(g)=\mathbb E[\ell(g)(X,Y)]\) is its population risk.

As in the main text, let \(f\) minimize this risk over measurable prediction functions with finite risk, and assume that \(f\) belongs to the class \(\mathscr K(\mathbb X)\) of right-continuous functions with square-integrable sectional derivatives.

As in the main text, \(\mathscr M\) is the class of right-continuous functions on \([0,1]^p\) with bounded sectional variation, and we write \(\mathscr M[B]\) for those satisfying \(\|g\|_{\mathscr M}\le B\).
Every member $g$ of $\mathscr M[B]$ is bounded by $B$, because \cref{eq:sectional-representation,eq:sectional-norm} give
\begin{equation}
\label{eq:sup-bound}
|g(x)| \le \sum_s \int_{(0,x_s]} |\mathrm{d}g_s| \le \|g\|_{\mathscr M} \le B
\quad \text{for every } x \in [0,1]^p,
\end{equation}
where the term for $s = \emptyset$ is $|g(0)|$.
The proofs below use this bound to clip functions to $[-B,B]$ and to bound losses in the sup norm, and the loss examples use it to find their constants.

Our rate argument requires some regularity conditions on the loss $\ell$ as an operator over the class $\mathscr M[B]$ because the HAR estimate $\hat f_n$ often does not lie in the HAR class \(\mathscr K(\mathbb X)\).

\begin{assumption}[Lipschitz loss]
    For every $B > 0$, there is a constant $C_B$ such that
    \[
    |\ell(g)(x,y) - \ell(g')(x,y)| \le C_B|g(x) - g'(x)|
    \]
    for all $g, g' \in \mathscr M[B]$ and all $(x,y)$ in the support of $\P$.
\label{ass:lipschitz}
\end{assumption}
\begin{assumption}[Divergence bounds]
    For every $B > 0$ and every $g \in \mathscr M[B]$, 
    $$\|g-f\|^2 \lesssim \P[\ell(g) - \ell(f)] \lesssim \|g-f\|_\infty^2,$$
    with constants that may depend on $B$.
\label{ass:l2}
\end{assumption}
We name the Lipschitz constant $C_B$ because the proofs below refer to it, and we leave the constants of \cref{ass:l2} unnamed because the proofs use only the order of its bounds.

We use three consequences of \cref{ass:lipschitz} throughout.
First, at each point $(x,y)$ in the support of $\P$, the loss $\ell(g)(x,y)$ depends on $g \in \mathscr M[B]$ only through the value $g(x)$, and it is Lipschitz in this value with the constant $C_B$.
The value can be any number in $[-B,B]$, because the constant functions with values in $[-B,B]$ are in $\mathscr M[B]$.
Second, when $B$ is large enough that $f \in \mathscr M[B]$, we can take $g' = f$, and squaring and integrating the Lipschitz bound gives $\|\ell(g) - \ell(f)\| \le C_B\|g - f\|$.
Third, for such $B$, \cref{eq:sup-bound} gives $|g - f| \le 2B$, so the Lipschitz bound gives $\|\ell(g) - \ell(f)\|_\infty \le 2BC_B$, where the sup norm of a function of $(x,y)$ is taken over the support of $\P$.

Instead of \cref{ass:l2} it is often more parsimonious and intuitive to assume the divergence $\P[\ell(g) - \ell(f)]$ is ``equivalent'' in a norm sense with the squared $\mathscr L_2$ norm, i.e. $\|g-f\|^2 \lesssim \P[\ell(g) - \ell(f)] \lesssim \|g-f\|^2$.
This equivalence implies \cref{ass:l2}, because $\|g-f\| \le \|g-f\|_\infty$.
\Cref{ass:lipschitz} together with this equivalence is condition \ref{cond:main-rate-loss} of \cref{thm:main-rate}.
These conditions typically hold under some further weak assumptions on the data distribution $\P$ as can be seen in the following examples, which roughly reproduce ideas found in prior literature (e.g. \cite{laan2006cross}):

\input{theory/squared-error}

Now we consider logistic regression where we take the function of interest $f$ to be the conditional log-odds of the outcome:

\input{theory/logistic-regression}

\subsection{Oracle Approximation}

Recall the HAR basis functions $h_{i,s}(x) = 1(0<X_{i,s} \le x_s)$.
Let $\X_{n,s}$ be the empirical distribution of $X_s$ and, assuming the densities $\frac{\mathrm{d}f_s}{\mathrm{d}\mathbb X_s}$ exist, define the approximation
\begin{align*}
    f_n(x) &= \sum_s \int_{(0,x_s]} \frac{\mathrm{d}f_s}{\mathrm{d}\X_s} \, \mathrm{d}\mathbb X_{n,s} \\
    &= \sum_s \left(
        \frac{1}{n} \sum_i h_{i,s}(x) \frac{\mathrm{d}f_s}{\mathrm{d}\X_s}(X_{i,s})
    \right) \\
    &= H(x)^\top \gamma
\end{align*}
where $\gamma_{i,s} = n^{-1}\frac{\mathrm{d}f_s}{\mathrm{d}\X_s}(X_{i,s})$ are collapsed into a vector $\gamma$.
When $X_{i,j} = 0$ for some $j \in s$, the subvector $X_{i,s}$ lies off the face $\mathcal X_s$ on which the derivative $\mathrm{d}f_s/\mathrm{d}\X_s$ is defined, and the basis function $h_{i,s}$ is zero.
We set the derivative to zero off the face, so that $\gamma_{i,s} = 0$ for these points and the oracle approximation $f_n$ does not change.

Now we consider how well $f_{n}$ approximates $f$ in loss-based divergence.

\input{theory/approx-rate}

Now we show that this fast-converging approximation has a quickly shrinking 2-norm for the coefficients.

\input{theory/approx-coef}

\subsection{Bracketing Entropy}
\label{sec:bracketing}

A bracket $[\underline h, \overline h]$ is the set of functions $h$ with $\underline h \le h \le \overline h$ at every point.
For a probability measure $\mathbb Q$ and a class $\mathscr A$ of functions on the space where $\mathbb Q$ lives, the $\mathscr L_2(\mathbb Q)$ width of the bracket is $\|\overline h - \underline h\|_{\mathscr L_2(\mathbb Q)}$, and $N_{[\,]}(\varepsilon, \mathscr A, \mathbb Q)$ is the number of brackets of $\mathscr L_2(\mathbb Q)$ width $\varepsilon$ needed to cover $\mathscr A$.
The proof of \cref{thm:main-rate} needs the bracketing entropy of the class $\mathscr M[B]$ in $\mathscr L_2(\X)$.
Proposition 2 of \citet{Bibaut2019-zf} bounds it when $\X$ has a bounded density with respect to Lebesgue measure.
\Cref{lem:product-entropy} below generalizes the same bound when $\X$ has a bounded density with respect to a product of probability measures, which may have atoms.

From here to the end of the proof of \cref{lem:product-entropy}, $\mathbb Q = \mathbb Q_1 \times \dots \times \mathbb Q_p$ is a product of probability measures on $[0,1]$.
For each coordinate $j$ and each point $a \in [0,1]$, write $\mathbb Q_j(a) = \mathbb Q_j\{[0,a]\}$ for the distribution function of $\mathbb Q_j$, $\mathbb Q_j(a-) = \mathbb Q_j\{[0,a)\}$ for its left limit, and $\mathbb Q_j\{a\}$ for the mass at $a$.
Let $\mathbb Q_j^{-1}(v) = \inf\{a \in [0,1] : \mathbb Q_j(a) > v\}$ for $v \in [0,1)$ and $\mathbb Q_j^{-1}(1) = 1$ be the right-continuous quantile function of $\mathbb Q_j$, and let the quantile map $\mathbb Q^{-1}$ send $v \in [0,1]^p$ to $(\mathbb Q_1^{-1}(v_1), \dots, \mathbb Q_p^{-1}(v_p))$.
For $a \in (0,1]$ and $v \in [0,1]$,
\begin{equation}
\label{eq:quantile-order}
\mathbb Q_j^{-1}(v) \ge a
\quad \text{exactly when} \quad
\mathbb Q_j(a-) \le v .
\end{equation}
For $v = 1$, both sides hold.
For $v < 1$, the infimum that defines $\mathbb Q_j^{-1}(v)$ is at least $a$ exactly when the distribution function $\mathbb Q_j$ is at most $v$ at every point below $a$, and the left limit $\mathbb Q_j(a-)$ is the supremum of these values.

Our bracketing entropy calculations require the following basic lemmas.
The first states that composition with the quantile map preserves right-continuity and a bound on sectional variation.

\input{theory/quantile-composition}

For $a \in [0,1]$ and $\omega \in (0,1)$, let $\psi_j(a, \omega) = \mathbb Q_j(a-) + \omega\, \mathbb Q_j\{a\}$, the distributional transform of \citet{ruschendorf2009distributional}, and for $x \in [0,1]^p$ and $\omega \in (0,1)^p$ let $\psi(x,\omega)$ have the coordinates $\psi_j(x_j, \omega_j)$.

\input{theory/transform-uniform}

\input{theory/distributional-transform}

\input{theory/product-entropy}

In the rest of this appendix, for a class $\mathscr A$ of functions of $(x,y)$, the bracketing entropy integral is
\[
J_{[\,]}(\delta, \mathscr A) = \int_0^\delta \big(1 + \log N_{[\,]}(\varepsilon, \mathscr A, \P)\big)^{1/2} \, \mathrm{d}\varepsilon,
\]
and the empirical process at a function $h$ of $(x,y)$ is $\mathbb G_n h = \sqrt n\, (\P - \P_n)h$.
The bracketing maximal inequality bounds the expected supremum of the empirical process $\mathbb G_n$ over a class of functions by the entropy integral of the class.
We restate it in our notation.

\input{theory/bracketing-maximal}

In the rest of this appendix, for $B > 0$, let the class of centered losses be
\[
\mathscr H[B] = \{\ell(g) - \ell(f) : g \in \mathscr M[B]\},
\]
the losses of the members of the variation ball $\mathscr M[B]$, centered at the loss $\ell(f)$ of the truth.
\Cref{ass:lipschitz} controls the losses only on the support of $\P$, so here a bracket of functions of $(x,y)$ need hold only on the support, which changes nothing in \cref{lem:bracketing-maximal}, by the remark after it.
The entropy integral that the rate proofs use is $J_{[\,]}(\delta, \mathscr H[B])$, the integral for this class, and not the integral for the variation ball $\mathscr M[B]$ itself.
The next lemma bounds it and applies the bracketing maximal inequality to subsets of the class $\mathscr H[B]$.

\input{theory/loss-entropy}

\subsection{Proofs of \autoref{thm:main-rate} and \autoref{cor:cv-rate}}
\label{pf:main-rate}

The proofs of \cref{thm:main-rate} and \cref{cor:cv-rate} rest on the rate theorem for M-estimators \cite[Theorem 3.4.1]{Van_der_Vaart1996-we}, restated here.

\input{theory/m-estimator-rate}

The proofs of \cref{thm:main-rate} and \cref{cor:cv-rate} both apply this theorem through the next lemma, which gives the rate of \cref{thm:main-rate} uniformly over an interval of radii.

\input{theory/rate-window}

\printProofs[main-rate]

Our rate result assumes that the bound $M$ is chosen larger than the HAR norm of $f$.
In practice, however, we do not know the value of that norm. 
Our next result (\cref{cor:cv-rate}) loosens that requirement: the analyst can instead cross-validate over a suitably chosen grid. 
This also obviates the need for the analyst to scale $M^2$ by $n^{-1}$ and lets them work directly on the scale of $\|\beta\|$, which is what is computationally exposed.

The proof uses two results of \citet{cv-oracle} on cross-validation, which we restate in our notation.

\input{theory/cv-inequality}

\input{theory/cv-maximal}

We recall the setting of \cref{cor:cv-rate} and its symbols, which \cref{sec:cv-rate} defines in full.
The observations are split, independently of their values, into $V \ge 2$ folds of $\floor{n/V}$ or $\ceil{n/V}$ observations each, and the number $V$ of folds does not depend on $n$.
The coefficient bounds form a grid $\tilde M_{n,1} < \dots < \tilde M_{n,J_n}$.
For a grid index $k$ and a fold $v$, $\hat f_{n,k,v}$ is the HAR fit with the bound $\|\beta\|_2 \le \tilde M_{n,k}$ on its coefficient vector $\beta$, estimated on the observations outside fold $v$.
The index $k^*_n$ minimizes the cross-validated risk, and the estimate is the average $\bar f_n = V^{-1}\sum_v \hat f_{n,k^*_n,v}$ of the fits at that index.

\printProofs[cv-rate]

\subsection{Proofs of the Loss Examples}
\label{sec:loss-examples}

\subsubsection{Proof of \autoref{ex:squared-error}}
\label{pf:squared-error}

\printProofs[squared-error]

\subsubsection{Proof of \autoref{ex:logistic-regression}}
\label{pf:logistic-regression}

\printProofs[logistic-regression]

The proof given above for log loss is easily generalized to other likelihood-based losses, i.e. losses of the form $\ell(g) = -\log p_g$ for some invertible map $g \mapsto p_g$ from a function-valued parameter to a density, and under mild regularity conditions.

%% file: theory/squared-error.tex
\begin{exampleE}[Squared Error Loss][text link={The proof is in \autoref{pf:squared-error}.}, category=squared-error]
\label{ex:squared-error}
If $Y$ is bounded, the above assumptions hold for squared error loss $\ell(g) = (g(X) - Y)^2$.
\end{exampleE}
\begin{proofE}
Under squared error loss, the minimizer is $f = \mathbb E[Y|X]$.
The divergence then equals the squared $\mathscr L_2$ distance.
For any $g$, adding and subtracting $f$ gives $(g-Y)^2 - (f-Y)^2 = (g-f)^2 + 2(g-f)(f-Y)$.
The cross term has mean zero by total expectation, because $\P[(g-f)(X)(f(X)-Y)] = \P[(g-f)(X)\, \mathbb E[f(X)-Y \mid X]] = 0$.
So $\P[\ell(g) - \ell(f)] = \|g-f\|^2$, which gives both bounds in \cref{ass:l2}.
To prove the Lipschitz condition note $(\ell(g) - \ell(g')) = (g-Y)^2 - (g'-Y)^2 = [(g-Y) - (g'-Y)][(g-Y) + (g'-Y)] = (g-g')(g+g'-2Y)$.
So for $g, g' \in \mathscr M[B]$ and every $(x,y)$ in the support of $\P$,
\begin{align*}
|\ell(g)(x,y) - \ell(g')(x,y)|
&= |g(x) - g'(x)|\,|g(x) + g'(x) - 2y| \\
&\le (2B + 2\|Y\|_\infty)\,|g(x) - g'(x)|,
\end{align*}
since $g$ and $g'$ are bounded by $B$ (\cref{eq:sup-bound}) and our additional assumption bounds $Y$.
This is \cref{ass:lipschitz} with $C_B = 2B + 2\|Y\|_\infty$.
\end{proofE}

%% file: theory/logistic-regression.tex
\begin{exampleE}[Logistic Regression][text link={The proof is in \autoref{pf:logistic-regression}.}, category=logistic-regression]
\label{ex:logistic-regression}
For binary $Y$, \cref{ass:lipschitz,ass:l2} hold for log loss 
$\ell(g) =
-Y\log
\left(
    \frac{1}{1 + e^{-g(X)}}
\right) 
- (1-Y)\log
\left(
    \frac{1}{1 + e^{g(X)}}
\right)
$.
Here the density $\mathrm{d} \mathbb P_g/\mathrm{d}\nu = p_g(y|x) \equiv \left(\frac{1}{1 + e^{-g(x)}}\right)^y\left(\frac{1}{1 + e^{g(x)}}\right)^{1-y}$, where $\nu$ is the product of counting measure on $\{0,1\}$ and the covariate distribution $\X$, is uniformly bounded away from 0 for all $g$ in $\mathscr M[B]$, because every such $g$ is bounded by $B$ (\cref{eq:sup-bound}), so that $p_g(y|x) \ge (1+e^{B})^{-1}$.
\end{exampleE}
\begin{proofE}
Although a direct proof of \cref{ass:l2} would be slightly quicker, we will instead prove the stronger $\mathscr L_2$ equivalence of the divergence because this argument is more easily generalized.
As noted above, the equivalence implies \cref{ass:l2}.

With the above we can write $\ell(g) = -\log(p_g)$, simply the negative log-likelihood.
The data distribution $\P$ has density $p_f$ with respect to the dominating measure $\nu$, so $\P = \mathbb P_f$, and $\P[\ell(g)-\ell(f)]$ is exactly the Kullback-Leibler divergence
$$\int \log(p_f / p_g) \, \mathrm{d}\mathbb P_f = -\int \log\left(p_g/p_f\right) \, \mathrm{d}\mathbb P_f.$$
The classes $\mathscr M[B]$ grow with $B$, and the constants may depend on $B$, so it is enough to prove the bounds for $B \ge \|f\|_{\mathscr M}$.
Take such a $B$.
Then $f \in \mathscr M[B]$, so the true density $p_f$ is also bounded below.
Write $\underline p = (1+e^{B})^{-1}$ for the lower bound on the densities, and $\sigma$ for the logistic function, so that $p_g(1|x) = \sigma(g(x)) = (1+e^{-g(x)})^{-1}$ and $p_g(0|x) = 1 - \sigma(g(x))$.
The function $\sigma$ is increasing, with $\sigma(-B) = \underline p$ and $\sigma(B) = 1-\underline p$, so on $[-B,B]$ it takes values in $[\underline p, 1-\underline p]$.
Its derivative is $\sigma(1-\sigma) = 1/4 - (\sigma - 1/2)^2$.
This is at most $1/4$, and it decreases as $|\sigma - 1/2|$ grows.
Because $\sigma$ is increasing, $|\sigma - 1/2|$ is largest on $[-B,B]$ at $\pm B$, where $\sigma(1-\sigma) = \underline p(1-\underline p)$.
So on $[-B,B]$ the derivative of $\sigma$ lies between $\underline p(1-\underline p)$ and $1/4$.

Reproducing known arguments \cite[proof of Lemma 5.35, p.~62]{Van_der_Vaart2000-yx}, an application of the pointwise inequality $-\log(z) \ge 2(1-\sqrt{z})$ for $z > 0$ gives the bound below.
That inequality follows from $\log w \le w - 1$ with $w = \sqrt z$.
We apply it with $z = p_g/p_f$ and integrate against $\mathrm{d}\mathbb P_f = p_f \, \mathrm{d}\nu$, and we use $\int p_f \, \mathrm{d}\nu = \int p_g \, \mathrm{d}\nu = 1$:
$$
-\int \log\left(\frac{p_g}{p_f}\right) \, \mathrm{d}\mathbb P_f
\ge 2 \left( 1 - \int \sqrt{p_g p_f} \, \mathrm{d}\nu \right)
= \int \left( \sqrt{p_{g}} - \sqrt{p_{f}} \right)^2 \, \mathrm{d}\nu,
$$
where the equality comes from expanding the square.

Next we bound the right side from below by a multiple of $\|g - f\|^2$, where $\|\cdot\|$ is the $\mathscr L_2(\P)$ norm as elsewhere.
Because the true density $p_f$ is at most $1$, the right side is at least $\int (\sqrt{p_g} - \sqrt{p_f})^2 p_f \, \mathrm{d}\nu =\|\sqrt{p_g} - \sqrt{p_f}\|^2$.
The densities satisfy $\sqrt{p_g} - \sqrt{p_f} = (p_g - p_f)/(\sqrt{p_g} + \sqrt{p_f})$ and $\sqrt{p_g} + \sqrt{p_f} \le 2$, so $\|\sqrt{p_g} - \sqrt{p_f}\| \ge \frac12 \|p_g - p_f\|$.
At each $x$ the density differences $p_g(y|x) - p_f(y|x)$ for $y = 1$ and $y = 0$ are equal up to sign, and $p_f(1|x) + p_f(0|x) = 1$, so $\|p_g - p_f\| = \|\sigma(g) - \sigma(f)\|$ for the logistic function $\sigma$.
Lastly, $|\sigma(g(x)) - \sigma(f(x))| \ge \underline p(1-\underline p)\,|g(x) - f(x)|$ by the mean value theorem and the lower bound on the derivative of $\sigma$, where $\underline p$ is the density lower bound.
Thus we have established the lower bound $\P[\ell(g)-\ell(f)] \ge \underline p^2(1-\underline p)^2\, \|g-f\|^2/4$.

The upper bound comes from a second-order Taylor expansion of $\log$ with the mean value form of the remainder.
At each $(x,y)$, the intermediate density $\tilde p$ takes a value between the densities $p_g(y|x)$ and $p_f(y|x)$, and
$$
\P[\ell(g)-\ell(f)]
= -\P[\log(p_g) - \log(p_f)] 
= -\P\left[(p_g-p_f) p_f^{-1} - \frac{1}{2}(p_g-p_f)^2 \tilde p^{-2} \right].
$$
The first term is zero, because $\P$ has density $p_f$ with respect to the dominating measure $\nu$, so $\P[(p_g - p_f)p_f^{-1}] = \int (p_g - p_f)\, \mathrm{d}\nu = 1 - 1 = 0$.
The second term is at most $\frac12 \underline p^{-2} \|p_g - p_f\|^2$, because the intermediate density $\tilde p$ is at least the density lower bound $\underline p$.
As above, $\|p_g - p_f\| = \|\sigma(g) - \sigma(f)\| \le \frac14 \|g - f\|$, by the upper bound on the derivative of the logistic function $\sigma$.
This gives the upper bound $\P[\ell(g)-\ell(f)] \le \|g-f\|^2/(32\,\underline p^2)$, which completes the proof of the equivalence of the divergence with the squared $\mathscr L_2$ norm.

Finally, $\ell(g)(x,y) = y\log(1+e^{-g(x)}) + (1-y)\log(1+e^{g(x)})$.
Its derivative with respect to the value $g(x)$ is
\[
-y\,\frac{e^{-g(x)}}{1+e^{-g(x)}} + (1-y)\,\frac{e^{g(x)}}{1+e^{g(x)}} = -y\bigl(1-\sigma(g(x))\bigr) + (1-y)\,\sigma(g(x)) = \sigma(g(x)) - y,
\]
which lies in $(-1,1)$ for $y \in \{0,1\}$, because the logistic function $\sigma$ takes values in $(0,1)$.
By the mean value theorem, $|\ell(g)(x,y) - \ell(g')(x,y)| \le |g(x) - g'(x)|$ for all $g, g'$ and all $(x,y)$ with $y \in \{0,1\}$, so \cref{ass:lipschitz} holds with Lipschitz constant $C_B = 1$ for every $B$.
\end{proofE}

%% file: theory/approx-rate.tex
\begin{lemma}
Let $\mathbb X$, $f$, and $f_n$ be as above.
Then the oracle approximation satisfies $\|f_n - f\|_\infty = O_{\P}(n^{-1/2})$.
If \cref{ass:l2} also holds, then the excess risk $\P[\ell(f_{n}) - \ell(f)]$ of the oracle approximation is $O_{\P}(n^{-1})$.
\label{lem:approx-rate}
\end{lemma}
\begin{proof}
First we show $\|f_n-f\|_\infty = O_{\P}(n^{-1/2})$.
We follow the same strategy employed in Lemma 23 of \citet{vanderlaan2023higher}.
Write $f(x) = \sum_s \int_{(0,x_s]} \mathrm{d}f_s$.
The difference is
\begin{align*}
    (f_n-f)(x)
    &= \sum_s \int_{(0,x_s]} \frac{\mathrm{d}f_s}{\mathrm{d}\X_s} \, (\mathrm{d}\mathbb X_{n,s} - \mathrm{d}\mathbb X_s) \\
    &= \sum_s (\mathbb X_{n,s} - \mathbb X_s)\left[
        1_{(0, x_s]}
        \frac{\mathrm{d}f_s}{\mathrm{d}\X_s}
    \right]
\end{align*}
This is an empirical process indexed by $x$, and for each section $s$ the products $u \mapsto 1(u \in (0, x_s]) \frac{\mathrm{d}f_s}{\mathrm{d}\X_s} (u)$ of the indicators with the sectional derivative, indexed by $x$, fall in a Donsker class.
Indeed, the boxes $(0, x_s]$ indexed by $x$ form a Vapnik-Chervonenkis (VC) class of sets, because each box is a cell $\{u : u \le x_s\}$ intersected with a fixed set \cite[Example 2.6.1 and Lemma 2.6.17(ii)]{Van_der_Vaart1996-we}, so their indicators $u \mapsto 1(u \in (0, x_s])$ form a VC-subgraph class \cite[Problem 2.6.9]{Van_der_Vaart1996-we}.
Multiplying them by the fixed function $\mathrm{d}f_s/\mathrm{d}\X_s$ gives a VC-subgraph class \cite[Lemma 2.6.18(vi)]{Van_der_Vaart1996-we} with the square-integrable envelope $|\mathrm{d}f_s/\mathrm{d}\X_s|$, so the class satisfies the uniform entropy condition \cite[Theorem 2.6.7]{Van_der_Vaart1996-we}.
The class is also pointwise measurable, because the boxes with rational corners approximate every box from above \cite[Example 2.3.4]{Van_der_Vaart1996-we}.
So it meets the measurability conditions of \citet[Theorem 2.5.2]{Van_der_Vaart1996-we}, and that theorem shows that it is Donsker.
Therefore the empirical process is uniformly bounded in probability at the rate $n^{-1/2}$, giving the desired supremum norm bound on $f_n-f$, the error of the oracle approximation.

For the second claim, we need \cref{ass:l2}. 
Since this assumption requires functions in a ball $\mathscr M[B]$, we must show $f_n \in \mathscr M[B]$. 
Fix a variation bound $B > \|f\|_{\mathscr M}$.
Each basis function has sectional variation norm at most one, so the sectional variation norm of the oracle approximation $f_n = H^\top\gamma$ is at most the 1-norm of its coefficients, $\|\gamma\|_1 = \sum_s \P_n|\mathrm{d}f_s/\mathrm{d}\X_s|$, which tends to $\sum_s \P|\mathrm{d}f_s/\mathrm{d}\X_s| =\|f\|_{\mathscr M}$ in probability by the law of large numbers.
So the event that $f_n \in \mathscr M[B]$ has probability tending to one.
On this event, \cref{ass:l2} gives $\P[\ell(f_n) - \ell(f)] \lesssim \|f_n - f\|_\infty^2 = O_{\P}(n^{-1})$.
Restricting to an event whose probability tends to one does not change an $O_{\P}$ bound.

\end{proof}

%% file: theory/approx-coef.tex
\begin{lemma}
Let $\X$, $f$, $f_n$, and $\gamma$ be as above and assume that $\mathrm{d}f_s/\mathrm{d}\X_s$ exists and is in $\mathscr L_2$ for each section.
Then the scaled squared norm $n\|\gamma\|_2^2$ of the oracle coefficients converges in probability to $\Sigma = \sum_s \|\mathrm{d}f_s/\mathrm{d}\X_s\|^2_{\mathscr L_2} = \|f\|^2_{\mathscr K(\mathbb X)}$.
In particular, $\|\gamma\|_2^2 = O_{\P}(n^{-1})$, and for every model constant $M > 0$ with $M^2 > \Sigma$, the probability that $\|\gamma\|_2^2 \le M^2/n$ tends to one.
\label{lem:approx-coef}
\end{lemma}

\begin{proof}
The oracle coefficients are $\gamma_{i,s} = n^{-1}\frac{\mathrm{d}f_s}{\mathrm{d}\X_s}(X_{i,s})$, so $\sum_i \gamma_{i,s}^2 = n^{-1}\P_n \left[\left(\frac{\mathrm{d}f_s}{\mathrm{d}\X_s}\right)^2\right]$ and $n\|\gamma\|_2^2 = \sum_s \P_n \left[\left(\frac{\mathrm{d}f_s}{\mathrm{d}\X_s}\right)^2\right]$.
Each squared sectional derivative $(\mathrm{d}f_s/\mathrm{d}\X_s)^2$ is integrable because $\mathrm{d}f_s/\mathrm{d}\X_s \in \mathscr L_2$, so by the weak law of large numbers $\P_n[(\mathrm{d}f_s/\mathrm{d}\X_s)^2]$ converges in probability to $\P[(\mathrm{d}f_s/\mathrm{d}\X_s)^2] = \|\mathrm{d}f_s/\mathrm{d}\X_s\|^2_{\mathscr L_2}$.
Summing over the finitely many sections gives the first claim.
Convergence in probability to the constant $\Sigma$ implies that $n\|\gamma\|_2^2 = O_{\P}(1)$.
For a model constant $M > 0$ with $M^2 > \Sigma$, the event that the scaled squared norm $n\|\gamma\|_2^2$ of the oracle coefficients is within $M^2 - \Sigma$ of $\Sigma$ is contained in the event $n\|\gamma\|_2^2 \le M^2$, and its probability tends to one.
\end{proof}

%% file: theory/quantile-composition.tex
\begin{lemma}
\label{lem:quantile-composition}
Let $B > 0$ and $\mathbb Q$ be a measure as above.
If $g \in \mathscr M[B]$, then $g \circ \mathbb Q^{-1} \in \mathscr M[B]$. 
\end{lemma}

\begin{proof}
By the \hyperref[thm:sectional-representation]{sectional representation theorem}, $g(x) = \sum_s g_s\{(0, x_s]\}$ and the total variations $|g_s|\{\mathcal X_s\}$ of these measures sum to $\|g\|_{\mathscr M} \le B$.
For a point $u$ of a face $\mathcal X_s$, let $Q^-_s(u)$ be the point of $[0,1]^p$ with coordinates $\mathbb Q_j(u_j-)$ for $j \in s$ and $0$ for $j \notin s$.
Also let $m$ be the signed measure on $[0,1]^p$ with
\[
m\{A\} = g(0)\, 1(0 \in A) + \sum_{s \neq \emptyset} g_s\{u \in \mathcal X_s : Q^-_s(u) \in A\}
\]
for every Borel set $A \subseteq [0,1]^p$.

For $v \in [0,1]^p$, let $x = \mathbb Q^{-1}(v)$, the image of $v$ under the quantile map.
Let $s \neq \emptyset$ and $u \in \mathcal X_s$.
For each $j \in s$, \cref{eq:quantile-order} with $a = u_j \in (0,1]$ gives $u_j \le \mathbb Q_j^{-1}(v_j)$ exactly when $\mathbb Q_j(u_j-) \le v_j$.
The coordinates $j \notin s$ of $Q^-_s(u)$ are $0 \le v_j$, so $u$ lies in the box $(0, x_s]$ exactly when $Q^-_s(u) \le v$.
In terms of the sectional measure $g_s$, this says
\begin{align*}
g_s\{(0, x_s]\}
&= \int_{\mathcal X_s} \prod_{j \in s} 1\big(u_j \le \mathbb Q_j^{-1}(v_j)\big) \, \mathrm{d}g_s(u) \\
&= \int_{\mathcal X_s} \prod_{j \in s} 1\big(\mathbb Q_j(u_j-) \le v_j\big) \, \mathrm{d}g_s(u) \\
&= g_s\{u \in \mathcal X_s : Q^-_s(u) \le v\}.
\end{align*}
Since $0 \in [0,v]$, the definition of $m$ and the representation of $g$ above give
\[
m\{[0,v]\}
= g(0) + \sum_{s \neq \emptyset} g_s\{u \in \mathcal X_s : Q^-_s(u) \le v\}
= g(0) + \sum_{s \neq \emptyset} g_s\{(0, x_s]\}
= g(\mathbb Q^{-1}(v)).
\]
So $g(\mathbb Q^{-1}(v)) = m\{[0,v]\}$ for every $v \in [0,1]^p$.
The total variation of the finite signed measure $m$ is at most $\sum_s |g_s|\{\mathcal X_s\} \le B$.
Thus $g \circ \mathbb Q^{-1}$ is a finite signed measure and corresponds to a right-continuous function of sectional variation $\le B$.
\end{proof}

%% file: theory/transform-uniform.tex
\begin{lemma}
\label{lem:transform-uniform}
Let $j \in \{1, \dots, p\}$, and let $Z_j \sim \mathbb Q_j$ and $U_j \sim \mathsf{Unif}(0,1)$ be independent.
Then $\psi_j(Z_j, U_j) \sim \mathsf{Unif}(0,1)$.
\end{lemma}

This is the first claim of Proposition 2.1 of \citet{ruschendorf2009distributional}, applied to the random variable $Z_j$ with values in $[0,1]$ and distribution function $\mathbb Q_j$.
The distributional transform of that proposition is $F(Z_j-) + U_j\,\big(F(Z_j) - F(Z_j-)\big)$ for the distribution function $F$ of $Z_j$, and it equals $\psi_j(Z_j, U_j)$ because $\mathbb Q_j\{a\} = \mathbb Q_j(a) - \mathbb Q_j(a-)$.
The second claim of the proposition recovers $Z_j$ from the transform as a left quantile, almost surely.
\Cref{lem:distributional-transform} proves the matching statement for the right-continuous quantile map $\mathbb Q^{-1}$, with an explicit null set.

%% file: theory/distributional-transform.tex
\begin{lemma}
\label{lem:distributional-transform}
The quantile map $\mathbb Q^{-1}$ undoes the distributional transform $\psi$ off a null set: there is a Borel set $\mathcal N \subseteq [0,1]^p$ with $\mathbb Q\{\mathcal N\} = 0$ such that $\mathbb Q^{-1}(\psi(x,\omega)) = x$ for every $x \in [0,1]^p$ outside $\mathcal N$ and every $\omega \in (0,1)^p$.

If $Z \sim \mathbb Q$ and $U \sim \mathsf{Unif}\big((0,1)^p\big)$ are independent, then $\psi(Z, U) \sim \mu$.
Equivalently, for every bounded Borel function $h$ on $[0,1]^p$,
\[
\int_{[0,1]^p} \int_{(0,1)^p} h\big(\psi(x,\omega)\big) \, \mathrm{d}\omega \, \mathrm{d}\mathbb Q(x) = \int_{[0,1]^p} h \, \mathrm{d}\mu .
\]
\end{lemma}

\begin{proof}
In this proof, let the exceptional set $\mathcal N_j$ be the set of points $a \in [0,1)$ that are not atoms of $\mathbb Q_j$ and have $\mathbb Q_j\{(a,a']\} = 0$ for some $a' > a$, that is, at which the distribution function $\mathbb Q_j$ stays constant on an interval immediately to the right of $a$, and let $\mathcal N$ be the set of points $x$ with $x_j \in \mathcal N_j$ for some $j$.

We claim that $\mathbb Q^{-1}(\psi(x,\omega)) = x$ for every $x \notin \mathcal N$ and every $\omega$.
Fix $j$, and in this paragraph and the next write $a = x_j$, which is not in $\mathcal N_j$, and $v = \psi_j(a, \omega_j)$.
Since $v \ge \mathbb Q_j(a-)$, \cref{eq:quantile-order} gives $\mathbb Q_j^{-1}(v) \ge a$, which is trivial for $a = 0$.
For the reverse inequality $\mathbb Q_j^{-1}(v) \le a$, note that every quantile is at most $1$, which settles $a = 1$, so let $a < 1$.
If $a$ is an atom of $\mathbb Q_j$, then $\mathbb Q_j\{a\} > 0$ and $\omega_j < 1$ give
\[
v = \mathbb Q_j(a-) + \omega_j \mathbb Q_j\{a\} < \mathbb Q_j(a-) + \mathbb Q_j\{a\} = \mathbb Q_j(a) .
\]
So $v < \mathbb Q_j(a) \le 1$, and $a$ belongs to the set $\{t \in [0,1] : \mathbb Q_j(t) > v\}$, whose infimum is $\mathbb Q_j^{-1}(v)$.
Hence $\mathbb Q_j^{-1}(v) \le a$.

If $a$ is not an atom of $\mathbb Q_j$, then $v = \mathbb Q_j(a-) = \mathbb Q_j(a)$.
Since the coordinate $a$ is neither an atom of $\mathbb Q_j$ nor in the exceptional set $\mathcal N_j$, the definition of $\mathcal N_j$ gives the interval $(a,b]$ positive mass $\mathbb Q_j\{(a,b]\} > 0$ for every $b \in (a,1]$, so
\[
\mathbb Q_j(b) = \mathbb Q_j(a) + \mathbb Q_j\{(a,b]\} > \mathbb Q_j(a) = v
\qquad \text{for every } b \in (a,1] .
\]
With $b = 1$, this gives $v < 1$, so the quantile $\mathbb Q_j^{-1}(v)$ of the transformed value $v$ is the infimum of the set $\{t \in [0,1] : \mathbb Q_j(t) > v\}$, and by the display this set contains the interval $(a,1]$.
Hence
\[
\mathbb Q_j^{-1}(v) = \inf\{t \in [0,1] : \mathbb Q_j(t) > v\} \le \inf\,(a,1] = a .
\]
In both cases $\mathbb Q_j^{-1}(\psi_j(a,\omega_j)) = a$, which proves the claim.

The set $\mathcal N$ is {null for the product measure $\mathbb Q$}.
Indeed, each point of $\mathcal N_j$ lies in an interval of positive length on which the distribution function $\mathbb Q_j$ is constant.
The maximal intervals of this kind are disjoint and have positive length, so there are countably many of them.
Such an interval has no mass except possibly at its left endpoint, and only if that endpoint is an atom, which is then not in $\mathcal N_j$.
So the exceptional sets have $\mathbb Q_j\{\mathcal N_j\} = 0$.
The set $\mathcal N$ is the union over $j$ of the sets $\{x : x_j \in \mathcal N_j\}$, whose $\mathbb Q$-measures are $\mathbb Q_j\{\mathcal N_j\}$, so
\[
\mathbb Q\{\mathcal N\} = \mathbb Q\Big\{ \bigcup_{j=1}^p \{x \in [0,1]^p : x_j \in \mathcal N_j\} \Big\} \le \sum_{j=1}^p \mathbb Q_j\{\mathcal N_j\} = 0 .
\]
{The set $\mathcal N$ is also Borel, because each exceptional set $\mathcal N_j$ is the union, over the rational numbers $r \in (0,1]$, of the intervals $\{a \in [0,r) : \mathbb Q_j(a) = \mathbb Q_j(r)\}$, with the countably many atoms of $\mathbb Q_j$ removed.}

For the second claim, let $Z_j$ and $U_j$ be the $j$th coordinates of $Z$ and $U$.
Then $Z_j \sim \mathbb Q_j$ and $U_j \sim \mathsf{Unif}(0,1)$ are independent, so \cref{lem:transform-uniform} gives $\psi_j(Z_j, U_j) \sim \mathsf{Unif}(0,1)$.
The pairs $(Z_j, U_j)$ are independent across $j$, because $\mathbb Q$ is a product and the coordinates of $U$ are independent.
So the coordinates of $\psi(Z, U)$ are independent and uniform, and $\psi(Z, U)$ has the product of $p$ uniform distributions as its law, which is $\mu$.
\end{proof}

%% file: theory/product-entropy.tex
\begin{lemma}
\label{lem:product-entropy}

Let $B > 0$.
{Then the bracketing numbers of the class $\mathscr M[B]$ under the product measure $\mathbb Q$ and under Lebesgue measure $\mu$ satisfy}
\[
N_{[\,]}\big(\varepsilon, \mathscr M[B], \mathbb Q\big) \le N_{[\,]}\big(\varepsilon, \mathscr M[B], \mu\big)
\qquad \text{for every } \varepsilon > 0.
\]
If, in addition, the covariate distribution $\X$ has a density with respect to $\mathbb Q$ that is bounded above, then
\[
\log N_{[\,]}\big(\varepsilon, \mathscr M[B], \X\big) \lesssim \varepsilon^{-1} \big(\log(1/\varepsilon)\big)^{2(p-1)}
\qquad \text{for } 0 < \varepsilon < 1/2,
\]
with a constant that depends only on $p$, $B$ and the bound on the density.
\end{lemma}

\begin{proof}
Firstly, the brackets carry over from Lebesgue measure $\mu$ to the product measure $\mathbb Q$.
Let $[\underline h, \overline h]$ be a bracket of $\mathscr L_2(\mu)$ width at most $\varepsilon$.
By \cref{eq:sup-bound}, every member of $\mathscr M[B]$ is bounded by $B$, so clipping the ends of the bracket to $[-B, B]$ keeps every function $g \circ \mathbb Q^{-1}$ that the bracket contains and does not increase its width.
After this clipping, each end equals a Borel function off a Borel $\mu$-null set.
Replacing the lower end by that Borel function off the null set and by $-B$ on it, and the upper end in the same way with $+B$, keeps every function $g \circ \mathbb Q^{-1}$ in the bracket and leaves its width unchanged, so we may take both ends Borel and bounded by $B$.

Then let $\underline h_{\mathbb Q}(x) = \int_{(0,1)^p} \underline h(\psi(x,\omega)) \, \mathrm{d}\omega$ for $x$ outside {a Borel null set $\mathcal N$ that \cref{lem:distributional-transform} gives} and $\underline h_{\mathbb Q}(x) = -B$ for $x \in \mathcal N$, and define $\overline h_{\mathbb Q}$ in the same way from $\overline h$, with $+B$ on $\mathcal N$.
These are Borel functions, by Fubini's theorem{, because the null set $\mathcal N$ is Borel}.
Next we show that if $g$ is in the class $\mathscr M[B]$ and $g \circ \mathbb Q^{-1}$ lies in the bracket $[\underline h, \overline h]$, then $g$ lies in the bracket $[\underline h_{\mathbb Q}, \overline h_{\mathbb Q}]$.
By the definition of a bracket, $\underline h(v) \le g(\mathbb Q^{-1}(v)) \le \overline h(v)$ for every $v \in [0,1]^p$.
Let $x \in [0,1]^p$ be outside the null set $\mathcal N$.
For every $\omega \in (0,1)^p$, the distributional transform $v = \psi(x,\omega)$ lies in $[0,1]^p$, and the quantile map sends it back to $x$, that is, $\mathbb Q^{-1}(\psi(x,\omega)) = x$, by \cref{lem:distributional-transform}, so
\[
\underline h\big(\psi(x,\omega)\big) \le g\big(\mathbb Q^{-1}(\psi(x,\omega))\big) = g(x) \le \overline h\big(\psi(x,\omega)\big) .
\]
Integrating over $\omega \in (0,1)^p$ gives
\[
\underline h_{\mathbb Q}(x) = \int_{(0,1)^p} \underline h\big(\psi(x,\omega)\big) \, \mathrm{d}\omega \le g(x) \le \int_{(0,1)^p} \overline h\big(\psi(x,\omega)\big) \, \mathrm{d}\omega =\overline h_{\mathbb Q}(x) .
\]
On the null set $\mathcal N$, $\underline h_{\mathbb Q} = -B \le g \le B = \overline h_{\mathbb Q}$ by \cref{eq:sup-bound}.
Jensen's inequality, the null set $\mathcal N$ and {the change of variables of \cref{lem:distributional-transform}} give
\[
\|\overline h_{\mathbb Q} - \underline h_{\mathbb Q}\|^2_{\mathscr L_2(\mathbb Q)}
\le \int_{[0,1]^p} \int_{(0,1)^p} (\overline h - \underline h)^2\big(\psi(x,\omega)\big) \, \mathrm{d}\omega \, \mathrm{d}\mathbb Q(x)
= \|\overline h - \underline h\|^2_{\mathscr L_2(\mu)} \le \varepsilon^2 .
\]
By {\cref{lem:quantile-composition}}, a cover of $\mathscr M[B]$ by brackets of $\mathscr L_2(\mu)$ width $\varepsilon$ contains every function $g \circ \mathbb Q^{-1}$, so the brackets $[\underline h_{\mathbb Q}, \overline h_{\mathbb Q}]$ built from it cover $\mathscr M[B]$ and have $\mathscr L_2(\mathbb Q)$ width at most $\varepsilon$.
This proves the first claim.

For the second claim, let $\zeta \ge 1$ be a bound on the density $\mathrm{d}\X/\mathrm{d}\mathbb Q$.
Then $\|h\|_{\mathscr L_2(\X)} \le \sqrt \zeta\, \|h\|_{\mathscr L_2(\mathbb Q)}$ for every $h$, so a bracket of $\mathscr L_2(\mathbb Q)$ width $\varepsilon/\sqrt \zeta$ has $\mathscr L_2(\X)$ width at most $\varepsilon$, and with the first claim
\[
N_{[\,]}\big(\varepsilon, \mathscr M[B], \X\big)
\le N_{[\,]}\big(\varepsilon/\sqrt \zeta, \mathscr M[B], \mathbb Q\big)
\le N_{[\,]}\big(\varepsilon/\sqrt \zeta, \mathscr M[B], \mu\big).
\]
For $p \ge 2$, Proposition 2 of \citet{Bibaut2019-zf}, applied with Lebesgue measure as the distribution and at the width $\varepsilon/\sqrt \zeta$ for the density bound $\zeta$, bounds the logarithm of the right side for every $\varepsilon$ below some $\varepsilon_0 \in (0, 1/2)$:
\[
\log N_{[\,]}\big(\varepsilon/\sqrt \zeta, \mathscr M[B], \mu\big)
\lesssim \frac{\sqrt \zeta}{\varepsilon} \Big(\log \frac{\sqrt \zeta}{\varepsilon}\Big)^{2(p-1)}
\lesssim \frac{1}{\varepsilon} \Big(\log \frac{1}{\varepsilon}\Big)^{2(p-1)},
\]
where $\varepsilon_0$ and the constants depend only on $p$, the variation bound $B$ and the density bound $\zeta$.
The first bound is the proposition, whose right side has the factor $B$ and the logarithm $\log(\sqrt\zeta B/\varepsilon)$, which is at most a constant times $\log(\sqrt\zeta/\varepsilon)$ when $\varepsilon_0 \le B$.
The second bound holds because the density bound $\zeta$ is at least $1$ and $\log(\sqrt\zeta/\varepsilon) = \log(1/\varepsilon) + \log\sqrt\zeta$, which is at most $(1 + \log\sqrt\zeta/\log 2)\log(1/\varepsilon)$ for $\varepsilon < 1/2$.
For $\varepsilon_0 \le \varepsilon < 1/2$, the logarithm is at most its value at $\varepsilon_0$, because bracketing numbers do not increase in $\varepsilon$, while $\varepsilon^{-1}(\log(1/\varepsilon))^{2(p-1)}$ is at least $2(\log 2)^{2(p-1)} > 0$.
So a larger constant gives the bound for all $0 < \varepsilon < 1/2$.
For $p = 1$, each member of the class is a constant in $[-B,B]$ plus the difference of two nondecreasing functions with values in $[0, B]$, and the bracketing bound for nondecreasing functions \cite[Theorem 2.7.5]{Van_der_Vaart1996-we} bounds the logarithm by a constant times $\varepsilon^{-1}$ under every distribution.
\end{proof}

%% file: theory/bracketing-maximal.tex
\begin{lemma}[Bracketing maximal inequality]
\label{lem:bracketing-maximal}
In this lemma and the remark after it, let $\delta, \delta_\infty > 0$, and let $\mathscr A$ be a class of measurable functions of $(x,y)$.
If every member $h$ of $\mathscr A$ has $\|h\| \le \delta$ and $\|h\|_\infty \le \delta_\infty$, then
\[
\mathbb E \sup_{h \in \mathscr A} |\mathbb G_n h|
\lesssim J_{[\,]}(\delta, \mathscr A)\left(1 + \frac{J_{[\,]}(\delta, \mathscr A)\, \delta_\infty}{\delta^2 \sqrt n}\right),
\]
with a universal constant.
\end{lemma}

This is Lemma 3.4.2 of \citet{Van_der_Vaart1996-we} for the $\mathscr L_2(\P)$ norm.
Their bracketing integral, with the term $1$ under the square root, is the entropy integral $J_{[\,]}$.
Their empirical process $\sqrt n\,(\P_n - \P)$ has the opposite sign to our empirical process $\mathbb G_n$, which the absolute value removes, and, as there, the expectation of a supremum that is not measurable is an outer expectation.
That lemma asks for the strict norm bound $\|h\| < \delta$ on each member $h$ of $\mathscr A$.
Applying it at each norm bound $\delta' > \delta$ and letting $\delta'$ decrease to $\delta$ gives the bound for $\|h\| \le \delta$, because the right side is continuous in the norm bound $\delta$ wherever it is finite.
Its sup norm and its brackets are over the whole space, but values off the support of $\P$ do not matter.
Setting the members of $\mathscr A$ and the ends of their brackets to zero off the support changes no value $\mathbb G_n h$ of the empirical process almost surely, no norm $\|h\|$ and no bracket width.
So the sup norm can be taken over the support, and a bracket need hold only on the support, as in this appendix.

%% file: theory/loss-entropy.tex
\begin{lemma}
\label{lem:loss-entropy}
Let $B > 0$.
Suppose that \cref{ass:lipschitz} holds, that the truth $f$ is in $\mathscr M[B]$, and that the covariate distribution $\X$ has a density that is bounded above with respect to a product of probability measures on $[0,1]$.
Then the class $\mathscr H[B]$ of centered losses has the bracketing entropy integral
\[
J_{[\,]}(\delta, \mathscr H[B]) \lesssim \delta^{1/2}\big(\log(1/\delta)\big)^{p-1}
\qquad \text{for } 0 < \delta < 1/2.
\]
For all $\delta, \delta_\infty > 0$, every subset $\mathscr A$ of the class $\mathscr H[B]$ whose members have norm at most $\delta$ and sup norm at most $\delta_\infty$ satisfies
\[
\mathbb E \sup_{h \in \mathscr A} |\mathbb G_n h|
\lesssim J_{[\,]}(\delta, \mathscr H[B])\left(1 + \frac{J_{[\,]}(\delta, \mathscr H[B])\, \delta_\infty}{\delta^2 \sqrt n}\right).
\]
The constants depend only on $p$, $B$, the Lipschitz constant $C_B$ and the bound on the density.
\end{lemma}

\begin{proof}
The brackets of the centered losses come from those of the class $\mathscr M[B]$.
By \cref{ass:lipschitz}, the loss $\ell(g)(x,y)$ at a point depends on $g$ only through the value $g(x)$, and it is Lipschitz in this value.
In this proof, for a point $(x,y)$ in the support of $\P$ and a number $t \in [-B,B]$, let $\ell_{x,y}(t)$ be the loss at $(x,y)$ of the constant function with value $t$.
For a member $g$ of $\mathscr M[B]$, the constant function with value $g(x)$ is in $\mathscr M[B]$ by \cref{eq:sup-bound}, so \cref{ass:lipschitz} with this constant function as $g'$ gives
\[
\ell(g)(x,y) = \ell_{x,y}\big(g(x)\big)
\qquad \text{for all } g \in \mathscr M[B].
\]
The same assumption, for two constant functions, gives $|\ell_{x,y}(t) - \ell_{x,y}(t')| \le C_B|t - t'|$ for $t, t' \in [-B,B]$, with the Lipschitz constant $C_B$ of \cref{ass:lipschitz}.
In this proof, we may take $C_B \ge 1$, because a larger constant also satisfies \cref{ass:lipschitz}.

Let $\varepsilon > 0$, and in this proof let $[\underline g, \overline g]$ be a bracket of $\mathscr L_2(\X)$ width $\varepsilon/C_B$ that contains some member of $\mathscr M[B]$, with its ends clipped to $[-B,B]$.
Clipping does not widen the bracket, and the bracket still contains each member $g$ of $\mathscr M[B]$ that it contained, because $|g| \le B$ by \cref{eq:sup-bound}.
In this proof, let the ends of the loss bracket be
\[
L^-(x,y) = \min_{\underline g(x) \le t \le \overline g(x)} \ell_{x,y}(t),
\qquad
L^+(x,y) = \max_{\underline g(x) \le t \le \overline g(x)} \ell_{x,y}(t)
\]
at each point $(x,y)$ in the support of $\P$: the smallest and the largest loss at $(x,y)$ over the values between the two ends, which exist because $\ell_{x,y}$ is continuous.
Each member $g$ of $\mathscr M[B]$ in the bracket has $\underline g(x) \le g(x) \le \overline g(x)$, so the representation $\ell(g)(x,y) = \ell_{x,y}(g(x))$ gives
\[
L^-(x,y) \le \ell(g)(x,y) \le L^+(x,y),
\]
so $[L^-, L^+]$ is a bracket that contains the loss $\ell(g)$, because in this appendix brackets of functions of $(x,y)$ hold on the support of $\P$.
The two ends are values of $\ell_{x,y}$ at two points of $[\underline g(x), \overline g(x)]$, so the Lipschitz bound on $\ell_{x,y}$ gives $0 \le L^+(x,y) - L^-(x,y) \le C_B\big(\overline g(x) - \underline g(x)\big)$.
For functions of $x$ the norm $\|\cdot\|$ is the norm of $\mathscr L_2(\X)$, so squaring and integrating give
\[
\|L^+ - L^-\| \le C_B\, \|\overline g - \underline g\|_{\mathscr L_2(\X)} \le \varepsilon .
\]
So each bracket of width $\varepsilon/C_B$ for the class $\mathscr M[B]$ yields a bracket of width $\varepsilon$ for the losses of its members.
Subtracting the fixed function $\ell(f)$ from both ends gives a bracket of the same width for the centered losses, so
\[
N_{[\,]}(\varepsilon, \mathscr H[B], \P) \le N_{[\,]}(\varepsilon/C_B, \mathscr M[B], \X).
\]

Here the density condition enters.
With it, \cref{lem:product-entropy} bounds the bracketing entropy of the class $\mathscr M[B]$: $\log N_{[\,]}(\varepsilon, \mathscr M[B], \X) \lesssim \varepsilon^{-1}(\log(1/\varepsilon))^{2(p-1)}$ for $0 < \varepsilon < 1/2$.
Because the Lipschitz constant $C_B$ is at least one, $\varepsilon/C_B < 1/2$ for $0 < \varepsilon < 1/2$, and $\log(C_B/\varepsilon)$ is at most a constant times $\log(1/\varepsilon)$, because $\log(1/\varepsilon) > \log 2$.
So the bracketing entropy $\log N_{[\,]}(\varepsilon, \mathscr H[B], \P)$ of the centered losses is at most a constant times $\varepsilon^{-1}(\log(1/\varepsilon))^{2(p-1)}$ for $0 < \varepsilon < 1/2$.
Integrating the square root of this bound gives the first claim.
For the main term, the substitution $\varepsilon = \delta u$ gives
\begin{align*}
\int_0^\delta \varepsilon^{-1/2}\big(\log(1/\varepsilon)\big)^{p-1} \, \mathrm{d}\varepsilon
&= \delta^{1/2} \int_0^1 u^{-1/2}\big(\log(1/\delta) + \log(1/u)\big)^{p-1} \, \mathrm{d}u \\
&\lesssim \delta^{1/2}\big(\log(1/\delta)\big)^{p-1},
\end{align*}
because $\log(1/\delta) > \log 2$ and $\int_0^1 u^{-1/2}(\log(1/u))^{p-1}\, \mathrm{d}u$ is finite.

A cover of the class $\mathscr H[B]$ by brackets also covers each subset $\mathscr A$ of it, so the entropy integral $J_{[\,]}(\delta, \mathscr A)$ of the subset is at most the entropy integral $J_{[\,]}(\delta, \mathscr H[B])$ of the class.
The bracketing maximal inequality, \cref{lem:bracketing-maximal}, gives the second claim with the entropy integral of the subset in place of that of the class, and its right side increases with the entropy integral.
\end{proof}

%% file: theory/m-estimator-rate.tex
\begin{theorem}[Rate of convergence of M-estimators]
\label{thm:m-estimator-rate}
Let $B > 0$ be such that the truth $f$ is in $\mathscr M[B]$, let $\hat g_n$ be an estimator with values in $\mathscr M[B]$, and let $r_n > 0$ for each $n$.
Suppose that the following three conditions hold, with constants that depend on neither $n$ nor $\delta$.
\begin{enumerate}
\item \label{cond:m-rate-growth} \emph{Quadratic growth.}
The divergence grows at least quadratically in the distance: $\|g - f\|^2 \lesssim \P[\ell(g) - \ell(f)]$ for every $g$ in the variation ball $\mathscr M[B]$.
\item \label{cond:m-rate-modulus} \emph{Modulus of continuity.}
There are functions $\omega_n$ from $(0, \infty)$ to $[0, \infty)$ such that, for every $n$ and every radius $\delta > 0$,
\[
\mathbb E \sup_{g \in \mathscr M[B] :\, \|g - f\| \le \delta} \big|\mathbb G_n(\ell(g) - \ell(f))\big| \lesssim \omega_n(\delta),
\]
and, for a power $\alpha < 2$ that does not depend on $n$, the ratio $\omega_n(\delta)/\delta^\alpha$ decreases in $\delta$.
\item \label{cond:m-rate-min} \emph{Near minimization.}
The estimator nearly minimizes the empirical risk: $\P_n \ell(\hat g_n) \le \P_n \ell(f) + O_{\P}(r_n^{-2})$.
\end{enumerate}
If $r_n^2\,\omega_n(1/r_n) \lesssim \sqrt n$ for all large $n$, then $\|\hat g_n - f\| = O_{\P}(r_n^{-1})$.
\end{theorem}

%% file: theory/rate-window.tex
\begin{lemma}[Rate over an interval of radii]
\label{lem:rate-window}
Let $\Xi \subset (0, \infty)$ be a closed and bounded interval.
For each sample size $n$ and each radius $M \in \Xi$, let $\hat f_{n,M} = \argmin_{g \in \mathscr K(\mathbb X_n)[M]} \allowbreak \P_n \ell(g)$ be the HAR estimate with radius $M$.
Suppose that conditions \ref{cond:main-rate-loss} and \ref{cond:main-rate-density} of \cref{thm:main-rate} hold.
If the truth $f$ is in the HAR class $\mathscr K(\mathbb X)$ and its norm $\|f\|_{\mathscr K(\mathbb X)}$ is less than $\min \Xi$, then
\[
\lim_{T \to \infty}\, \limsup_{n \to \infty}\, \sup_{M \in \Xi}\,
\P\bigl\{n^{1/3}(\log n)^{-2(p-1)/3}\,\|\hat f_{n,M} - f\| > T\bigr\} = 0.
\]
\end{lemma}

\begin{proof}
We first show that the HAR estimate attains the rate for every sequence of radii in $\Xi$, and then that the rate is uniform over $\Xi$.
In this proof, let $M_1, M_2, \dots$ be radii in $\Xi$, one for each sample size $n$; the first part allows any such radii, and the second part chooses them.

We apply the rate theorem for M-estimators, \cref{thm:m-estimator-rate}, with the HAR estimate $\hat f_{n,M_n}$ as the estimator.
In this proof, its variation bound is $B = 2^{p/2}\max \Xi$.
That theorem requires the estimate and the truth to lie in the variation ball $\mathscr M[B]$, and \cref{prop:variation-balls} places them there.
By that proposition, the sectional variation norm $\|g\|_{\mathscr M}$ is at most $2^{p/2}$ times the empirical HAR norm $\|g\|_{\mathscr K(\mathbb X_n)}$ for every function $g$ in the empirical HAR class $\mathscr K(\mathbb X_n)$.
The estimate lies in the ball $\mathscr K(\mathbb X_n)[M_n]$, so its empirical HAR norm is at most the radius $M_n$, and its sectional variation norm is
\[
\|\hat f_{n,M_n}\|_{\mathscr M} \le 2^{p/2} M_n \le 2^{p/2}\max \Xi = B .
\]
The same proposition for the HAR class $\mathscr K(\mathbb X)$, in which the truth has norm less than the left end $\min \Xi$ of the interval, gives
\[
\|f\|_{\mathscr M} \le 2^{p/2}\|f\|_{\mathscr K(\mathbb X)} < 2^{p/2}\min \Xi \le B .
\]
So both the estimate and the truth are in the variation ball $\mathscr M[B]$.
The rest of the first part checks the three conditions of \cref{thm:m-estimator-rate}, then chooses the inverse rate $r_n$ and checks the rate condition, the inequality for $r_n$ in the last sentence of that theorem.

Firstly, condition \ref{cond:m-rate-growth}, quadratic growth, is the lower bound in \cref{ass:l2}.

Secondly, condition \ref{cond:m-rate-modulus}, the modulus of continuity, comes from \cref{lem:loss-entropy}.
Its conditions hold: \cref{ass:lipschitz} holds, the truth is in the class $\mathscr M[B]$, and condition \ref{cond:main-rate-density} is its density condition.
The entropy integral in that lemma is $J_{[\,]}(\delta, \mathscr H[B])$, for the class $\mathscr H[B]$ of centered losses with the variation bound $B$ of this proof.
In this proof, let $C_B$ be the Lipschitz constant of \cref{ass:lipschitz} for the variation bound $B$.
For every radius $\delta > 0$ and every $g$ in the variation ball $\mathscr M[B]$ with $\|g - f\| \le \delta$, the Lipschitz bound gives $\|\ell(g) - \ell(f)\| \le C_B\delta$, and it gives $\|\ell(g) - \ell(f)\|_\infty \le C_B(\|g\|_\infty + \|f\|_\infty) \le 2BC_B$, where the last step is \cref{eq:sup-bound}.
So the lemma bounds the modulus, the expected supremum of the absolute value of the empirical process $\mathbb G_n$ at the centered loss $\ell(g) - \ell(f)$ over these $g$, by a multiple of the modulus bound, which in this proof is the function $\omega_n$ of the modulus condition of \cref{thm:m-estimator-rate},
\[
\omega_n(\delta) = J_{[\,]}(C_B\delta, \mathscr H[B])\left(1 + \frac{2BC_B\, J_{[\,]}(C_B\delta, \mathscr H[B])}{C_B^2\delta^2\sqrt n}\right).
\]

\Cref{thm:m-estimator-rate} also requires that $\omega_n(\delta)/\delta^\alpha$ decrease in $\delta$ for some power $\alpha < 2$.
In this proof, we take the power $\alpha = 1$.
Recall the definition of the entropy integral,
\[
J_{[\,]}(\delta, \mathscr H[B]) = \int_0^\delta \big(1 + \log N_{[\,]}(\varepsilon, \mathscr H[B], \P)\big)^{1/2} \, \mathrm{d}\varepsilon .
\]
Its integrand is positive and does not increase in the width $\varepsilon$, because the bracketing number $N_{[\,]}(\varepsilon, \mathscr H[B], \P)$ does not increase as the width grows.
The average of a positive function that does not increase, over an interval $[0, t]$ with $t > 0$, does not increase in $t$.
So the entropy integral $J_{[\,]}(C_B\delta, \mathscr H[B])$ divided by $\delta$, which is $C_B$ times the average of the integrand over $[0, C_B\delta]$, does not increase in $\delta$, and $J_{[\,]}(C_B\delta, \mathscr H[B])/\delta^2$, which is this ratio times $1/\delta$, decreases.
The ratio of the modulus bound to the radius is
\[
\frac{\omega_n(\delta)}{\delta}
= \frac{J_{[\,]}(C_B\delta, \mathscr H[B])}{\delta}\left(1 + \frac{2B}{C_B\sqrt n} \cdot \frac{J_{[\,]}(C_B\delta, \mathscr H[B])}{\delta^2}\right),
\]
the product of a positive factor that does not increase and a positive factor that decreases, so it decreases.
This verifies condition \ref{cond:m-rate-modulus} of \cref{thm:m-estimator-rate}.

Thirdly, for condition \ref{cond:m-rate-min}, the estimator nearly minimizes the empirical risk, because the oracle approximation $f_n$ is in the model.
The square of the left end $\min \Xi$ of the interval exceeds the squared HAR norm $\|f\|_{\mathscr K(\mathbb X)}^2$ of the truth, which is at least the limit of the scaled squared norm of the oracle coefficients in \cref{lem:approx-coef}.
So that lemma, with $\min \Xi$ as its model constant, shows that the oracle coefficients have squared norm at most $(\min \Xi)^2/n$ with probability tending to one.
On this event the oracle approximation $f_n$ is in the model $\mathscr K(\mathbb X_n)[M_n]$, because the radius $M_n$ is at least $\min \Xi$, and then the empirical risk $\P_n \ell(\hat f_{n,M_n})$ of the estimate is at most the empirical risk $\P_n \ell(f_n)$ of the oracle approximation.
This gives the near-minimization condition of \cref{thm:m-estimator-rate},
\[
\P_n \ell(\hat f_{n,M_n}) \le \P_n \ell(f) + \P_n[\ell(f_n) - \ell(f)],
\]
with the error $\P_n[\ell(f_n) - \ell(f)]$.

The error is small.
Split it as $\P[\ell(f_n) - \ell(f)] + (\P_n - \P)[\ell(f_n) - \ell(f)]$.
The first part is $O_{\P}(n^{-1})$, by \cref{lem:approx-rate}.
For the second, note that on the event of the previous paragraph the oracle approximation $f_n$ is in the model, and so in the variation ball $\mathscr M[B]$, as the estimate is.
There, \cref{lem:approx-rate} and the Lipschitz bound, with the Lipschitz constant $C_B$, give $\|\ell(f_n) - \ell(f)\|_\infty \le C_B\|f_n - f\|_\infty = O_{\P}(n^{-1/2})$.
In this proof, for a constant $\eta > 0$, let $\mathscr A_n$ be the set of centered losses $h$ in the class $\mathscr H[B]$ with $\|h\| \le \eta\,n^{-1/2}$ and $\|h\|_\infty \le \eta\,n^{-1/2}$.
The norm of a function is at most its sup norm, so with probability close to one, for a large constant $\eta$, the centered loss $\ell(f_n) - \ell(f)$ of the oracle approximation lies in the set $\mathscr A_n$.
We apply \cref{lem:loss-entropy} to this set, with the norm bound $\delta$ and the sup-norm bound $\delta_\infty$ of that lemma both equal to $\eta\,n^{-1/2}$.
For large $n$, its first claim gives
\[
J_{[\,]}(\eta\,n^{-1/2}, \mathscr H[B])
\lesssim \eta^{1/2} n^{-1/4}\big(\log(n^{1/2}/\eta)\big)^{p-1}
\lesssim n^{-1/4}(\log n)^{p-1}.
\]
With these bounds, the ratio in the parentheses of its second claim is
\[
\frac{J_{[\,]}(\delta, \mathscr H[B])\,\delta_\infty}{\delta^2\sqrt n}
= \frac{J_{[\,]}(\eta\,n^{-1/2}, \mathscr H[B])}{\eta}
\lesssim n^{-1/4}(\log n)^{p-1},
\]
which tends to zero, so its second claim gives
\[
\mathbb E \sup_{h \in \mathscr A_n} |\mathbb G_n h| \lesssim n^{-1/4}(\log n)^{p-1}.
\]
On the event that the centered loss $\ell(f_n) - \ell(f)$ lies in the set $\mathscr A_n$, the second part of the error is at most $n^{-1/2}\sup_{h \in \mathscr A_n}|\mathbb G_n h|$, because the empirical process $\mathbb G_n$ is $\sqrt n$ times $\P - \P_n$.
So Markov's inequality gives $(\P_n - \P)[\ell(f_n) - \ell(f)] = O_{\P}(n^{-3/4}(\log n)^{p-1})$, and the two parts together give
\[
\P_n[\ell(f_n) - \ell(f)]
= O_{\P}(n^{-1}) + O_{\P}\big(n^{-3/4}(\log n)^{p-1}\big)
= O_{\P}\big(n^{-3/4}(\log n)^{p-1}\big).
\]

It remains to choose the inverse rate $r_n$.
\Cref{thm:m-estimator-rate} bounds the error of the HAR estimate $\hat f_{n,M_n}$ by $\|\hat f_{n,M_n} - f\| = O_{\P}(r_n^{-1})$ if two orders hold: the rate condition $r_n^2\,\omega_n(1/r_n) \lesssim \sqrt n$ for all large $n$, with the modulus bound $\omega_n$, and the order $r_n^{-2}$ of the error $\P_n[\ell(f_n) - \ell(f)]$ in the near-minimization condition.
In this proof, let $\delta_n = 1/r_n$ be the radius at which the rate condition evaluates the modulus bound $\omega_n$, so that the rate condition reads $\omega_n(\delta_n) \lesssim \delta_n^2\sqrt n$.
The modulus bound $\omega_n(\delta_n)$ is at least the entropy integral $J_{[\,]}(C_B\delta_n, \mathscr H[B])$, so the rate condition needs this entropy integral to be of order at most $\delta_n^2\sqrt n$.
We choose the smallest radius $\delta_n$ that allows this, where the two are of the same order.
With the entropy integral $J_{[\,]}(C_B\delta_n, \mathscr H[B])$ replaced by $\delta_n^{1/2}(\log(1/\delta_n))^{p-1}$, which bounds it up to a constant by the first claim of \cref{lem:loss-entropy}, the two balance where
\[
\delta_n^2\sqrt n \asymp \delta_n^{1/2}\big(\log(1/\delta_n)\big)^{p-1},
\qquad\text{that is,}\qquad
\delta_n^{3/2} \asymp n^{-1/2}\big(\log(1/\delta_n)\big)^{p-1}.
\]
A radius $\delta_n$ that is a negative power of $n$ times a power of $\log n$ has $\log(1/\delta_n) \asymp \log n$, so the balance is $\delta_n^{3/2} \asymp n^{-1/2}(\log n)^{p-1}$.
So, in this proof, take
\[
\delta_n = n^{-1/3}(\log n)^{2(p-1)/3},
\qquad
r_n = 1/\delta_n = n^{1/3}(\log n)^{-2(p-1)/3}.
\]

Next we verify the rate condition at this radius.
For large $n$, the first claim of \cref{lem:loss-entropy} gives
\[
J_{[\,]}(C_B\delta_n, \mathscr H[B])
\lesssim \delta_n^{1/2}\big(\log(1/\delta_n)\big)^{p-1}
\lesssim n^{-1/6}(\log n)^{4(p-1)/3},
\]
where the last step uses $\log(1/\delta_n) \le \log n$.
This is the order of $\delta_n^2\sqrt n = n^{-1/6}(\log n)^{4(p-1)/3}$.
So the ratio of the entropy integral $J_{[\,]}(C_B\delta_n, \mathscr H[B])$ to $\delta_n^2\sqrt n$ is bounded, the factor in parentheses in the modulus bound $\omega_n(\delta_n)$ is bounded, and
\[
r_n^2\,\omega_n(1/r_n)
\lesssim r_n^2\, J_{[\,]}(C_B\delta_n, \mathscr H[B])
\lesssim n^{2/3}(\log n)^{-4(p-1)/3} \cdot n^{-1/6}(\log n)^{4(p-1)/3}
= \sqrt n .
\]
This verifies the rate condition of \cref{thm:m-estimator-rate}.

Condition \ref{cond:m-rate-min} of \cref{thm:m-estimator-rate}, near minimization, requires an error $\P_n[\ell(f_n) - \ell(f)]$ of order at most
\[
r_n^{-2} = n^{-2/3}(\log n)^{4(p-1)/3}.
\]
The error is $O_{\P}(n^{-3/4}(\log n)^{p-1})$, and $n^{-3/4}(\log n)^{p-1}$ is the squared rate $r_n^{-2}$ times $n^{-1/12}(\log n)^{-(p-1)/3}$, which tends to zero, so the error is $o_{\P}(r_n^{-2})$.
This verifies condition \ref{cond:m-rate-min}, and \cref{thm:m-estimator-rate} gives
\[
\|\hat f_{n,M_n} - f\| = O_{\P}(r_n^{-1}) = O_{\P}\big(n^{-1/3}(\log n)^{2(p-1)/3}\big).
\]
This proves the rate for every sequence of radii in $\Xi$.

The uniformity over $\Xi$ follows by a routine argument.
Seeking a contradiction, suppose that the limit in the lemma is not zero.
Then there are a number $\chi > 0$, sample sizes $n_1 < n_2 < \cdots$ and, for each $j$, a radius $M$ in $\Xi$ at which the estimate $\hat f_{n_j,M}$ has $\P\{r_{n_j}\|\hat f_{n_j,M} - f\| > j\} \ge \chi$.
Take the radii $M_n$ of the first part equal to these radii at the sample sizes $n_j$, and equal to $\min \Xi$ at the other sample sizes $n$.
The first part of the proof gives $\|\hat f_{n,M_n} - f\| = O_{\P}(r_n^{-1})$ for these radii, so the probabilities at the sample sizes $n_j$ tend to zero, which is a contradiction.
\end{proof}

%% file: theory/cv-inequality.tex
\begin{lemma}[Cross-validation inequality]
\label{lem:cv-inequality}
Let the observations be split into $V \ge 2$ folds, and let $\P_{n,v}$ be the empirical distribution of fold $v$.
For each index $k \in \{1, \dots, J\}$ and each fold $v$, let $u_{k,v}$ be a function of an observation, and let $k^*_n$ minimize $V^{-1}\sum_v \P_{n,v}u_{k,v}$ over these indices.
Then, for every index $k$,
\begin{align*}
\frac1V\sum_v \P u_{k^*_n,v} &\le \frac3V\sum_v \P u_{k,v} + \frac1V\sum_v \bigl(\zeta_v^+ + \zeta_v^-\bigr),\\
\text{where}\quad
\zeta_v^\pm &= \max_{1 \le j \le J}\bigl\{\pm 2(\P_{n,v} - \P)u_{j,v} - \P u_{j,v}\bigr\}.
\end{align*}
\end{lemma}

This is Lemma 2.1 of \citet{cv-oracle}, with their constant $\delta$ equal to one.

%% file: theory/cv-maximal.tex
\begin{lemma}[Maximal inequality]
\label{lem:cv-maximal}
Let $\P_m$ be the empirical distribution of an IID sample of size $m$ from $\P$, and let $u_1, \dots, u_J$ be functions of an observation.
Suppose that, for a constant $b$, each function $u_j$ satisfies $|u_j| \le b$ on the support of $\P$, $\P u_j \ge 0$ and $\P u_j^2 \le b\,\P u_j$.
Then, for every $\delta > 0$,
\[
\mathbb E\max_{1 \le j \le J}\bigl\{(\P_m - \P)u_j - \delta\,\P u_j\bigr\} \le 8\Bigl(1 + \frac{3}{2\delta}\Bigr)\frac{b\log(1+J)}{m}.
\]
The same bound holds with $\P_m - \P$ replaced by $\P - \P_m$.
\end{lemma}

This is Lemma 2.2 of \citet{cv-oracle} with their exponent equal to one, written with the empirical distribution $\P_m$ in place of their empirical process.

%% file: sections/A3-experiments.tex
\section{Details of the Experiments}
\label{sec:experiment-details}

\subsection{Regularization Search for Kernel Ridge Regression}
\label{sec:krr-reg}

Here we describe in more detail how we searched and set the regularization hyperparameter $\lambda$ for all of the kernel methods in our experiments.
This method is not specific to HAR and can be used with any positive semidefinite kernel.

Our strategy is to find an upper endpoint $\lambda_0$ at which the model is sufficiently regularized that its predictions are close to a constant, relative to the largest absolute observed outcome.
Once we have this value, we can generate a log-scale grid of values up to $\lambda_0$ to validate over, or perform a bounded adaptive search.

Since we only care about the prediction function we have in hand, we consider the training set fixed and denote it $X_1,\ldots,X_n$ and $\bm Y=(Y_1,\ldots,Y_n)^\top$.
Let $\hat f_\lambda$ be the kernel ridge regression learned from these data with regularization $\lambda>0$.
Our fits include an unpenalized intercept, so the constant limit is $\bar Y=n^{-1}\sum_iY_i$.
We first bound deviations from this limit and then use a training-data version of the bound to choose $\lambda_0$.

We start with the expression
\[
\hat f_\lambda(x)-\bar Y
=
\bigl(\bm k(x)-\bar{\bm k}\bigr)^\top
\underbrace{
(\bm K+\lambda I)^{-1}(\bm Y-b_\lambda\bm 1)
}_{\bm\alpha_\lambda},
\]
where $\bm K_{ij}=K(X_i,X_j)$, $\bm k(x)_i=K(x,X_i)$, $\bar{\bm k}=n^{-1}\bm K\bm 1$, and $\bm 1$ is the vector of $n$ ones.
The fitted intercept is
\[
b_\lambda
=
\frac{
\bm 1^\top(\bm K+\lambda I)^{-1}\bm Y
}{
\bm 1^\top(\bm K+\lambda I)^{-1}\bm 1
}.
\]
The coefficients and intercept satisfy
\[
(\bm K+\lambda I)\bm\alpha_\lambda+b_\lambda\bm 1=\bm Y,
\qquad
\bm 1^\top\bm\alpha_\lambda=0.
\]
The centered expression above follows by averaging the first equation, which gives $b_\lambda=\bar Y-\bar{\bm k}^{\top}\bm\alpha_\lambda$.

From Cauchy--Schwarz we know
\[
|\hat f_\lambda(x)-\bar Y|
\leq
\|\bm k(x)-\bar{\bm k}\|_2\|\bm\alpha_\lambda\|_2.
\]
Taking the supremum over the prediction domain gives
\[
\|\hat f_\lambda-\bar Y\|_\infty
\leq
\sup_x\|\bm k(x)-\bar{\bm k}\|_2\|\bm\alpha_\lambda\|_2.
\]
To bound the coefficient norm, let $\bm Y_c=\bm Y-\bar Y\bm 1$ and let $d_{\min}=\operatorname{eig}_n(\bm K)\geq0$ denote the smallest eigenvalue of $\bm K$.
The coefficient equations imply
\[
\begin{aligned}
\left(
\min_{\|\bm v\|_2=1}
\bm v^\top(\bm K+\lambda I)\bm v
\right)\|\bm\alpha_\lambda\|_2^2
&=(d_{\min}+\lambda)\|\bm\alpha_\lambda\|_2^2\\
&\leq
\bm\alpha_\lambda^\top(\bm K+\lambda I)\bm\alpha_\lambda\\
&=\bm\alpha_\lambda^\top\bm Y_c\\
&\leq
\|\bm\alpha_\lambda\|_2\|\bm Y_c\|_2.
\end{aligned}
\]
Thus $\|\bm\alpha_\lambda\|_2 \leq\|\bm Y_c\|_2/(d_{\min}+\lambda)$.

Putting it all together, we have
\[
\|\hat f_\lambda-\bar Y\|_\infty
\leq
\frac{
\sup_x\|\bm k(x)-\bar{\bm k}\|_2\|\bm Y_c\|_2
}{
d_{\min}+\lambda
}.
\]
For $\bm Y\ne0$ and a finite supremum, a sufficient condition for this bound to be at most $\varepsilon\|\bm Y\|_\infty$ is
\[
\lambda
\geq
\frac{
\sup_x\|\bm k(x)-\bar{\bm k}\|_2\|\bm Y_c\|_2
}{
\varepsilon\|\bm Y\|_\infty
}
-d_{\min},
\qquad
\|\bm Y\|_\infty=\max_i|Y_i|.
\]

In practice, we use the training data to set the endpoint.
For nonzero $\bm K$ and $\bm Y$, our experiments used
\[
\lambda_0
=
\frac{
\max_i\|\bm K_i\|_2\|\bm Y\|_2
}{
\varepsilon\max_i|Y_i|
}
-\operatorname{eig}_n(\bm K),
\]
where $\bm K_i$ is the $i$th row of $\bm K$ and $0<\varepsilon<1$.
To see what this endpoint guarantees, write $R_K=\max_i\|\bm K_i\|_2$.
Since $\bar{\bm k}=n^{-1}\sum_i\bm k(X_i)$, we have $\max_i\|\bm k(X_i)-\bar{\bm k}\|_2\leq2R_K$.
Also, $\|\bm Y_c\|_2\leq\|\bm Y\|_2$.
Therefore
\[
\lambda\geq\lambda_0
\quad\Longrightarrow\quad
\max_i|\hat f_\lambda(X_i)-\bar Y|
\leq
\frac{2R_K\|\bm Y\|_2}{d_{\min}+\lambda}
\leq
2\varepsilon\|\bm Y\|_\infty.
\]
This gives an empirical criterion for the upper endpoint; a uniform out-of-sample guarantee would require control of the supremum above.
If $\bm Y=0$ or $\bm K=0$, the fit is already constant and any positive endpoint can be used.

Everything in the expression for $\lambda_0$ can be calculated directly from the training data.
Our experiments used the grid of \cref{sec:estimators}, whose largest value is $\lambda_0$, with $\varepsilon=\mthKrrEps$.
The lower endpoint and number of grid values are specified there.

\subsection{Additional Estimators}
\label{sec:additional-empirical}

This subsection adds three estimators to the comparison of \cref{sec:empirical}.
Each one used the same training and test data, and the same \mthFolds cross-validation folds, as the estimators of \cref{sec:estimators}.
\Cref{tab:empirical-full,tab:runtime-full} repeat \cref{tab:empirical,tab:runtime} with all the estimators, so each ratio in \cref{tab:empirical-full} is relative to the best of all of them.

\paragraph*{MARS}
Multivariate adaptive regression splines (MARS) \cite{Friedman1991-mars} fit a greedy linear combination of products of first-order splines with knots at the data, so, like HAR and HAL, they build a fit from products of one-dimensional splines.
We used pymars \cite{pymars}, a Python implementation that follows the R package earth \cite{earth}.
Each fit added terms in a forward pass and then pruned them in a backward pass, which chose the number of terms by generalized cross-validation.
We kept the defaults of earth, except that the forward pass could add up to \mthMarsMaxTerms terms: the default limit, $\min(200, \max(20, 2p)) + 1$ terms, stops the forward pass in about half of the fits of a pilot on our data.
We tuned the largest degree of a term, which is the number of covariates that one term can multiply, as we tuned the depth of HAR: along the same path, by \mthFolds-fold cross-validation with the same early stopping, followed by a refit at the chosen degree on all the training data.

\paragraph*{Random Forest}
Random forest is a popular and easy-to-use tree-based method.
We included it as a common nonparametric baseline that works with an inductive bias similar to that of HAR and HAL.
We used the random forest of scikit-learn \cite{scikit-learn}, version \mthSklearn, with \mthRfTrees trees and every other setting at its default, without tuning.

\paragraph*{Multilayer Perceptron (MLP)}
Neural networks are now ubiquitous but are known to sometimes perform poorly on tabular data \cite{Grinsztajn2022-pz, Shwartz-Ziv2022-lm}.
As with radial basis KRR, we included them as a baseline that has a very different inductive bias from the tree-based and sectional-variation-controlling methods.
Unlike radial basis KRR, though, MLPs are readily scalable (modulo the complex tuning) and very popular in modern ML \cite{LeCun2015-fo}.
We used the scikit-learn implementation of MLP, with \mthMlpLayers hidden layers of equal width and ReLU activations, trained by Adam with the learning rate \mthMlpRate in batches of \mthMlpBatch rows, on standardized covariates and a standardized outcome.
We tuned the width over \mthMlpWidths, the weight decay (the $\mathscr L_2$ penalty on the weights) over \mthMlpDecays, and the number of training epochs, by \mthFolds-fold cross-validation.
For each setting of the width and the weight decay, one network trained in each fold, all of them one epoch at a time, and after each epoch we averaged their validation errors over the folds.
Training stopped when this average had not improved for \mthMlpPatience epochs, or after \mthMlpMaxEpochs epochs.
We then trained the setting with the smallest average validation error on all the training data, for the number of epochs at which it reached that error.

\begin{table}[ht]
\centering
\scriptsize
\setlength{\tabcolsep}{2pt}
\input{artefacts/tables/empirical-full.tex}
\caption{Test MSE of all the estimators across the UCI datasets, averaged over the \mthTableReps repetitions, as a multiple of the smallest such average on the same dataset, which is in bold.
The column ``rows'' gives the number of rows used, before the split, and the column $p$ gives the number of covariates.
A dash marks a method that we did not run on a dataset.}
    \label{tab:empirical-full}
\end{table}

\begin{table}[ht]
\centering
\scriptsize
\setlength{\tabcolsep}{2pt}
\input{artefacts/tables/runtime-full.tex}
\caption{Mean time in seconds over the \mthTableReps repetitions to tune and fit each of the estimators on the training data and predict the test data, on one core.
A dash marks a method that we did not run on a dataset.}
    \label{tab:runtime-full}
\end{table}

%% file: artefacts/tables/empirical-full.tex
\begin{tabular}{lrr|lllllllll}
\toprule
data & rows & $p$ & HAR & \makecell{Mixed\\Sobolev\\KRR} & HAL & \makecell{Gradient\\Boosted\\Trees} & \makecell{Radial\\Basis\\KRR} & MARS & \makecell{Random\\Forest} & MLP & \makecell{Elastic\\Net} \\
\midrule
power & 2000 & 4 & 1.02 & 1.05 & 1.08 & \textbf{1.00} & 1.13 & 1.19 & 1.04 & 1.16 & 1.34 \\
yacht & 308 & 6 & 2.68 & \textbf{1.00} & 2.57 & 2.01 & 2.10 & 6.67 & 6.67 & 10.7 & 453 \\
concrete & 1030 & 8 & 1.03 & \textbf{1.00} & 1.13 & 1.25 & 1.93 & 2.03 & 1.68 & 1.44 & 7.61 \\
energy & 768 & 8 & 1.22 & 1.24 & 1.01 & \textbf{1.00} & 2.04 & 2.30 & 2.10 & 4.18 & 79.6 \\
kin8nm & 2000 & 8 & 3.06 & 2.69 & 3.37 & 3.18 & 1.24 & 3.20 & 3.82 & \textbf{1.00} & 5.78 \\
protein & 2000 & 9 & 1.03 & 1.07 & \textbf{1.00} & 1.00 & 1.09 & 1.09 & 1.01 & 1.12 & 2.14 \\
wine & 1599 & 11 & 1.09 & 1.09 & 1.18 & 1.11 & 1.14 & 1.33 & \textbf{1.00} & 1.18 & 1.25 \\
boston & 506 & 13 & 1.23 & \textbf{1.00} & 1.35 & 1.17 & 1.12 & 1.66 & 1.25 & 1.17 & 2.58 \\
naval & 2000 & 17 & 168 & 138 & 439 & 1222 & \textbf{1.00} & 165 & 767 & 1122 & 6574 \\
yearmsd & 2000 & 90 & 1.10 & 1.11 & --- & 1.13 & 1.06 & 1.39 & 1.23 & \textbf{1.00} & 1.31 \\
slice & 2000 & 384 & 1.77 & 1.25 & --- & 6.46 & \textbf{1.00} & 49.4 & 14.1 & 2.44 & 18.4 \\
\bottomrule
\end{tabular}

%% file: artefacts/tables/runtime-full.tex
\begin{tabular}{lrr|lllllllll}
\toprule
data & rows & $p$ & HAR & \makecell{Mixed\\Sobolev\\KRR} & HAL & \makecell{Gradient\\Boosted\\Trees} & \makecell{Radial\\Basis\\KRR} & MARS & \makecell{Random\\Forest} & MLP & \makecell{Elastic\\Net} \\
\midrule
power & 2000 & 4 & 48 & 11 & 220 & 24 & 24 & 1.4 & 7.3 & 261 & 0.099 \\
yacht & 308 & 6 & 0.29 & 0.24 & 2.5 & 14 & 0.34 & 0.85 & 1.4 & 126 & 0.13 \\
concrete & 1030 & 8 & 9.7 & 3 & 64 & 42 & 3.6 & 44 & 4.4 & 205 & 0.12 \\
energy & 768 & 8 & 4 & 1.4 & 18 & 20 & 1.8 & 3.4 & 1.8 & 166 & 0.2 \\
kin8nm & 2000 & 8 & 71 & 14 & 842 & 54 & 24 & 260 & 13 & 301 & 0.1 \\
protein & 2000 & 9 & 52 & 11 & 572 & 15 & 24 & 8.2 & 15 & 194 & 0.25 \\
wine & 1599 & 11 & 44 & 12 & 567 & 14 & 15 & 41 & 9 & 94 & 0.11 \\
boston & 506 & 13 & 1.1 & 0.85 & 35 & 23 & 0.77 & 27 & 3.8 & 100 & 0.1 \\
naval & 2000 & 17 & 42 & 11 & 3494 & 18 & 24 & 33 & 16 & 966 & 7.6 \\
yearmsd & 2000 & 90 & 156 & 65 & --- & 420 & 29 & 294 & 146 & 586 & 0.47 \\
slice & 2000 & 384 & 346 & 142 & --- & 1095 & 48 & 1388 & 120 & 407 & 33 \\
\bottomrule
\end{tabular}